\documentclass[preprint,3p,times]{elsarticle}

\usepackage{amssymb}
\usepackage{amsthm}

\usepackage{amsmath, amsfonts}%, amsthm, amssymb}
\usepackage{mathabx}   % upup-..-dowdow arrows
\usepackage[makeroom]{cancel}    % cancellation symbol
\usepackage{tipa}      % fonetica
\usepackage{textcomp}  % caratteri extra
\usepackage{latexsym}  % simboli vari
\usepackage{mathrsfs}  % caratteri corsivi (agg. io); per utilizzarlo: \mathscr{•}

\usepackage{graphicx}
\usepackage{tikz}
\usepackage{pgfplots}
\pgfplotsset{compat=1.6}%{width=10cm, compat=1.9} % retrocompatibilità e proprietà varie plot
\usepgfplotslibrary{groupplots}
\usepackage{pgf}
\usetikzlibrary{positioning,shapes,arrows,fit,calc,automata,perspective,3d}
\usepackage{neuralnetwork}
\usepackage{tikz-network}

\usepackage{hyperref}
\hypersetup{
	colorlinks=true,       % false: boxed links; true: colored links
	linkcolor=blue,%red,          % color of internal links (change box color with linkbordercolor)
	citecolor=red,%green,        % color of links to bibliography
	filecolor=magenta,      % color of file links
	urlcolor=cyan           % color of external links
}
\usepackage{cleveref}

\usepackage{hhline}
\usepackage{multirow}
\usepackage{colortbl} 
\usepackage{algpseudocode}
\usepackage{algorithm}
\usepackage{subcaption} 
\usepackage{enumitem}% http://ctan.org/pkg/enumitem

\newtheorem{definition}{Definition}[section]%
\newtheorem{proposition}{Proposition}[section]%
\newtheorem{theorem}{Theorem}[section]%

\theoremstyle{definition}
\newtheorem{remark}{Remark}[section]

\newcommand{\green}[1]{{\color{green}#1}}
\newcommand{\cyan}[1]{{\color{cyan}#1}}
\newcommand{\orng}[1]{{\color{orange}#1}}

\newcommand{\mgnt}[1]{{\color{blue}#1}}

\newcommand{\norm}[1]{\Vert{#1}\Vert}	% Norma
\newcommand{\N}{\ensuremath{\mathbb{N}}}    %   N:  insieme dei numeri naturali
\newcommand{\R}{\ensuremath{\mathbb{R}}}	%   R:  insieme dei numeri Reali
\DeclareMathOperator{\argmin}{arg\,min}

\renewcommand{\v}{\ensuremath{\boldsymbol}}	%   v:	vettore (grassetto)

\journal{arXiv}

\begin{document}

\begin{frontmatter}



\title{Graph-Instructed Neural Networks for Sparse Grid-Based Discontinuity Detectors}


\author[inst1,inst3]{Francesco Della Santa\corref{cor1}}
\cortext[cor1]{Corresponding author}

\affiliation[inst1]{organization={Department of Mathematical Sciences, Politecnico di Torino},%Department and Organization
            addressline={Corso Duca degli Abruzzi 24}, 
            % city={Turin},
            postcode={10129}, 
            state={Turin},
            country={Italy}}

\affiliation[inst3]{organization={Gruppo Nazionale per il Calcolo Scientifico INdAM},%Department and Organization
            addressline={Piazzale Aldo Moro 5}, 
            % city={Rome},
            postcode={00185}, 
            state={Rome},
            country={Italy}}

\author[inst1,inst3]{Sandra Pieraccini}

\begin{abstract}
In this paper, we present a novel approach for detecting the discontinuity interfaces of a discontinuous function. This approach leverages Graph-Instructed Neural Networks (GINNs) and sparse grids to address discontinuity detection also in domains of dimension larger than 3. GINNs, trained to identify troubled points on sparse grids, exploit graph structures built on the grids to achieve efficient and accurate discontinuity detection performances. We also introduce a recursive algorithm for general sparse grid-based detectors, characterized by convergence properties and easy applicability.
Numerical experiments on functions with dimensions $n = 2$ and $n = 4$ demonstrate the efficiency and robust generalization properties of GINNs in detecting discontinuity interfaces. Notably, the trained GINNs offer portability and versatility, allowing integration into various algorithms and sharing among users.
\end{abstract}

% \includegraphics[width=1.\textwidth]{Figures/sg_detector_algscheme.png}

% \begin{highlights}
% \item Introduction of a novel sparse grid-based method for detecting discontinuity interfaces via neural network models.

% \item The sparse grid-based and neural network-based detection method is suitable for problems in domains of dimension higher than three.

% \item Provided a set of trained NN models for public use for discontinuity detection in 2D and 4D, facilitating integration into diverse scientific applications.
% \end{highlights}


\end{frontmatter}

% \linenumbers

% ------------------------------------------------------------------

\section{Introduction}\label{sec:intro}

Detecting discontinuity interfaces of discontinuous functions is a challenging task with significant implications across various scientific and engineering applications. Identifying these interfaces is particularly critical for functions with a high-dimensional domain, as their discontinuities can significantly influence the behavior of numerical methods and simulations; for example, within the realm of uncertainty quantification, where the smoothness of the target function plays a fundamental role in the use of stochastic collocation methods. Specifically, the knowledge of discontinuity interfaces enables the partitioning of the function domain into regions of smoothness, a crucial factor in improving the performance of numerical methods (e.g., see \cite{JAKEMAN2013}). Other examples of discontinuity detection applications include signal processing, investigations of phase transitions in physical systems \cite{Jaeger1998_PHASETRANSITION}, and change-point analyses in geology or biology, to name a few \cite{Wang2022}. 

The central objective of most discontinuity detection methods is to identify the position of discontinuities in the function domain using function evaluations on sets of points. Over the last few decades, progresses have been made in discontinuity detection, leading to the development of various algorithms. Notable works, such as \cite{Archibald2005,Archibald2009,Jakeman2011,ZWGB2016}, have introduced significant contributions in this field. In particular, \cite{Archibald2005} introduced a polynomial annihilation edge detection method designed for piece-wise smooth functions with low-dimensional domains ($n\leq 2$). This method identifies discontinuous interfaces by reconstructing jump functions based on a set of function evaluations. Starting from these results, \cite{Archibald2009} extended the approach to higher dimensions, applying the detection method for each input dimension within a generalized polynomial chaos approximation of the target function. However, in \cite{Archibald2009}, the curse of dimensionality restricts considerably the practical applicability of the method. A significant advance in addressing the curse of dimensionality was made in \cite{Jakeman2011}, where sparse grids (see \cite{Sm63,bungartz_griebel_2004,Piazzola2024}) were utilized to create an adaptive method. This approach increased the method's potential for use in higher dimensions. On the other hand, few years later, \cite{ZWGB2016} presented a unique approach based on the approximation of hypersurfaces representing the discontinuity interfaces using hyper-spherical coordinates. This method is suitable for large domain dimensions but was primarily designed for detecting a single interface, assuming it satisfies star-convexity assumptions.

Other techniques addressing the discontinuity detection problem include wavelet-based methods \cite{Bozzini1994,Suresh2016}, filter-based algorithms \cite{Jain1995book,Wei2005}, and troubled cell detection \cite{Gao2017,Vuik2016,Ray2018,Wang2022}. In the recent work \cite{Bracco2023}, the authors present a method using the so-called null rules, applied and tested on bivariate functions; in particular, this method analyzes the rules' asymptotic properties, introducing two indicators for function and gradient discontinuities for the detection of points near discontinuities. Moreover, the method presented in \cite{Bracco2023} is integrated with adaptive approximation techniques based on hierarchical spline spaces for reconstructing surfaces with discontinuities.

In recent years, Neural Networks (NNs) have also been successfully applied to discontinuity detection. For example, in \cite{DISC4NN} a novel approach emerged, aimed at constructing discontinuous NNs capable of approximating discontinuous functions incorporating trainable discontinuity jump parameters, and enabling the model to simultaneously learn the target function and its discontinuity interface. This method offered several advantages, with few restrictions on its applicability; actually, the main limit is the curse of dimensionality related to the training set cardinality, which can be very large for functions in high dimensional domains.

Another interesting application of NNs in this context is reported in \cite{Wang2022}, which introduces a discontinuity detection method based on specifically tailored Convolutional Neural Network (CNN) models. This study has been motivated by the success of CNNs in the context of image edge detection (e.g., \cite{ElSayed2013,Liu2019,Wang2016inbook,Wen2018,Xue2019}); indeed, the problem of edge detection in computer vision is essentially a two-dimensional discontinuity detection problem, with observed data taking the form of intensity values of image pixels. The CNN-based method developed in \cite{Wang2022} consists of a two-level detection procedure. The procedure involves a coarse-to-fine process with two detectors: the first aimed at preliminary detection over coarsened grids for rapid identification of coarse-troubled cells, and the second focused on refining and providing detailed detection within the detected coarse-troubled cells, primarily on fine grids. In general, \cite{Wang2022} introduces a cheap and very efficient method for discontinuity detection with respect to functions with two- or three-dimensional domain and, in principle, works also in higher dimensions. Nonetheless, the curse of dimensionality limits its applicability in the latter case, both for the need of using regular grids for function evaluations and for the need of implementing $n$-dimensional convolutional layers with $n > 3$ (typically, existing Deep Learning frameworks address dimensions $n\leq 3$, such as \cite{tensorflow2015-whitepaper}). 

We believe that \cite{Wang2022} describes a very promising direction; in this work, we modify and extend the metodology of \cite{Wang2022}, moving from regular grids (and CNNs) to sparse grids (and NNs based on graphs), in order to tackle the problem in dimensions higher than 3.
In particular, we introduce a novel approach that utilizes Graph-Instructed Neural Networks (GINNs) \cite{GINN,EWGINN} trained for detecting the so-called troubled points in a sparse grid used for the function evaluations; in particular, each GINN is based on the adjacency matrix of a graph built on the structure of the chosen sparse grid. In addition, partially inspired by \cite{Jakeman2011}, we define a recursive algorithm for general sparse grid-based detectors, and characterized by finite termination and convergergence properties. Then, running this algorithm using as detector a GINN well-trained in dimension $n$, we obtain a cheap and fast discontinuity detection method suitable for any function with $n$-dimensional domain. Actually, also Multi-Layer Perceptrons (MLPs) can be trained as sparse grid-based detectors but, according to the results in \cite{GINN}, the experiments show that GINNs take advantage of the graph structure built on the sparse grid, returning better results.

Concerning the experiments, we test our method on functions with domains of dimension $n=2$ but also on a function with domain of dimension $n=4$, in order to prove its efficiency. The results obtained are promising and show very good generalization abilities of the GINNs in detecting discontinuity interfaces even for functions different from the ones used in the training sets. Moreover, we point the attention of the reader to the fact that the trained GINNs for discontinuity detection have the advantage of being portable and versatile; i.e., they can be shared among different users and integrated into new algorithms. For this reason, we make available the trained GINNs used in the two- and four-dimensional experiments for public use and research (see \Cref{rem:github_link}).

The work is organized as follows. 
In \Cref{sec:SGandGraphs}, we recall briefly how to build sparse grids based on equispaced collocation knots and we introduce the concept of sparse grid graph. In \Cref{sec:disc_det}, we give the definition of discontinuity detectors and we introduce the algorithms for finding the discontinuity interfaces of a discontinuous function. Then, in \Cref{sec:NN_discontinuity_det}, we describe the procedure for building and training a NN-based discontinuity detector (both MLP and GINN); further details are illustrated in \ref{sec:building_details}, at the end of the paper. In \Cref{sec:num_exp}, we show the results of the NN-based discontinuity detectors for two-dimensional functions (including images, for edge-detection applications) and four-dimensional functions, proving the efficiency and the promising potentialities of the new method, especially if based on GINNs. We end with some conclusions drawn in \Cref{sec:conc}.

% # ------------------------------------------------- #
% # ------------------------------------------------- #
% # ------------------------------------------------- #
% # ------------------------------------------------- #

\section{Problem Settings and Preliminaries}\label{sec:SGandGraphs}

In this work, we consider the discontinuity detection problem with respect to a discontinuous function $g:\Omega\subseteq\R^n\rightarrow\R$, where the task is performed exploiting sparse grids in order to contrast the curse of dimensionality. In this Section, we briefly recall the basic notions about sparse grids, according to the overview described in \cite{Piazzola2024} and focusing on the case of nested and equispaced sets of collocation knots; then, we introduce new definitions, mainly focused on building graphs with vertices corresponding to the points of a sparse grid. Indeed, the algorithm we introduce later in this work for discontinuity detection (see \Cref{sec:discdet_algorithm}) is based on such a kind of graphs.

For the sake of simplicity, let us consider in the following $\Omega=[\alpha_1,\beta_1]\times\cdots\times[\alpha_n,\beta_n]$. This is not a restricting assumption as any function $g:\Omega\subseteq\R^n\rightarrow\R$ can be extended, continuously or not, to a hyperrectangular domain $R=[\alpha_1,\beta_1]\times\cdots\times[\alpha_n,\beta_n]$; e.g., setting $g(\v{x})\equiv 0$ for each $\v{x}\in R\setminus \Omega$.

\subsection{Sparse Grids}\label{sec:SG}

Sparse grids were originally introduced by Smolyak in \cite{Sm63} for high dimensional quadrature and since then also used for polynomial interpolation and stochastic collocation (e.g., see \cite{Barthelmann2000,Nobile2007,Nobile2008}). While referring the reader to \cite{bungartz_griebel_2004} for a comprehensive work on sparse grids, we sketch here the main ideas.

A sparse grid $\mathcal{S}$ in $\R^n$ is the union of a collection of $n$-dimensional tensor grids, each one obtained by taking the Cartesian product of univariate sets of knots hyerarchically built through a number of refinement levels. Starting from a univariate distribution of knots, we consider, on the $i$-th interval $[\alpha_i ,\beta_i]$, a sequence of set of knots $\{\mathcal{K}_{i, h}\}_{h \geq 1}$, where the set $\mathcal{K}_{i, h+1}$ is a given refinement of the set $\mathcal{K}_{i, h}$, for each $h\geq 1$. Sparse grids are then obtained by collecting suitable $n$-dimensional tensor products of knot sets $\mathcal{K}_{i, h}$. More precisely, a sparse grid is built with the following steps:
\begin{enumerate}
    \item Upon selecting a univariate knot distribution, choose an initial starting grid $\mathcal{K}_{i, 1}$ on each interval  $[\alpha_i,\beta_i]$; let $m(1)$ denote the number of nodes in $\mathcal{K}_{i, 1}$, for all $i=1, \ldots, n$; then, for $h>1$, $h \in \N$, consider a finer grid $\mathcal{K}_{i, h}$ with $m(h)$ nodes, with $m(h+1) \geq m(h)$; the index $h$ represents the refinement level of $\mathcal{K}_{i, h}$.

    \item Choose a set $\mathcal{I}\subset \N_+^n$ of multi-indices; for each $\v{h}=(h_1,\ldots ,h_n) \in \mathcal{I}$, let $\mathcal{S}_{\v{h}}$ denote the $n$-dimensional grid defined as
    \begin{equation*}\label{eq:tensor_grid}
    \mathcal{S}_{\v{h}} = \bigotimes_{i=1}^n \mathcal{K}_{i,h_i}\,.
    \end{equation*}
    
    \item The sparse grid $\mathcal{S}$ is the union of all the tensor grids generated at the previous step:
    \begin{equation*}
        \mathcal{S} := \bigcup_{\v{h}\in\mathcal{I}} \mathcal{S}_{\v{h}}\,.
    \end{equation*}
\end{enumerate}
Several choices are possible for the univariate knot distribution, for the number of nodes $m(h)$ to be introduced on an interval at a given refinement level $h$, and for the set of multi-indices used to build $\mathcal{S}$. We refer the reader to \cite{bungartz_griebel_2004} for a comprehensive work on sparse grids, and to \cite{Piazzola2024} for a MATLAB library for the generation of sparse grids; in \cite{Piazzola2024} the parameter $h$ is called \emph{discretization level}, 
 and the function $m:\N_+\rightarrow\N_+$ providing $m(h)$ for each $h$ is called \emph{level-to-knots} function.
 
Herein, we focus on equispaced nodes, essentially doubling the number of nodes at each refinement, namely
\begin{equation}\label{eq:ltk_func_doubling}
        m(h) := 
        \begin{cases}
            1\,, \quad & \text{if }h=1,\\
            2^{h-1}+1\,, \quad & \text{otherwise}.
        \end{cases}
\end{equation}
It is worth noting that this choice, upon refinement, yields a nested set of nodes.

As far as the selection of multi-indices set is concerned, let $h^*\in\N_+$ denote a fixed level of approximation; then, the subset can be described as
\begin{equation*}\label{eq:multiindex_set_withrule}
        \mathcal{I}(h^*):=\{\v{h}\in \N_+^n \, |\, r(\v{h})\leq h^*\}\,.
    \end{equation*}
where $r:\N_+^n\rightarrow \N_+$ is suitably chosen. Some examples are \cite{bungartz_griebel_2004,BackNobileTamelliniTempone2011,Piazzola2024}:
\begin{equation}\label{eq:multiindex_rules}
        r_{\text{prod}}(\v{h}) := \prod_{i=1}^n h_i\,, \quad r_{\text{sum}}(\v{h}) := \sum_{i=1}^n h_i\,, \quad  
        r_{\max}(\v{h}) := \max_{i=1,\ldots ,n} h_i\,.
    \end{equation}
Choice $r_{\text{sum}}$ is known as the Total Degree rule, in the framework of polynomial approximation; choice  $r_{\max}$ corresponds to a full tensor grid.
The multi-index generated by choices in \eqref{eq:multiindex_rules} will be denoted by $\mathcal{I}_{\text{prod}}(h^*)$, $\mathcal{I}_{\text{sum}}(h^*)$, and $\mathcal{I}_{\max}(h^*)$, respectively. 
    
In the following, we show a sparse grid example, starting from $[\alpha_i,\beta_i]=[-1,1]$ for $i=1,2$, from univariate equispaced nodes following \eqref{eq:ltk_func_doubling}, and using the multi-indices set $\mathcal{I}_{\text{sum}}(6)$; the obtained 2-dimensional sparse grid is depicted in  \Cref{fig:sg_cross_2d_noedges}. From now on, for simplicity, we will denote as \emph{equispaced sparse grid} any sparse grid characterized by an initial equispaced nodes distribution.

\begin{figure}[htb]
    \centering
    \includegraphics[width=0.35\textwidth]{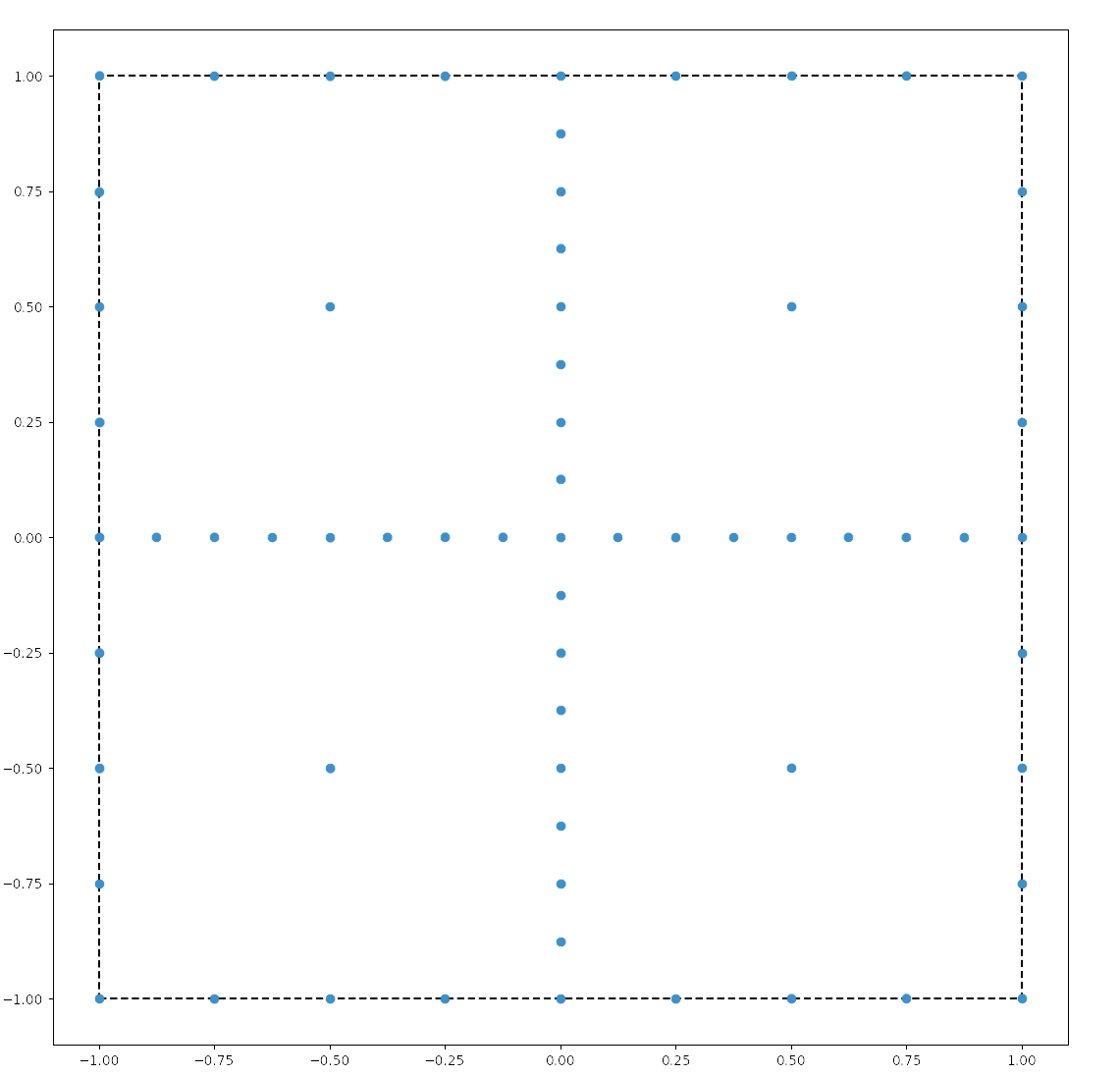}
    \caption{Example of sparse grid in $\R^2$ obtained with $[\alpha_i,\beta_i]=[-1,1]$ for $i=1,2$, level-to-knot function $m$ (see \eqref{eq:ltk_func_doubling}), equispaced knots $\mathcal{K}^{[2]}_{i,h}$, and multi-indices set $\mathcal{I}_{\text{sum}}(6)$. Black dotted lines depict the boundary of the sparse grid box (see \Cref{def:sg_box}).}
    \label{fig:sg_cross_2d_noedges}
\end{figure}

\subsubsection{Sparse Grid Box and Sparse Grid Similarity}\label{sec:sgBox_sgSim}

In view of the description of the discontinuity detection algorithm presented in \Cref{sec:disc_det}, we introduce the following definitions related to sparse grids.

\begin{definition}[Sparse Grid Box]\label{def:sg_box}
Let $\mathcal{S}$ be a sparse grid in $\R^n$. Then, the hyperrectangle
\begin{equation}\label{eq:sg_box}
    \Omega = [\alpha_1, \beta_1] \times \cdots \times [\alpha_n, \beta_n]\,
\end{equation}
upon which $\mathcal{S}$ is built, is called \emph{box} of $\mathcal{S}$ and denoted as $\mathbb{B}(\mathcal{S})$

\end{definition}

\quad

\begin{definition}[Sparse Grid Similarity]\label{def:sg_similarity}
Let $\mathcal{S}'$ and $\mathcal{S}''$ be two sparse grids in $\R^n$, given by the set of points
$\mathcal{S}'=\{\v{x}'_1,\ldots ,\v{x}'_N\}$ and $\mathcal{S}''=\{\v{x}''_1,\ldots ,\v{x}''_N\}$, respectively. Then, $\mathcal{S}'$ and $\mathcal{S}''$ are similar if and only if there exist $a, b \in \R$, $a\neq 0$, such that
\begin{equation}\label{eq:sg_similarity}
\v{x}'_i = a\,\v{x}''_i + b\,,
\end{equation}
for all $i=1,\ldots , N$.
The similarity between two sparse grids is denoted by $\mathcal{S}'\simeq\mathcal{S}''$.
\end{definition}

\quad

\subsection{Sparse Grid Graphs}

The position of the points of a sparse grid is naturally inclined to define a graph. Nonetheless, to the best of authors' knowledge, in literature there are no formal procedures to build a graph from a sparse grid. Therefore, we introduce here a formal definition for a graph associated to a sparse grid in $\R^n$. The graph representations of sparse grids is aimed at exploiting the connections between the grid points in the framework of a discontinuity detection task (see later \Cref{sec:disc_det}). 

Given two points $\v{x}_1,\v{x}_2 \in \R^n$, we let $\overline{(\v{x}_1,\v{x}_2)}$ denote the line segment connecting $\v{x}_1$ and $\v{x}_2$.
We define a graph associated to $\mathcal{S}$ according to the following definition.

\begin{definition}[Sparse Grid Graph]\label{def:graph_sg}
Let $\mathcal{S} \subset \R^n$ be a sparse grid made of $N$ points. We define \emph{sparse grid graph} (SGG) based on $\mathcal{S}$ the non-directed graph $G=(\mathcal{S}, E)$ such that, for each $\v{x}_i,\v{x}_j\in \mathcal{S}$, $\v{x}_i \neq \v{x}_j$, the set $\{\v{x}_i, \v{x}_j\}$ is an edge of $G$ if and only if:
\begin{enumerate}[label=(\roman*)]
    \item $\v{x}_i,\v{x}_j$ are aligned along one of the {axes} in the $\R^n$ space;
    \item \label{SGG:rule_1} there is no $\v{x}\in \mathcal{S}$, $\v{x}\neq\v{x}_i,\v{x}_j$, such that $\v{x}$ belongs to the segment $\overline{(\v{x}_i,\v{x}_j)}$;
    \item \label{SGG:rule_2} it holds 
    \begin{equation*}
        \norm{\v{x}_i - \v{x}_j} < \norm{\v{x}_p - \v{x}_q}\,,
    \end{equation*}
    for each $\v{x}_p,\v{x}_q\in \mathcal{S}$ such that $\overline{(\v{x}_p,\v{x}_q)}$ is perpendicular to $\overline{(\v{x}_i,\v{x}_j)}$ and $\overline{(\v{x}_p,\v{x}_q)}\cap\overline{(\v{x}_i,\v{x}_j)}\neq \emptyset$.
\end{enumerate}
\end{definition}

\quad

In a nutshell, we connect two points of the sparse grid aligned along one of the {axes} if no other grid points exist between them (item \ref{SGG:rule_1}, see \Cref{fig:graph_from_sg_firstexample}-left); then, we remove all the edges corresponding to line segments $\overline{(\v{x}_i,\v{x}_j)}$ that intersect other line segments $\overline{(\v{x}_p,\v{x}_p)}$ whose length is smaller than or equal to the one of $\overline{(\v{x}_i,\v{x}_j)}$ (item \ref{SGG:rule_2}, see \Cref{fig:graph_from_sg_firstexample}-right).

\begin{figure}[htb]
    \centering
    \includegraphics[width=0.35\textwidth]{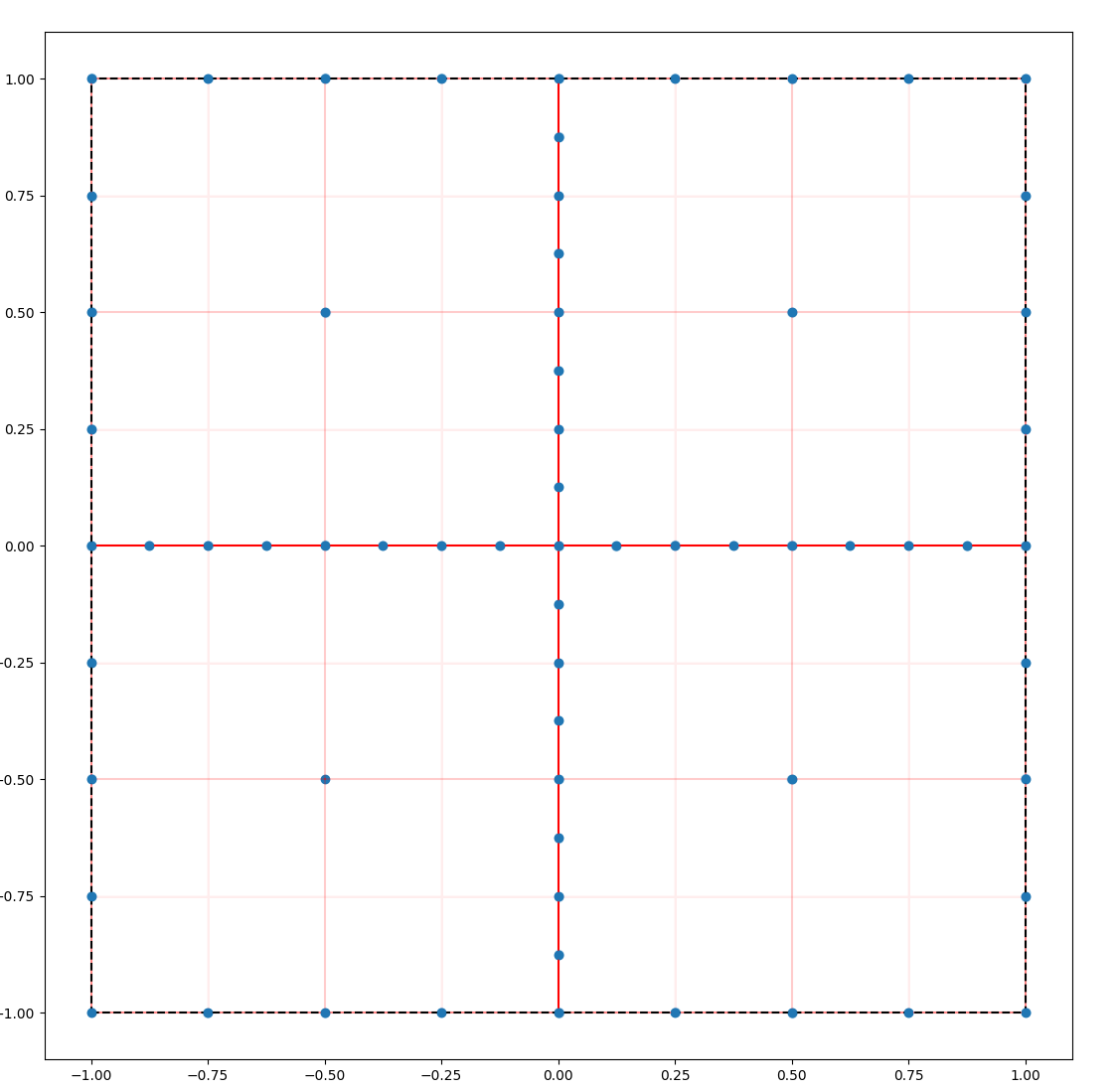}
    \quad 
    \includegraphics[width=0.35\textwidth]{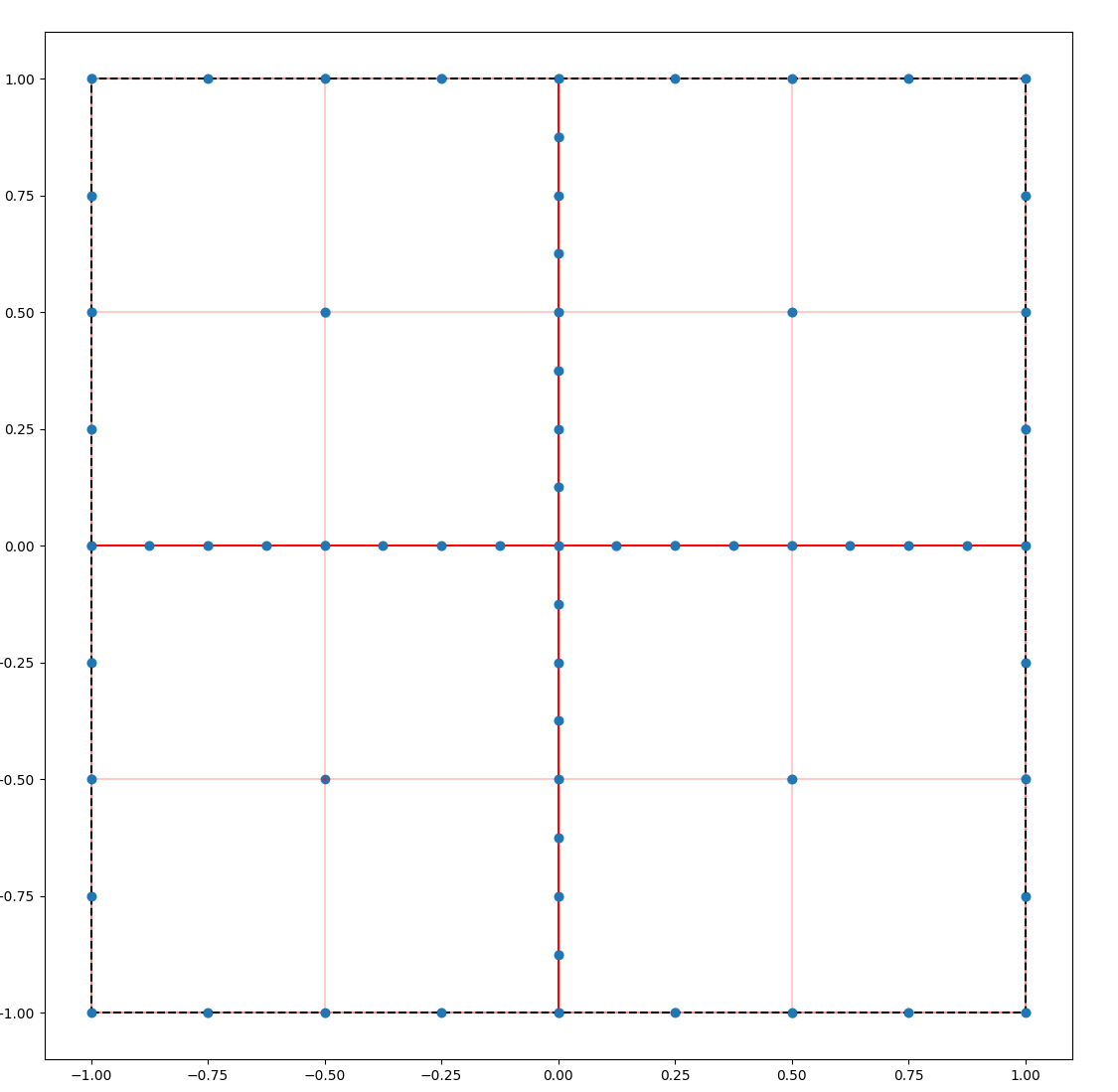}
    \caption{Building a sparse grid graph based on the sparse grid $\mathcal{S}$ in \Cref{fig:sg_cross_2d_noedges}. \emph{Left.} Raw graph obtained connecting the aligned and consecutive points of the sparse grid. \emph{Right.} Final graph obtained removing from the raw one the edges that intersect an edge of length equal or shorter. The edge opacity depends on the edge weights (see \Cref{def:wgraph_sg}).}
    \label{fig:graph_from_sg_firstexample}
\end{figure}

From now on, overloading notation, we will identify the line segment $\overline{(\v{x}_i,\v{x}_j)}$ with the corresponding edge $\{\v{x}_i,\v{x}_j\}$ in the SGG. Then, the edges of a SGG correspond to line segments in $\R^n$ with {possibly} different lengths. Therefore, it is natural to exploit the length of line segments to assign weights to edges.

Here, we use the strategy of setting a weight inversely proportional to the distance in $\R^n$ between the points defining the edges. A weighted version of SGGs will be useful for improving the training of discontinuity detectors based on Graph-Instructed NNs.

\begin{definition}[Weighted Sparse Grid Graph]\label{def:wgraph_sg}
Let $G=(\mathcal{S}, E)$ be a sparse grid graph based on a sparse grid $\mathcal{S}$ in $\R^n$; let $\ell\in\R$ be the length of the shortest segment among the ones defined by edges of $G$, i.e.:
\begin{equation}\label{eq:ell_min}
    \ell:=\min_{\{\v{x}_i,\v{x}_j\}\in E} \norm{\v{x}_i - \v{x}_j} \,.
\end{equation}
Then, $G$ is a \emph{weighted sparse grid graph} (WSGG) based on $\mathcal{S}$ if the edge weights are defined by
\begin{equation}\label{edgeweight_general}
    \omega_{ij}:=\frac{\ell}{\norm{\v{x}_i - \v{x}_j}}\in(0, 1]\,,
\end{equation}
for each edge $\{\v{x}_i, \v{x}_j\}\in E$.
\end{definition}

\quad

We remark that the edge weights introduced in \eqref{edgeweight_general} are relative quantities, depending on the relative distance of the points in the sparse grid with respect to $\ell$, and $\ell$ is related to the grid refinement level. Therefore, the WSGG of two similar sparse grids only differs for the set of vertices.
This property is summarized in the following proposition.

\begin{proposition}\label{prop:weghts_constant}
Let $\mathcal{S}'$ and $\mathcal{S}''$ be two similar sparse grids in $\R^n$ and let $G'=(\mathcal{S}', E'), G''=(\mathcal{S}'', E'')$ be the corresponding weighted sparse grid graphs. 
Then, $G'$ and $G''$ have the same adjacency matrix.
\end{proposition}

\begin{proof}
The proof is almost straightforward. Let $A'=(A_{ij}'), A''=(A_{ij}'')\in\R^{N\times N}$ be the adjacency matrices of $G',G''$, respectively. Since $\mathcal{S}'\simeq\mathcal{S}''$, $A'$ and $A''$ have the same sparsity pattern, i.e.: 
\begin{equation*}
    A_{ij}' \neq 0 \Longleftrightarrow A_{ij}'' \neq 0\,, \ \forall \ i,j\in\{1,\ldots ,N\}\,.
\end{equation*}

Moreover, for each $i,j\in\{1,\ldots ,N\}$ such that $\{\v{x}'_i,\v{x}'_j\}\in E'$ and $\{\v{x}''_i,\v{x}''_j\}\in E''$, from equation \eqref{eq:sg_similarity} we have that
\begin{equation*}
    \norm{\v{x}'_i-\v{x}'_j} = \norm{a\,\v{x}''_i + b - a\,\v{x}''_j - b} = |a|\norm{\v{x}''_i - \v{x}''_j}\,;
\end{equation*}
therefore, the corresponding
non-zero elements of the adjacency matrices are equal, because

\begin{equation}
    A'_{ij}=\omega_{ij}' = \frac{\ell'}{\norm{\v{x}'_i - \v{x}'_j}} = \frac{|a|\,\ell''}{|a|\norm{\v{x}''_i - \v{x}''_j}} = \omega_{ij}''=A''_{ij}\,.
\end{equation}
\end{proof}

\begin{definition}[Sparse Grid Graph Similarity]\label{def:sg_graph_similarity}
Let $G',G''$ be two sparse grid graphs (weighted or not) based on two sparse grids $\mathcal{S}',\mathcal{S}''$, respectively. Then, $G'$ and $G''$ are similar ($G'\simeq G''$) if $\mathcal{S}'\simeq \mathcal{S}''$.
\end{definition}

\quad

\begin{remark}[Weights for graphs based on equispaced sparse grids with hypercubic box]\label{rem:weights_hypercubic_nested_cells}
It is easily seen that if $\mathcal{S}$ is a sparse grid characterized by equispaced collocation knots and by a hypercubic box of edge $\Lambda$ (i.e., $\beta_1-\alpha_1=\cdots=\beta_n - \alpha_n=\Lambda$) then, for each edge $e_{ij}=\{\v{x}_i,\v{x}_j\}$, there exists $d_{ij}\in\N$ such that the segment length is $\Lambda/2^{d_{ij}}$. Let $h_{\max} > 1$ be the maximum refinement level of $\mathcal{S}$ and $M:=m(h_{\max}) - 1=2^{h_{\max}-1}$ (see \eqref{eq:ltk_func_doubling}); then, $d_{ij}\leq h_{\max}-1$, $\ell=\Lambda / M$, and $\omega_{ij} = 2^{(d_{ij} - (h_{\max} - 1))}\in (0, 1\, ]$.
\end{remark}

% # ------------------------------------------------- #
% # ------------------------------------------------- #
% # ------------------------------------------------- #
% # ------------------------------------------------- #

\section{Discontinuity Detection and Sparse Grids}\label{sec:disc_det}

The discontinuity detection method proposed herein aims to extend the one proposed in \cite{Wang2022} moving from uniform grids to sparse grids. In view of this extension, we seek for a NN architecture leveraging sparse grids; thus, we consider the use of Graph-Instructed Neural Networks (GINNs) built with respect to a SGG.
We remark that the method proposed herein is designed in such a way that also classic NNs, such as Multi-Layer Perceptrons (MLPs), can be used; however, the numerical experiments show that the method benefits from the use of GINN (see \Cref{sec:num_exp}); this is in line with the results presented in \cite{GINN}.

For a given function $g : \Omega\subseteq\R^n \to \R$, the method proposed relies on the following ingredients:
\begin{itemize}
    \item the concept of {\em trouble points} of a sparse grid (i.e, points of the grid that are near to a discontinuity interface of $g$);
    \item the availability of a so-called {\em discontinuity detector}, i.e. a function or algorithm that returns the troubled points. 
\end{itemize}
The algorithm defined in this work does not depend on the nature of the  discontinuity detector used, and the function $g$ can be rather general;  however, the algorithm works better with discontinuous functions characterized by discontinuity interfaces with small codimensions (with respect to $\R^n\supseteq\Omega$).
In the following, we formally introduce the notions of troubled point and discontinuity detector based on a sparse grid. The definition of troubled points is inspired by the definition of troubled cells given in \cite{Wang2022}. 

\begin{definition}[Troubled Points]\label{def:troubled_point}
Let $\mathcal{S}$ be a sparse grid in $\R^n$ and let $G=(\mathcal{S}, E)$ be the (possibly weighted) graph based on $\mathcal{S}$. Let $g:\Omega\subseteq\R^n\rightarrow\R$ be a given function such that $\mathbb{B}(\mathcal{S})\subseteq\Omega$ and let $\delta$ be the set of all and only the discontinuity points of $g$. Then, the node $\v{x}_i$ of $\mathcal{S}$ is said to be a troubled point with respect to $g$ if there exist $e_{ij}=\{\v{x}_i,\v{x}_j\}\in E$ such that:
\begin{itemize}
    \item $\delta \cap \overline{(\v{x}_i,\v{x}_j)}\neq\emptyset\,$;
    \item it holds 
    \begin{equation}\label{eq:sg_ginn_pi_2}
        \norm{\v{x}_i - \v{x}^\delta_i} \leq \norm{\v{x}^\delta_i - \v{x}_j}\,,
    \end{equation}
    where $\v{x}^\delta_i$ is the point in $\delta \cap e_{ij}$ nearest to $\v{x}_i$.
\end{itemize}
\end{definition}

\quad

In other words, $\v{x}_i \in \mathcal{S}$ is a troubled point if $\delta$ intersects at least one edge $\{\v{x}_i,\v{x}_j\}\in E$ in a point between $\v{x}_i$ and the mid-point of the line segment corresponding to the edge.

\begin{definition}[Discontinuity Detector]\label{def:disc_det}
Let $\mathcal{S}$ be a sparse grid in $\R^n$ made of $N$ points and let $\mathcal{S}_{/\simeq}$ denotes the set of all the sparse grids similar to $\mathcal{S}$. Let $\mathcal{F}$ denotes the set of all the functions $g:\Omega\rightarrow\R$, for an arbitrary $\Omega\subseteq\R^n$. Then, a \emph{discontinuity detector} based on $\mathcal{S}$ is a function $\Delta$ that, for each  $\mathcal{S}'\simeq \mathcal{S}$ and each function $g\in\mathcal{F}$, returns a vector $\v{p}=(p_1,\ldots ,p_N)$, with $0 \leq  p_i \leq 1$ for all $i=1,\ldots, N$, where $p_i$ gives an indication of how likely $\v{x}_i'\in\mathcal{S}'$ is a troubled point. The closer $p_i$ is to $1$, the more likely $\v{x}'_i$ is a troubled point.
\end{definition}

\quad

According to \Cref{def:disc_det}, a discontinuity detector provides, for each grid point, a sort of likelihood for being a troubled point. A point $\v{x}_i'$ is labeled as troubled if $p_i\geq \tau$, where $\tau\in(0,1]$ is a given threshold.

The previous definition is intentionally qualitative, as several types of discontinuity detectors can be defined. For example, we can define detectors that evaluate $\v{p}$ knowing: $i$) the function $g$ and the coordinates of the points in $\mathcal{S}'$; $ii$) the {values of $g$ taken} at the points in $\mathcal{S}'$ and the points' coordinates; $iii$) the {values of $g$ taken} at the points in $\mathcal{S}'$, only. In this work, we focus on building NN-based detectors of the latter type; i.e., detectors $\Delta$ such that $\Delta : (\R\cup\{\infty\})^N \rightarrow [0,1]^N$ and
        \begin{equation}\label{eq:sg_discdet}
            \Delta\left(
            \v{g}'
            \right)
            =
            \Delta\left(g(\v{x}_1'),\ldots ,g(\v{x}_N')\right)
            =
            \v{p}
            \,.
        \end{equation}
By convention, we set $g(\v{x}_i')=\infty$ and $p_i=0$, if $\v{x}_i'\not\in\Omega$

Based on the previous definition, we also introduce the concept of {\em exact} discontinuity detector based on a sparse grid.

\begin{definition}[Exact Discontinuity Detector]\label{def:exact_disc_det}
Let $\mathcal{S}$ be a sparse grid in $\R^n$ made of $N$ points. An \emph{exact discontinuity detector} based on $\mathcal{S}$ is a discontinuity detector $\Delta^*$ that for each $\mathcal{S}' \simeq \mathcal{S}$, and for each given function $g: \R^n \to \R$, returns a vector $\v{p}\in\{0,1\}^N$ such that, for $i=1,\ldots ,N$,
$$p_i=\begin{cases}
  1 & \text{ if $\v{x}_i' \in \mathcal{S}'$ is a troubled point with respect to $\mathcal{S}'$ and $g$} \\
  0 & \text{ otherwise}.
\end{cases}$$
\end{definition}

\quad

A non-exact discontinuity detector will be called {\em inexact detector}.

Of course, given a sparse grid, an inexact detector returns a guess for the points to be troubled, whereas the former one detects them exactly. In \Cref{fig:sg_detectors} the different behavior between an exact and an inexact discontinuity detector is depicted. In this example, for the inexact detector, a point $\v{x}_i$ is labeled as a troubled point if the corresponding value $p_i$ returned by the detector is $p_i \geq \tau=0.5$. From the example we can see that the inexact detector provides both false-positive troubled points (magenta circles) and false-negative troubled points (magenta dots).

However, an exact detector is typically defined through expensive algorithms and in general it is not a detector based on the discontinuous function evaluations only (as in  \eqref{eq:sg_discdet}). Therefore, given a sparse grid, our idea consists in building a non-expensive but effective discontinuity detector that is only-evaluations dependent, training a NN model for approximating the exact one; given such a kind of NN-based detector, we can use it with a sparse grid-based discontinuity detection algorithm, aiming to reach very good results at a consistently lower computational cost, even in domains of dimension $n > 3$.

\begin{figure}[htb]
    \centering
    \includegraphics[width=0.4\textwidth]{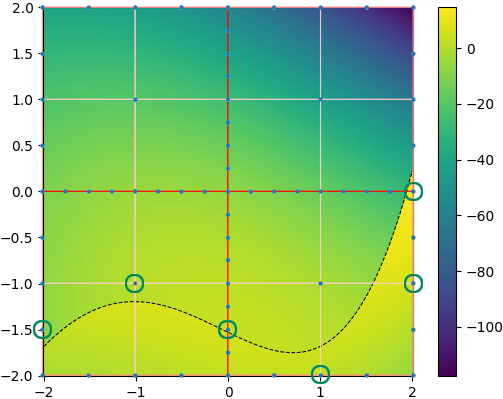}
    \quad
    \includegraphics[width=0.4\textwidth]{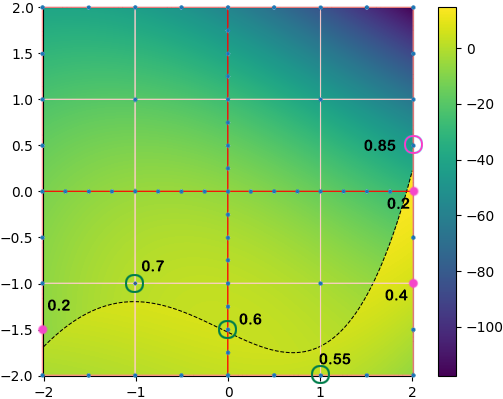}
    \caption{Example of sparse grid points classification made for a function $g:\R^2\to\R$ with an exact discontinuity detector $\Delta^*$ (left) and a inexact discontinuity detector $\Delta$ (right). The black dotted line is the function's discontinuity interface. Green circles are true troubled points, magenta circles are false troubled points, magenta dots are false non-troubled points. For the inexact detector, we report also the values $p_i$ corresponding to the highlighted sparse grid points (threshold $\tau=0.5$).}
    \label{fig:sg_detectors}
\end{figure}

\begin{remark}[Detectors independent of point coordinates]\label{rem:no_coord_detectors}
Similarly to what done in \cite{Wang2022}, the introduction of a NN-based detector that is only-evaluations dependent (see \eqref{eq:sg_discdet}) is useful to make the detector universally applicable to any sparse grid $\mathcal{S}'\simeq\mathcal{S}$ and any discontinuous function with domain dimension equal to $n$. On the other hand, since these detector will not be exact, we can expect that the smaller are the ratios between the function evaluations and the edge lengths (i.e., the slopes $|g(\v{x}_i') - g(\v{x}_j')|/\norm{\v{x}_i' - \v{x}_j'}$), the lower is the reliability of the predicted troubled points.
\end{remark}

\subsection{Algorithm for Discontinuity Detection}\label{sec:discdet_algorithm}

In \cite{Wang2022} the discontinuity detection method is characterized by a two-level approach:
\begin{enumerate}
    \item A first CNN detector is applied on a coarse uniform grid, identifying the coarse troubled cells;
    \item A second CNN detector is applied on a fine uniform grid defined on the coarse troubled cells, identifying the fine troubled cells.
\end{enumerate}

We adopt a similar approach for our sparse grid-based method. However, thanks to \Cref{prop:weghts_constant} and \Cref{rem:no_coord_detectors}, we can use the same detector $\Delta$ for both the coarse step and any refinement step. Therefore, we define an iterative approach that stops when a minimum box size for the sparse grid is reached. We point the reader to the fact that we can apply several refinements leveraging the use of sparse grids, which allow to save on the number of function evaluations, thus delaying the {\em curse of dimensionality} phenomenon.

Let us consider a sparse grid $\mathcal{S}$ built starting from equispaced knots on a hypercubic box; we recall that, for simplicity, we denote as \emph{equispaced sparse grids} the sparse grids built upon equispaced collocation knots. Given a discontinuity detector $\Delta$ based on $\mathcal{S}$, and given a discontinuous function $g:\Omega \to \R$, we sketch the basic version of a sparse grid-based discontinuity detection algorithm as follows (see \Cref{fig:algorithm_scheme}), while referring for the detailed version to the pseudo-code reported in \Cref{alg:sg_alg_detection}:
\begin{enumerate}
    \item place in $\Omega$ one or more initial sparse grids similar to $\mathcal{S}$ (equispaced, hypercubic box);
    \item using $\Delta$, detect the troubled points of one sparse grid and mark {the grid} as ``checked'';
    \item for each troubled point $\v{x}$, generate a new sparse grid $\mathcal{S}'$ similar to $\mathcal{S}$, with box $\mathbb{B}(\mathcal{S}')$ centered in $\v{x}$ and box's edge length equal to the length of the longest graph's edge intersecting the troubled point. If this latter length is shorter than a chosen value $\Lambda_{\min}\in\R_+$, do not create the new sparse grid.
    \item repeat steps 2 and 3, until all the sparse grids are ``checked''.
    \item the troubled points obtained with respect to the smallest sparse grids are the final troubled points detected by $\Delta$.
\end{enumerate}

\begin{algorithm}[htb!]
\caption{Sparse grid-based discontinuity detection (basic)}\label{alg:sg_alg_detection}
\begin{algorithmic}[1]
\Require \quad

$g:\Omega\subseteq\R^n\rightarrow\R$, function;

$\mathcal{S}$, sparse grid (equispaced, hypercubic box);

$\Delta$, discontinuity detector based on $\mathcal{S}$; 

$L=((\v{x}_1, \lambda_1),\ldots ,(\v{x}_m, \lambda_m))$, sequence of \emph{center} and \emph{edge length} pairs characterizing the sparse grid boxes $\mathbb{B}(\mathcal{S}^{(1)}),\ldots ,\mathbb{B}(\mathcal{S}^{(m)})\subset\Omega$ (typically a queue); 

$\Lambda_{\min}\in\R_+$, minimum edge length allowed for sparse grid boxes; 

$\tau\in (0,1]$, detection tolerance.

\Ensure \quad

$T\subset\R^n$, set of points detected as troubled points.

\State $T\gets \emptyset$
\State $C\gets\emptyset$ \Comment{set of ``visited'' center-edge pairs}
\While{$L$ is not empty}
    \State $(\v{x},\lambda)\gets$ get element of $L$ \Comment{typically the first one, if $L$ is a queue}
    \State $C\gets C\cup\{(\v{x},\lambda)\}$
    \State $D'\gets [x_1-\lambda/2,x_1+\lambda/2] \times \cdots \times [x_n-\lambda/2,x_n+\lambda/2]$
    \State $G'=(V',E')\gets$ SGG (or WSGG) based on $\mathcal{S}'\simeq\mathcal{S}$, s.t. $\mathbb{B}(\mathcal{S}')=D'$
    \State $\v{p}\gets$ vector returned by $\Delta$, w.r.t. $\mathcal{S}'$ and $g$
    \For{$i=1,\ldots ,N$}
        \If{$p_i\geq\tau$}
            \State $\lambda\gets \max_{\{\v{x}'_i,\v{x}'_j\}\in E'}\norm{\v{x}_i'-\v{x}_j'}$
            \If{$\lambda \geq \Lambda_{\min}$}
                \If{$(\v{x}'_i,\lambda)\not \in C$}
                    \State insert $(\v{x}'_i,\lambda)$ in $L$ \Comment{typically as last element, if $L$ is a queue}
                \EndIf
            \Else
                \State $T\gets T\cup\{\v{x}_i'\}$
            \EndIf
        \EndIf
    \EndFor
\EndWhile
\end{algorithmic}
\end{algorithm}

\begin{figure}[htb]
    \centering
    \includegraphics[width=1.\textwidth]{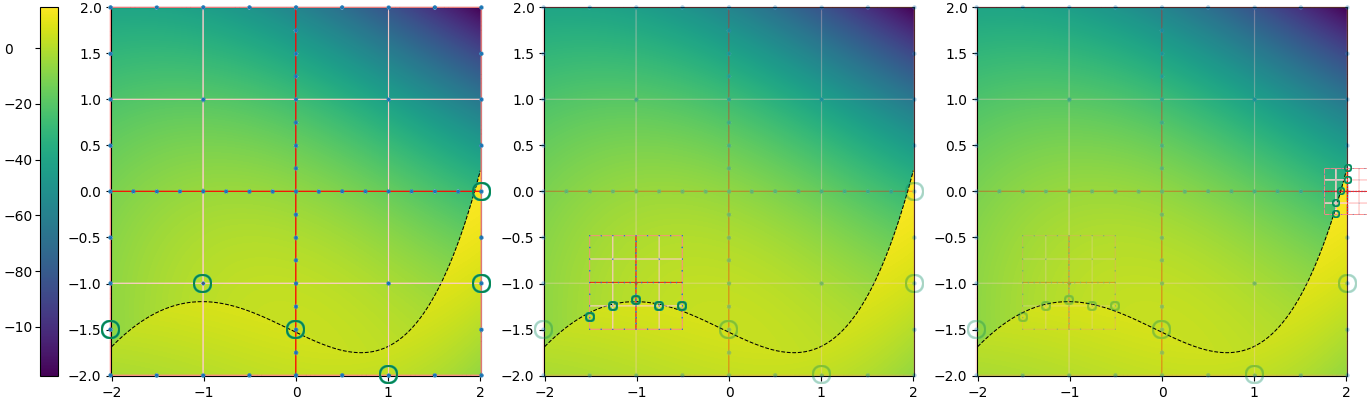}
    \caption{Example of three steps of \Cref{alg:sg_alg_detection}, assuming an exact detector $\Delta^*$ and a SGG as in \Cref{fig:graph_from_sg_firstexample}-right. Opaque green circles are the troubled points identified with the current grid. Transparent green circles are the troubled points identified at the previous steps and where the next sparse grids will be centered.
    }
    \label{fig:algorithm_scheme}
\end{figure}

\begin{remark}[Algorithm and non-hypercubic boxes]
In \Cref{alg:sg_alg_detection} the sparse grid are built on hypercubic boxes. This is not a restricting assumption, as the algorithm can be generalized resorting to sparse grid built on hyperrectangles, changing only the stopping criterion based on the box's edge length. For the sake of simplicity we keep the discussion focused on the case on sparse grids built upon hypercubic boxes.
\end{remark}

\begin{remark}[Equispaced sparse grids for optimized nested structure]
The algorithm exploits equispaced sparse grids for building a sort of nested structure that, in case of an exact discontinuity detector, can guarantee convergence to the discontinuity interfaces of $g$ (see \Cref{teo:alg_convergence}, next); moreover, this structure permits to save computational resources because it increases the probability that some points of a new sparse grid, generated at step 3, are shared with other sparse grids. In order to enhance this characteristic it is important that the boxes of the starting sparse grids are centered properly (i.e., according to the collocation knots criterium), if more than one is given at the first step of the algorithm. 
Theoretically, the algorithm can be extended also to non-equispaced sparse grids, but the generation of the new sparse grids (step 3) is not straightforward and is postponed to future work.
\end{remark}

\begin{remark}[Algorithm improvements]\label{rem:alg_improvements}
\Cref{alg:sg_alg_detection} represents the backbone of sparse grid-based algorithms for discontinuity detection. Then, depending on the user's needs, the algorithm can be changed and/or improved, e.g.:
\begin{itemize}
    \item a stopping criterion based on function evaluations can be introduced (useful for computationally expensive functions $g$);
    \item the function evaluation of the grid points can be stored and re-used;
    \item the troubled points' detection, for each box in the list $L$, can be optimized and/or parallelized.
\end{itemize}
We point the attention of the reader on the fact that $\Delta$ is independent of such a kind of improvements; then, it can be used with almost any modified version of \Cref{alg:sg_alg_detection}.
\end{remark}

It is worth remarking that any NN-based discontinuity detector is particularly suitable to be used with the algorithm for the troubled points predictions (see the last item in \Cref{rem:alg_improvements}); indeed, NNs can make predictions of thousands of inputs in the order of the second, exploiting the inner operations characterizing the NNs. 
In order to fully exploit the potential of NNs, we propose a modification of \Cref{alg:sg_alg_detection}. In particular, given NN-based detector $\Delta$, we can compute the vectors $\v{p}^{(1)},\ldots ,\v{p}^{(K)}\in\R^N$ corresponding to the points of the sparse grids $\mathcal{S}^{(1)},\ldots ,\mathcal{S}^{(K)}$, respectively, with just one NN batch-evaluation with respect to the vectors $\v{g}^{(1)},\ldots ,\v{g}^{(K)}$; i.e.,
\begin{equation}\label{eq:batch_p_pred}
    P = \Delta(\Gamma)\,,
\end{equation}
where 
\begin{equation}\label{eq:batch_p_pred_P_Gamma}
    P :=
    \begin{bmatrix}
        \v{p}^{(1)\,T}\\
        \vdots\\
        \v{p}^{(K)\,T}
    \end{bmatrix}
    =
    \left(p_i^{(k)}\right)_{\substack{k=1,\ldots ,K\\i=1,\ldots ,N}}
    \in\R^{K\times N}\,,
    \quad
    \Gamma :=
    \begin{bmatrix}
        \v{g}^{(1)\,T}\\
        \vdots\\
        \v{g}^{(K)\,T}
    \end{bmatrix}
    =
    \left(g(\v{x}_i^{(k)})\right)_{\substack{k=1,\ldots ,K\\i=1,\ldots ,N}}
    \in\R^{K\times N}\,.
\end{equation}
Therefore, the efficiency of the algorithm increases because we can evaluate entire batches of sparse grids at once and not {one at a time} (see lines 4-8 in \Cref{alg:sg_alg_detection}). 
This new NN-based method is detailed in \Cref{alg:sg_nn_alg_detection_parallel}.

\begin{algorithm}[htb!]
\caption{NN-based discontinuity detection}\label{alg:sg_nn_alg_detection_parallel}
\begin{algorithmic}[1]
\Require \quad

$g:\Omega\subseteq\R^n\rightarrow\R$, function; 

$\mathcal{S}$, sparse grid (equispaced, hypercubic box); 

$\Delta:\R^N\rightarrow[0,1]^N$, NN-based discontinuity detector (only-eval. dep.), based on $\mathcal{S}$; 

$L$, set of \emph{center} and \emph{edge length} pairs characterizing the boxes $\mathbb{B}(\mathcal{S}^{(1)}),\ldots ,\mathbb{B}(\mathcal{S}^{(m)})\subset\Omega$;  

$\Lambda_{\min}\in\R_+$, minimum edge length allowed for sparse grid boxes; 

$\tau\in (0,1]$, detection tolerance.

\Ensure \quad

$T\subset\R^n$, set of points detected as troubled points.

\State $T\gets \emptyset$
\State $C\gets\emptyset$ \Comment{set of ``visited'' center-edge pairs}
\While{$L$ is not empty}
    \State $\{G^{(1)},\ldots ,G^{(K)}\}\gets$ weighted graphs based on the sparse grids with boxes defined by $L$
    \State $C\gets C\cup L$
    \State $L\gets\emptyset$
    \State $\Gamma\in\R^{K\times N}\gets$ matrix $(g(\v{x}_i^{(k)}))$ \Comment{``batch of inputs'', see \eqref{eq:batch_p_pred_P_Gamma}}
    \State $P\in [0, 1]^{K\times N}\gets$ $\Delta(\Gamma)$ \Comment{``batch-predictions'' of $\Delta$, see \eqref{eq:batch_p_pred} and \eqref{eq:batch_p_pred_P_Gamma}}
    \For{$p_{ki}$ element of $P$}
        \If{$p_{ki}\geq \tau$}
            \State $\lambda\gets \max_{\{\v{x}_i^{(k)},\v{x}_j^{(k)}\}} \norm{\v{x}_i^{(k)} - \v{x}_j^{(k)}}$
            \If{$\lambda \geq \Lambda_{\min}$}
                \If{$(\v{x}_i^{(k)},\lambda)\not \in C$}
                    \State $L\gets L\cup\{(\v{x}_i^{(k)},\lambda)\}$
                \EndIf
            \Else
                \State $T\gets T\cup\{\v{x}_i^{(k)}\}$
            \EndIf
        \EndIf
    \EndFor
\EndWhile
\end{algorithmic}
\end{algorithm}

% \red{[VECCHIA VERSIONE].}
We conclude by observing the convergence properties of \Cref{alg:sg_alg_detection} under particular assumptions.

\begin{theorem}[Algorithm Convergence with Exact Detector]\label{teo:alg_convergence}
Let $\mathcal{S}$ be an equispaced sparse grid of $N>1$ points in $\R^n$ with minimum refinement level greater than one, for each dimension. Let $\Delta^*$ be an exact discontinuity detector based on $\mathcal{S}$. Let $\delta$ be the set of all and only the discontinuity points of a \emph{discontinuous} function $g:\Omega\subseteq\R^n\rightarrow\R$ and let $\mathcal{S}^{(1)},\ldots ,\mathcal{S}^{(m)}$ be the $m$ starting sparse grids of \Cref{alg:sg_alg_detection}.

Assume that there exists $j\in\{1,\ldots ,m\}$ such that the intersection of $\delta$ with the edges of the SGG based on $\mathcal{S}^{(j)}$ is not empty; then, for each $\varepsilon >0$, we have that \Cref{alg:sg_alg_detection}: 
\begin{enumerate}
    \item terminates in a finite number of steps.
    
    \item returns a set of troubled points $T$ such that 
    \begin{equation}\label{eq:alg_convergence}
        \mathrm{dist}(\v{x}^*, \delta) < \varepsilon\,,
    \end{equation}
    for each troubled point $\v{x}^*\in T$ detected, if $\Lambda_{\min}\leq 2\varepsilon$;
\end{enumerate}

\end{theorem}

\begin{proof}
Without loss of generality, we assume to start the algorithm with a unique sparse grid $\mathcal{S}$. Let $\Lambda$ denote the edge length of the sparse grid box $\mathbb{B}(\mathcal{S})$. Let $e_{ij}=\{\v{x}_i,\v{x}_j\}$ be an edge of the SGG based on $\mathcal{S}$ that intersects $\delta$ and, without loss of generality, let us assume that $\v{x}_i$ is a troubled point (see \Cref{def:troubled_point}); we denote by $\v{x}_{\delta}$ the point in $\delta\cap e_{ij}$ nearest to $\v{x}_i$. Let $\lambda_i$ be the length of the longest edge with $\v{x}_i$ as vertex; then, it holds 
\begin{equation}\label{eq:lambdas_ineq01}
\norm{\v{x}_i - \v{x}_\delta} \leq \frac{\norm{\v{x}_i - \v{x}_j}}{2} \leq \frac{\lambda_i}{2}.
\end{equation}
Since we assume that at least one refinement is made for each dimension, and since we use equispaced nodes as initial univariate distribution, we have $\lambda_i \leq \Lambda/2$; hence we can rewrite \eqref{eq:lambdas_ineq01} as 
\begin{equation}\label{eq:lambdas_ineq01b}
\norm{\v{x}_i - \v{x}_\delta} \leq \frac{\Lambda}{4}\,.
\end{equation}

According to \Cref{alg:sg_alg_detection}, two scenarios are possible: ($i$) $\lambda_i < \Lambda_{\min}$, then stop and assign $\v{x}_i$ to $T$; ($ii$) $\lambda_i \geq \Lambda_{\min}$, then refine. 

If case ($i$) occurs, \eqref{eq:alg_convergence} holds true because 
$$\mathrm{dist}(\v{x}_i,\delta)\leq\norm{\v{x}_i - \v{x}_\delta}\leq \lambda_i/2 < \Lambda_{\min}/2 \,\leq\, \varepsilon.$$
If case ($ii$) occurs, consider the sparse grid $\mathcal{S}'$ introduced in the refinement step, centered on $\v{x}_i$ and with box edge length $\Lambda' = \lambda_i$.
Let $e'_{pq}=\{\v{x}'_p,\v{x}'_q\}$ be an edge of the SGG based on $\mathcal{S}'$ that intersects $\delta$ such that, without loss of generality, $\v{x}'_p$ is a troubled point; we denote by $\v{x}'_{\delta}$ the point in $\delta\cap e'_{pq}$ nearest to $\v{x}'_p$. Let $\lambda'_{p}$ be the length of the longest edge with $\v{x}'_{p}$ as vertex; then, it holds 
\begin{equation}\label{eq:lambdas_ineq_02}
\norm{\v{x}'_{p} - \v{x}'_\delta} \leq \frac{\lambda'_{p}}{2} \leq \frac{\Lambda'}{4} \leq \frac{\Lambda}{8}
\end{equation}
where the third inequality is due to the fact that ${\Lambda'} \leq \frac{\Lambda}{2}$ because at least one refinement is made, and the second inequality applies because ${\lambda'_p} \leq \frac{\Lambda'}{2}$, as any edge length is at least half of the sparse grid box.

Repeating the same reasoning, we obtain after $k$ refinement steps, $k\geq 1$, the following inequalities:
\begin{equation}\label{eq:lambdas_ineq_general}
\norm{\v{x}^{(k)} - \v{x}^{(k)}_\delta} \leq \frac{\lambda^{(k)}}{2} \leq \frac{\Lambda^{(k)}}{4} \leq \frac{\Lambda}{2^{k + 2}} \quad \text{and} \quad \lambda^{(k)}\leq \frac{\Lambda}{2^{k+1}}\,,
\end{equation}
where the subscript for  $\v{x}^{(k)},\lambda^{(k)}$ is omitted for the ease of notation.

For a fixed $\Lambda_{\min} > 0$, let $\kappa  \in \N, \kappa \geq 0$, be the smallest integer such that  $\Lambda/(2^{\kappa +1})<\Lambda_{\min}$; then, $\lambda^{(\kappa)} \leq \frac{\Lambda}{2^{\kappa +1}}< \Lambda_{\min}$ and the algorithm has a finite termination. Moreover, it holds
\begin{equation*}\label{eq:alg_convergence_hth_level}
    \mathrm{dist}(\v{x}^{(\kappa)}, \delta) \leq \norm{\v{x}^{(\kappa)} - \v{x}^{(\kappa)}_\delta} \leq \frac{\lambda^{(\kappa)}}{2} < \frac{\Lambda_{\min}}{2} \leq \varepsilon
\end{equation*}
and thus the thesis follows.

\end{proof}

We conclude this subsection with an observation related to the potential versatility of the sparse grid-based discontinuity detectors.

\begin{remark}[Detectors and space dimension]\label{rem:disc_detectors_dimensions}
A discontinuity detector based on a sparse grid $\mathcal{S}$ in $\R^n$ can also be used for detecting discontinuities of functions $g':\Omega'\subseteq\R^{n'}\rightarrow \R$, with $n'<n$. Indeed, we can define a function $g:\Omega' \times \R^{n-n'} \subseteq\R^n\rightarrow\R$ as
$$g(x_1,\ldots ,x_{n'}, x_{n'+1},\ldots ,x_n) := g'(x_1,\ldots ,x_{n'}),$$
for each $\v{x}\in\Omega' \times \R^{n-n'}$, and use the detector with respect to $g$. Moreover, in this case \Cref{alg:sg_alg_detection} can be modified to ignore any detected troubled point with at least one coordinate $x_{n'+1},\ldots , x_n\neq 0$, assuming to start with sparse grids all centered in points with coordinates $x_{n'+1}=\cdots=x_n = 0$. The study of such a kind of generalization is postponed to future work.
\end{remark}

\section{NN-based Discontinuity Detectors}\label{sec:NN_discontinuity_det}

In this section we describe the method used for building a NN-based discontinuity detector to be used with \Cref{alg:sg_nn_alg_detection_parallel}. We recall that each discontinuity detector is based on a fixed sparse grid $\mathcal{S}$ and it is designed to work only with similar sparse grids; then, a dataset built with respect to $\mathcal{S}$ can be used to train only NN-based detectors based on $\mathcal{S}$.

First, we will sketch the entire procedure in \Cref{sec:sketch_procedure_NNbuilding}. Then, in the following subsections, we will provide details about the loss function and NN models. We refer the reader to \ref{sec:building_details} for further details related to the building of a deterministic approximation of an exact detector (\ref{sec:zerolevdet}), the creation of synthetic data for the training (\ref{sec:synth_data_creation}), data preprocessing (\ref{sec:data_prep}), and NN architecture archetypes (\ref{sec:NNarch_archetypes}).

\subsection{Building Procedure for NN-based Detectors}\label{sec:sketch_procedure_NNbuilding}

The construction of a NN-based detector is based on the following steps:
\begin{enumerate}
    \item Choose a sparse grid $\mathcal{S}$ in $\R^n$. Let $N$ denote the number of points in $\mathcal{S}$.
    
    \item\label{item:random_discfuncs} Generate a set of random piece-wise continuous functions $\mathcal{G}=\{g^{(1)},\ldots ,g^{(Q)}\}$ with known discontinuities, typically contained in the zero-level sets $\zeta^{(q)}$ of a function $f^{(q)}$, for each $q=1,\ldots ,Q$;
    
    \item\label{item:zlc_det_dataset_01} For each $g^{(q)}\in\mathcal{G}$, run \Cref{alg:sg_alg_detection} with {either an exact detector $\Delta^*$ or a deterministic approximation $\Delta$ of $\Delta^*$} (see \ref{sec:zerolevdet} for details); then, at each detection step with respect to {a} sparse grid $\mathcal{S}'\simeq\mathcal{S}$, {compute}:
    \begin{itemize}
        \item the evaluations of $g^{(q)}$ in the sparse grid points, i.e. the vector $\v{g}'=(g^{(q)}(\v{x}_1'),\ldots ,g^{(q)}(\v{x}_N'))$;
        \item the current value of $\v{p} = \Delta(\v{g}')$.
    \end{itemize}
    
    \item\label{item:zlc_det_dataset_02} Create a dataset of pairs $(\v{g}',\v{p})$ {collecting} the vectors {obtained} at the previous {step};
    
    \item Split the dataset into training, validation, and test set (denoted by $\mathcal{T}$, $\mathcal{V}$, and $\mathcal{P}$, respectively);
    
    \item\label{item:NNtraining_step} Given a NN with characterizing function $\widehat{\Delta}:\R^N\rightarrow [0, 1]^N$, train {the model} with respect to $\mathcal{T}$ and $\mathcal{V}$ to obtain {an approximation of} $\Delta$ and, as a consequence, the exact discontinuity detector $\Delta^*$; i.e.:
    \begin{equation*}\label{eq:NN_apprx_exactdiscdet}
        \widehat{\Delta}(\v{g}')\approx\Delta^*(\v{g}')\,,
    \end{equation*}
    for each {vector $\v{g}'$ of function evaluations at the grid points of $\mathcal{S}'$}.
    
    \item Measure the performance of the trained NN with respect to $\mathcal{P}$.
    
\end{enumerate}

The most expensive part of this procedure is item \ref{item:zlc_det_dataset_01}, because it involves a deterministic (approximation of) $\Delta^*$; typically, this detector is expensive, especially when we have a high space's dimension $n$ (i.e., $n\gg 1$).

\subsection{Loss Function}\label{sec:loss_func}

Concerning the loss function for training a NN-based detector, several choices are possible. In \cite{Wang2022}, the Mean Squared Error is used as loss function for training the convolutional detector. However, since the target vectors $\v{p}$ have values in $\{0,1\}$ (non-trouble point and troubled point, respectively), we prefer to use the loss function
\begin{equation}\label{eq:my_loss_func}
    \mathcal{L}(\mathcal{B}) = \frac{1}{|\mathcal{B}|}\sum_{(\v{g}',\v{p})} \sum_{i=1}^N -{\mu_1} p_i \log(\widehat{p}_i) - {\mu_0} (1 - p_i) \log(1-\widehat{p}_i)\,,
\end{equation}
for any batch $\mathcal{B}$ of input-output data, and where $\widehat{\v{p}}:=
\widehat{\Delta}(\v{g}')\in [0,1]^N$; namely, $\mathcal{L}$ is the average sum of the component-wise application of the binary cross-entropy loss \cite{BinaryCrossEntropy_TF, Goodfellow-et-al-2016}. The hyper-parameters {$\mu_1$ and $\mu_0$} are used to {suitably weight} the identification of troubled points or non-troubled points, respectively.

\subsection{NN Models: MLPs and GINNs}\label{sec:NN_models}

Looking at step \ref{item:NNtraining_step} of the building procedure for NN-based detectors (\Cref{sec:sketch_procedure_NNbuilding}), we observe that the only practical restrictions for the NN model consist of using an input layer of $N$ units and an output layer of $N$ units with values in $[0,1]$. Therefore, any architecture is allowed for the hidden layers.

Since \Cref{alg:sg_alg_detection} is based on SGGs with the same adjacency matrix (see \Cref{prop:weghts_constant}), it is reasonable to train a NN model tailored for tasks defined on graph-structured data, where the graph is fixed and only the node features vary. In \cite{GINN}, Berrone et al. recently observed that a good choice for learning a target function $\v{F}:\R^N\rightarrow\R^N$, dependent on the adjacency matrix of a fixed graph $G$, is to use a Graph-Instructed Neural Network (GINN) because they perform better than Multi-Layer Perceptrons (MLPs) and classic Graph Neural Networks (GNNs). The GINN models are part of the wide family of spatial GNNs and have been specifically developed for such a kind of task \cite{GINN}. Indeed, GNNs are mainly designed for other tasks than learning $\v{F}$ (e.g., graph classification, semi-supervised learning on nodes, etc.), while MLPs do not exploit the graph's connections \cite{GNNsurvey2020, GINN,EWGINN}. We point the reader to the fact that in \cite{GINN} GI layers and GINNs are denoted as Graph-Informed layers and Graph-Informed NNs, respectively. In \cite{HALL2021110192}, in a different framework from the one addressed in \cite{GINN}, a homonymous but different model is presented; therefore, starting from the work \cite{EWGINN} of Della Santa et al., the authors changed the names of both layers and NNs, to avoid confusion with \cite{HALL2021110192}.

{Due to} the observations about MLPs and GINNs, we focus {here} on these architectures for building the NN-based detectors. Moreover, the numerical experiments in \Cref{sec:num_exp} confirm the observation of \cite{GINN}; i.e., we observe that GINN-based detectors perform better than MLP-based detectors.

\subsubsection{Graph-Instructed Neural Networks}\label{sec:GINNs}

A GINN is a NN characterized by Graph-Instructed (GI) Layers. These layers are defined by an alternative graph-convolution operation introduced in \cite{GINN}. 
In brief, from a message-passing point of view, this graph-convolution operation is equivalent to having each node $v_i$ of $G$ sending to its neighbors a message equal to the input feature $x_i$, scaled by a weight $w_i$; then, each node sum up all the messages received from the neighbors, add the bias, and applies the activation function. {We point the attention of the reader to the fact that, only in this subsection, $\v{x}$ denotes feature vectors and not the sparse grid's points.}

In a nutshell, the message-passing interpretation can be summarized by \Cref{fig:ginnfilter} and the following node-wise equation
\begin{equation}\label{eq:ginn_node_action}
x_{i}' = \sum_{j \in \mathrm{N}_{\text{in}}(i)\cup \{i\}} x_j \, w_j  + b_i\,,
\end{equation}
where $x_i'$ is the layer's output feature corresponding to node $v_i$ and $\mathrm{N}_{\rm in}(i)$ is the set of indices such that $j\in\mathrm{N}_{\rm in}(i)$ if and only if $e_{ij}=\{v_i,v_j\}$ is an edge of the graph. We dropped the action of the activation function for simplicity.

\begin{figure}[htb]
    \centering
    \resizebox{0.65\textwidth}{!}{
    \begin{tikzpicture}[multilayer=3d,rotate=90]
    \Vertices{Figures/ginnfilteraction/ginnfilter_verts_v2.csv}
    \Edges{Figures/ginnfilteraction/ginnfilter_edges_v2.csv}
    \end{tikzpicture}
	}
    \caption{Visual representation of \eqref{eq:ginn_node_action}. Example with $n=4$ nodes (non-directed graph), $i=1$; for simplicity, the bias is not illustrated.}
    \label{fig:ginnfilter}
\end{figure}

Starting from \eqref{eq:ginn_node_action}, we can formally define the GI layers through a compact formulation. Given a graph $G$ (without self-loops) and its adjacency matrix $A\in\R^{N\times N}$, the simplest version of GI layer for $G$ is a NN layer described by a function $\mathcal{L}^{GI}:\R^N\rightarrow\R^N$, with one input feature per node and one output feature per node, such that
\begin{equation}\label{eq:GI_action_simple}
    \begin{aligned}
    \mathcal{L}^{GI}(\v{x}) 
    = 
    \v{\psi}\left( \,(\,\mathrm{diag}(\v{w}) (A + \mathbb{I}_N)\,)^T\, \v{x} + \v{b}\right),
    \end{aligned}
    \,
\end{equation}
where $\v{x}\in\R^N$ denotes the vector of input features and:
\begin{itemize}
    \item $\v{w}\in\R^N$ is the weight vector, with the component $w_i$  associated to the graph node $v_i$, $i=1,\ldots , N$;

    \item $\mathrm{diag}(\v{w})\in\R^{N\times N}$ is the diagonal matrix with elements of $\v{w}$ on the diagonal and $\mathbb{I}_N\in\R^{N\times N}$ is the identity matrix. For future reference, we set $\widehat{W}:=\mathrm{diag}(\v{w}) (A + \mathbb{I}_N)$; {the diagonal matrix $\mathrm{diag}(\v{w})$ in \eqref{eq:GI_action_simple} is introduced only for describing in matrix form the multiplication of the $i$-th row of $(A + \mathbb{I}_N)$ by} the weight $w_i$ associated to node $v_i$, for each $i=1,\ldots ,N$;

    \item $\v{\psi}:\R^N\rightarrow\R^N$ is the element-wise application of the activation function $\psi$;
    
    \item $\v{b}\in\R^N$ is the bias vector.
\end{itemize}

From another point of view, equation \eqref{eq:GI_action_simple} is equivalent to the action of a ``constrained'' Fully-Connected (FC) layer where the weights are the same if the connection is outgoing from the same unit, whereas it is zero if two units correspond to graph nodes that are not connected (see \Cref{fig:GIasFC}); more precisely:
\begin{equation}\label{eq:GI_weights_simple}
    \widehat{w}_{ij}=
    \begin{cases}
    w_i\,,\quad & \text{if }a_{ij}\neq 0 \text{ or }i=j\\
    0\,,\quad & \text{otherwise}
    \end{cases}
    \,,
\end{equation}
where $a_{ij},\widehat{w}_{ij}$ denote the $(i,j)$-th element of $A, \widehat{W}$, respectively.

\begin{figure}[htb]
    \centering
    \subcaptionbox{FC representation of GI layer}{
        \resizebox{0.3\textwidth}{!}{
            \begin{tikzpicture}[x=1cm, y=1cm, >=stealth]

% HIDDEN 3

\node [circle,fill=white!50,minimum size=0.75cm] (hidden3-1) at (0,1.5) {$x_1$};
\node [] (w1) at (0.75,1.75) {{\color{cyan!50}$w_1$}};

\node [circle,fill=white!50,minimum size=0.75cm] (hidden3-2) at (0,0.5) {$x_2$};
\node [] (w2) at (0.7,0.8) {{\color{magenta!50}$w_2$}};

\node [circle,fill=white!50,minimum size=0.75cm] (hidden3-3) at (0,-0.5) {$x_3$};
\node [] (w3) at (0.7,-0.8) {{\color{orange!50}$w_3$}};

\node [circle,fill=white!50,minimum size=0.75cm] (hidden3-4) at (0,-1.5) {$x_4$};
\node [] (w4) at (0.75,-1.75) {{\color{green!50}$w_4$}};

% CHAR.FUNC.

%\node [] (L) at (1,1.75) {$\mathcal{L}$};

% OUTPUT
\node [] (L) at (5,2.25) {$L^{GI}$};

\node [circle,fill=blue!30,minimum size=0.75cm] (output-1) at (5,1.5) {};
%\draw [->] (output-1) -- ++(1,0);
%\node [] (yhat-1) at (2+3,1) {$(\mathcal{L}(\v{x}))_1 = f(W_{\cdot \, 1}^\top \v{x} + b_1)$};

\node [circle,fill=blue!30,minimum size=0.75cm] (output-2) at (5,0.5) {};
%\draw [->] (output-2) -- ++(1,0);
%\node [] (yhat-1) at (2+3,0) {$\vdots$};

\node [circle,fill=blue!30,minimum size=0.75cm] (output-3) at (5,-0.5) {};
%\draw [->] (output-3) -- ++(1,0);
%\node [] (yhat-1) at (2+3,-1) {$(\mathcal{L}(\v{x}))_d = f(W_{\cdot \, d}^\top \v{x} + b_d)$};

\node [circle,fill=blue!30,minimum size=0.75cm] (output-4) at (5,-1.5) {};

% EDGES

\draw [->,color=cyan!100] (hidden3-1) -- (output-1);
\draw [->,color=cyan!100] (hidden3-1) -- (output-2);
\draw [->,color=cyan!100] (hidden3-1) -- (output-4);

\draw [->,color=magenta!100] (hidden3-2) -- (output-1);
\draw [->,color=magenta!100] (hidden3-2) -- (output-2);
\draw [->,color=magenta!100] (hidden3-2) -- (output-3);
\draw [->,color=magenta!100] (hidden3-2) -- (output-4);

\draw [->,color=orange!100] (hidden3-3) -- (output-2);
\draw [->,color=orange!100] (hidden3-3) -- (output-3);
\draw [->,color=orange!100] (hidden3-3) -- (output-4);

\draw [->,color=green!100] (hidden3-4) -- (output-1);
\draw [->,color=green!100] (hidden3-4) -- (output-2);
\draw [->,color=green!100] (hidden3-4) -- (output-3);
\draw [->,color=green!100] (hidden3-4) -- (output-4);

\end{tikzpicture}
            }
    }
    \subcaptionbox{Weight matrix of GI layer}{\input{Figures/GIlayer_WeightMatrix.tex}}
    \caption{Visual representation of a GI layer as a ``constrained'' FC layer (subfigure ($a$)), with weight matrix defined by \eqref{eq:GI_weights_simple} (subfigure ($b$)). This figure is based on the same graph illustrated in \Cref{fig:ginnfilter}}
    \label{fig:GIasFC}
\end{figure}

Layers characterized by \eqref{eq:GI_action_simple} can be generalized to read any arbitrary number $K\geq 1$ of input features per node and to return any arbitrary number $F\geq 1$ of output features per node.
Then, the action of a general GI layer is a function $\mathcal{L}^{GI}: \R^{N\times K}\rightarrow\R^{N\times F}$. Additionally, pooling and mask operations can be added. For more details about GI layers, see \cite{GINN,EWGINN,sparseGIlayers}.

Now, we point the attention of the reader to the fact that the GI layer definition is very general and does not ask for a non-weighted graph. Indeed, one of the main advantages of using a GI layer is that it exploits the zeros of the adjacency matrix $A$, ``pruning'' the connections between layer units according to the graph edges; therefore, the zero elements of $A$ are in the same position independently on the usage of weights for the graph edges.

The observation above is important to build GINNs with respect to the WSGG of a given sparse grid. In particular, we prefer a weighted graph because we want to emphasize the difference between connections according to the distance of the nodes in $\R^n$. Then, using a weighted graph for the GINN, we have that the action of the GINN's weights is rescaled with respect to the non-trainable edge weights $\omega_{ij}\in (0, 1]$ (see \Cref{def:wgraph_sg}).

% # ------------------------------------------------- #
% # ------------------------------------------------- #
% # ------------------------------------------------- #
% # ------------------------------------------------- #

\section{Numerical Experiments}\label{sec:num_exp}

In this section, we apply the proposed NN-based discontinuity detection method (see \Cref{alg:sg_nn_alg_detection_parallel}) on various test functions, both using MLP-based and GINN-based detectors (for details on architectures, see \ref{sec:NNarch_archetypes}). We recall that our NN-based detectors detect the troubled points reading only the function evaluations at the sparse grid's points; then, they can be applied to any sparse grid (similar to the one used for building the NN model) and to any discontinuous function with domain dimension equal to the one of the sparse grid (see \Cref{rem:no_coord_detectors}) or, {in principle}, with smaller domain dimension (see \Cref{rem:disc_detectors_dimensions}).

We show results for: $i$) dimension $n=2$, with application also to the boundary detection problem for images (see \Cref{sec:2dim_detection}); $ii$) dimension $n=4$, in order to prove the potentialities of the method for finding discontinuities also in dimensions higher than 3 (see \Cref{sec:4dim_detection}). In both case ($i$) and  case ($ii$), we show the results obtained using just one set of hyper-parameters and one training configuration of the NNs, and one sparse grid; these choices have been made after a preliminary exploration of the learning problem. However, we deserve to future work a study concerning the optimal choice of hyper-parameters, training options and sparse grid characteristics.

After the numerical experiments, we discuss the potentiality and the cost of the proposed method into a dedicated subsection (see \Cref{sec:proscons}).

\subsection{Two-dimensional Discontinuity Detection}\label{sec:2dim_detection}

We start the numerical experiments testing our method on 2-dimensional test functions; we will use the notation introduced after equation \eqref{eq:multiindex_rules} for describing the sparse grids. 

Let $\mathcal{S}_2$ be the 2-dimensional sparse grid $\mathcal{S}_2 :=\mathcal{I}_{\rm sum}(6)$ ($N=65$ points, see Figures \ref{fig:sg_cross_2d_noedges} and \ref{fig:graph_from_sg_firstexample}). Then, we randomly generate a set $\mathcal{G}_2$ of $Q=600$ piece-wise continuous functions characterized by a linear, spherical, and polynomial {discontinuity interface} identified by the zero-level set of given functions (see \ref{sec:synth_data_creation} for details); specifically, we have 200 functions for each type. Then, we build the synthetic dataset of pairs $(\v{g}',\v{p})$ from $\mathcal{G}_2$, running \Cref{alg:sg_alg_detection} and detecting the zero-level sets containing the discontinuity interfaces of the functions in $\mathcal{G}_2$ (we use {the} detector $Z^{(150)}$ described in \ref{sec:zerolevdet}). {From now on, for the sake of simplicity, we will also refer to the discontinuity interfaces as ``\emph{cuts}''.}

However, most of the detected/generated pairs $(\v{g}',\v{p})$ are characterized by null vectors $\v{p}$ (i.e., are continuous-region samples); they are in a number $D_0$ that is much larger than the number $D_1$ of samples with non-null vector $\v{p}$. Therefore, in order to have a more balanced dataset, we randomly keep $D_0'$ pairs out of $D_0$ and discard the others, where
\begin{equation*}\label{eq:kept_allzero_cases}
    D_0':=\max_{i=1,\ldots ,N}D_{1}^{(i)}\,,
\end{equation*}
and $D_{1}^{(i)}$ denotes the number of generated samples where $\v{p}$ has exactly $i$ non-null values. We denote by $\mathcal{D}_2$ the final dataset of $(\v{g}',\v{p})$ samples obtained.

Given $\mathcal{D}_2$, we build an MLP-based detector and a GINN-based detector and we train them with the following training options:
\begin{itemize}
    \item Dataset split: the test set cardinality $|\mathcal{P}_2|$ is $30\%$ of $|\mathcal{D}_2|$, the training set cardinality $|\mathcal{T}_2|$ is $80\%$ of $|\mathcal{D}_2|-|\mathcal{P}_2|$, and the validation set cardinality is $|\mathcal{V}_2|=|\mathcal{D}_2 |-|\mathcal{P}_2|-|\mathcal{T}_2|$;
    \item Loss function: loss \eqref{eq:my_loss_func}, with ${\mu_0}=0.5$ and ${\mu_1}=1.5$;
    \item Mini-batch size: 64;
    \item Optimizer: Adam \cite{Kingma2015_ADAM} (learning rate $\epsilon=0.001$, moment decay rates $\beta_1=0.9$, $\beta_2=0.999$);
    \item Learning rate decay: reduce on plateau \cite{ReduceLRPlateau_TF} (factor $0.75$, $7$ epochs of patience);
    \item Regularization: early-stopping method with best-weights restoration \cite{EarlyStopping_TF} (35 epochs of patience).
    
\end{itemize}
Concerning the NN-based detectors, we use a \emph{leaky relu} activation function \cite{SURVEYFACTIVATIONS_Apicella2021}, a \emph{Glorot Normal} initialization for the weights \cite{Glorot2010_GLOROTunifANDnormal}, and (for the GINN-based detector) $F=15$ filters for the hidden layers{; both MLPs and GINNs are characterized by layers of $N=|\mathcal{S}|$ units and by residual blocks (see \cite{He2016_ResidualNN}) for exploiting depth in the models, depth that is proportional to the diameter of the SGG. For more details about the model's architectures, see \ref{sec:NNarch_archetypes}}. 

At the end of the training, we measure the loss and, for a better understanding, the Mean Absolute Error (MAE) of these models on the test set $\mathcal{P}_2$ (see \Cref{tab:testset_perf_2dims}). Looking at the values in the table, we see that the GINN's prediction performances are slightly better than the MLP's ones, at a reasonable cost in terms of trainable parameters. In particular, there are more trainable parameters in the GINN model because {the GI hidden layers are endowed with $F=15$ filters; nonetheless, thanks to the sparsity of the adjacency matrix of the SGG, the number of trainable weights in the GINN is much smaller than $15$ times the number of MLP's weights.}

\begin{table}[htb!]
    \centering
    \begin{tabular}{|c|c||c|c|}
    \hline
        NN model & \# trainable param.s & Loss function value & MAE\\
        \hline
        \hline
        MLP & $52\,910$ & 4.8678 & 0.0480\\
        \hline
        GINN & $173\,880$ & 3.5396 & 0.0393\\
        \hline
    \end{tabular}
    \caption{Two-dimensional case. Performances on $\mathcal{P}_2$ of the MLP-based and GINN-based detectors.}
    \label{tab:testset_perf_2dims}
\end{table}

Now, we test the discontinuity detection abilities of these NN models if used as discontinuity detectors together with the sparse grid-based algorithm proposed (see \Cref{alg:sg_nn_alg_detection_parallel}). In particular, we test them on four piece-wise continuous functions defined on $\Omega_2=[-1,1]^2$ (see \Cref{fig:2D_testfuncs}): $i$) a function with circular cut; $ii$) a function with Legendre polynomial cut; $iii$) a function with sinusoidal cut; $iv$) a function with one elliptic cut and two bow-shaped cuts. We point the attention of the reader to the fact that functions ($i$) and ($ii$) are of the same type of functions used to generate the synthetic dataset $\mathcal{D}_2$. On the other hand, functions ($iii$) and ($iv$) are functions of a different type with respect to the ones in $\mathcal{G}_2$; then, they are somehow ``new'' to the NN models and good performances on these functions {witness} good generalization abilities of the NNs. Moreover, looking at \Cref{fig:2D_testfuncs}, we can observe that test functions ($i$), ($ii$), and ($iv$) have discontinuity interfaces characterized by parts where the discontinuity jump is approximately zero, i.e., where the interface is not of (evident) discontinuity; then, they are important to understand how much sensitive the NN-based detectors are in identifying discontinuities.

\begin{figure}[htb!]
    \centering
    \subcaptionbox{Function ($i$)}{\includegraphics[trim=25 20 50 25,clip,width=0.24\textwidth]{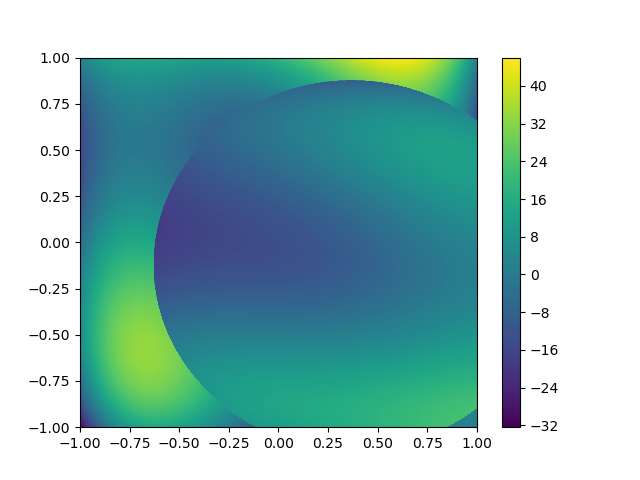}}
	\subcaptionbox{Function ($ii$)}{\includegraphics[trim=25 20 50 25,clip,width=0.24\textwidth]{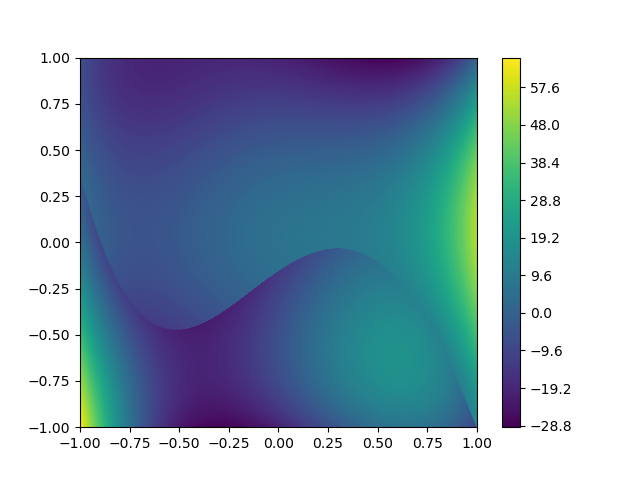}}
    \subcaptionbox{Function ($iii$)}{\includegraphics[trim=25 20 50 25,clip,width=0.24\textwidth]{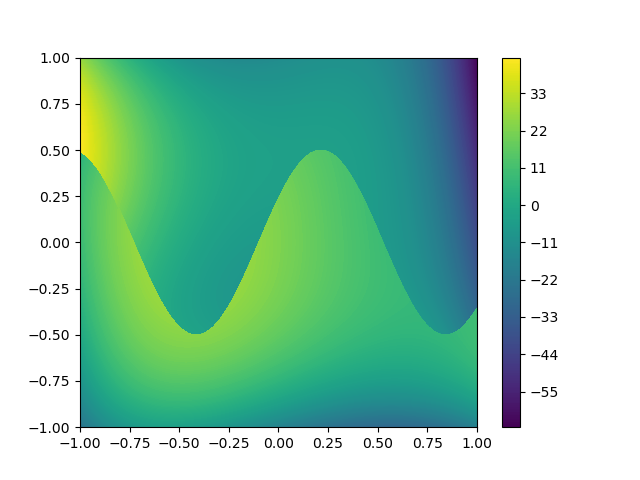}}
    \subcaptionbox{Function ($iv$)}{\includegraphics[trim=25 20 50 25,clip,width=0.24\textwidth]{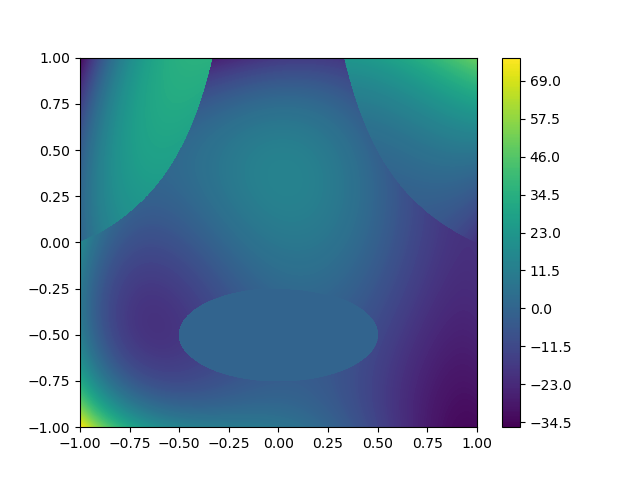}}
    \caption{Top view of test functions ($i$)-($iv$), function values given by colormaps. From left to right: circular cut, Legendre polynomial cut, sinusoidal cut, and ellipsis plus two bow-shaped cuts.}
    \label{fig:2D_testfuncs}
\end{figure}

The performances of the NN-based detectors are measured according to a True Positive Rate (TPR). Let $T$ be the set of final troubled points detected by the method and let $\Lambda_{\min}$ be the minimum edge length used as stopping criterion (see Algorithms \ref{alg:sg_alg_detection} and \ref{alg:sg_nn_alg_detection_parallel}); then, the TPR with respect to each test functions is computed {as follows}:
\begin{enumerate}
    \item let $\zeta$ be the zero-level set of the function characterizing the discontinuity interface of the test function. Let $\widetilde{\mathcal{S}}_2$ be the 2-dimensional sparse grid {represented by the blue dots in \Cref{fig:TPRandTNRprobs_a};}
    
    \item for each $\v{x}\in T$, center on $\v{x}$ a sparse grid $\widetilde{\mathcal{S}}'_2(\v{x})\simeq \widetilde{\mathcal{S}}_2$ with box of edge length $\Lambda_{\min}$. We denote by $\widetilde{G}(\v{x})$ the SGG of $\widetilde{\mathcal{S}}'_2(\v{x})$ (see \Cref{fig:TPRandTNRprobs_a});
    
    \item $\v{x}$ is considered as a true troubled point if $\zeta$ intersects at least one edge of $\widetilde{G}(\v{x})$ (see \Cref{fig:TPRandTNRprobs_b});
    
    \item compute TPR as $(\textit{\# true troubled points})/|T|$.
\end{enumerate}

\begin{remark}[Focus on TPR]\label{rem:focus_TPR}
For measuring the detection quality of our method,
we do not consider a True Negative Rate (TNR) because a definition of \emph{true non-troubled point} based on the distance from $\zeta$ (like \emph{true trouble points} in TPR) would be misleading, even in the case of exact detectors. 
In fact, \Cref{alg:sg_alg_detection} labels a point $\v{x}_i$ as \emph{non-troubled point} not because it is far from the discontinuity interface, but according to other criteria. Specifically,  $\v{x}_i$ is a \emph{non-troubled point} either because no edge with $\v{x}_i$ as vertex intersects the discontinuity interface (e.g., see the magenta triangle in \Cref{fig:TPRandTNRprobs_b}) or because there exists an edge $e_{ij}=\{\v{x}_i,\v{x}_j\}$, $\norm{\v{x}_i-\v{x}_j} < \Lambda_{\min}$, such that $e_{ij}$ intersects the discontinuity interface but $\v{x}_j$ is closer to the intersection than $\v{x}_i$ (e.g., see orange the triangle in \Cref{fig:TPRandTNRprobs_b}). In both cases, $\v{x}_i$ is marked as \emph{non-troubled point} not due to its distance from the discontinuity interface, but because there were better alternatives to it. To better understand these observations, we report in \Cref{fig:TPRandTNRprobs_b} an example in which two true troubled points are detected; in the figure we highlight two examples of non-troubled points inside the area delimited by the grid $\widetilde{\mathcal{S}}_2$ in the middle of the picture: the orange triangle marks a non-troubled point that has an edge intersecting $\zeta$, and the magenta triangle denotes a non-troubled point without edges that intersect $\zeta$. If we would use these points as centers for the grid $\widetilde{\mathcal{S}}_2$ adopted to check the $\zeta$-intersection of the detected troubled points, the grid would intersect $\zeta$. For this reason a definition of \emph{true non-troubled point} based on the distance from $\zeta$ like the \emph{true trouble points} would be misleading.
\end{remark}

\begin{figure}[htb]
    \centering
    \subcaptionbox{SGG of the sparse grid $\widetilde{\mathcal{S}}_2$\label{fig:TPRandTNRprobs_a}}{
    \includegraphics[width=0.2\textwidth]{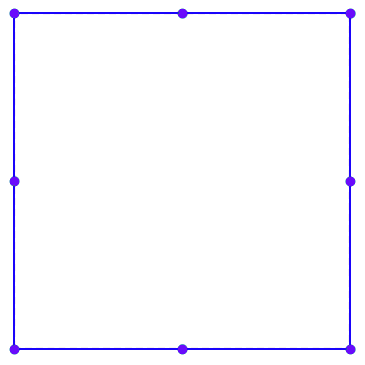}
    \vspace{1.1cm}
    }
    \qquad
    \subcaptionbox{TPR check example\label{fig:TPRandTNRprobs_b}}{\includegraphics[width=0.35\textwidth]{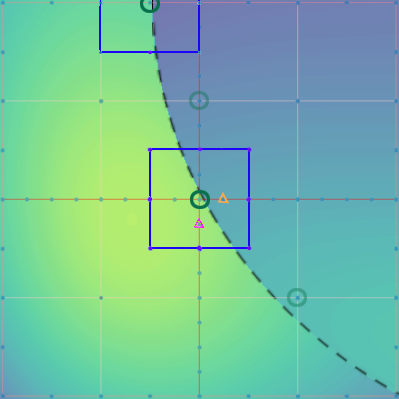}}
    \caption{Left: SGG based on the sparse grid $\widetilde{\mathcal{S}}_2$. Right: use of the SGG for computing the TPR. Red segments are edges of the detector's SGG; green circles are troubled points detected at the current step (dark green for final troubled points and light green for points used as grid centers in the next step); in the $\widetilde{\mathcal{S}}_2$ grid in the middle, orange and magenta triangles highlight two examples of non-troubled points sufficiently near to $\zeta$ (orange triangle if an edge intersects $\zeta$, magenta triangle if no edge intersects $\zeta$, see \Cref{rem:focus_TPR}).}
    \label{fig:TPRandTNRprobs}
\end{figure}

We run \Cref{alg:sg_nn_alg_detection_parallel} on the four test functions using the NN-based detectors and using the same input parameters for all the cases; i.e., the same starting set $L$ of center and edge length pairs and the same $\Lambda_{\min} = 2/2^{(h_{\max}+1)} = 2^{-5}$ for stopping criterion ($h_{\max}=5$, maximum refinement level of the sparse grid). Moreover, since during the procedure some sparse grids may have points outside the domain $\Omega_2$, we modify \Cref{alg:sg_nn_alg_detection_parallel} in order to stop the search if a troubled point is detected outside the domain (i.e., it stops even if the stopping criterion with respect to $\Lambda_{\min}$ is not satisfied).

In \Cref{tab:MLPGINN_2D_results}, we report the TPRs, the number of troubled points detected, and the total number of points visited by the two NN-based discontinuity detectors, with respect to the test functions. Looking at the results reported in this table, we observe that both the NN-based detectors have extremely good TPRs (at least $98\%$). According to the test set performances (see \Cref{tab:MLPGINN_2D_results}), the GINN-based detector is more precise than the MLP-based one (higher TPRs and TPR equal to $100\%$ in three cases over four); moreover, more troubled points are detected. In particular, looking at the positioning of the detected troubled points (Figures \ref{fig:MLPGINN_2D_results_01} and \ref{fig:MLPGINN_2D_results_02}), we can deduce that the higher sensitiveness of the GINN-based detector depends on its ability to detect troubled points even if the discontinuity jumps are very small.

In general, we can claim that the given GINN-based detector is a very good 2-dimensional detector for finding discontinuity interfaces via \Cref{alg:sg_nn_alg_detection_parallel}; indeed, it is able to detect troubled points even in presence of very small discontinuity jumps and its TPRs on the test functions are always approximately $100\%$. On the other hand, the MLP-based detector proves to be a good detector, but it can miss some parts of the discontinuity interfaces due to its lower sensitiveness.

\begin{table}[htb!]
    \centering
    \begin{tabular}{|c|c||c|c|c|}
         \hline
         NN model & Test function & TPR & Num. of troubled points & Num. of visited points \\
         \hline
         \hline
         MLP & ($i$) & $100\%$ & 557 & 11750 %($\sim 108^2$)
         \\
         GINN & ($i$) & $100\%$ & 829 & 13244 %($\sim 115^2$)
         \\
         \hline
         MLP & ($ii$) & $98.29\%$ & 410 & 10864 %($\sim 104^2$)
         \\
         GINN & ($ii$) & $99.84\%$ & 642 & 11745 %($\sim 108^2$)
         \\
         \hline
         MLP & ($iii$) & $99.22\%$ & 639 & 12281 %($\sim 110^2$)
         \\
         GINN & ($iii$) & $100\%$ & 740 & 12046 %($\sim 109^2$)
         \\
         \hline
         MLP & ($iv$) & $99.58\%$ & 717 & 13403 %($\sim 115^2$)
         \\
         GINN & ($iv$) & $100\%$ & 1028 & 15397 %($\sim 124^2$)
         \\
         \hline
    \end{tabular}
    \caption{Results of \Cref{alg:sg_nn_alg_detection_parallel} with respect to test functions ($i$)-($iv$), using NN-based detectors.}
    \label{tab:MLPGINN_2D_results}
\end{table}

\begin{figure}[htb!]
    \centering
    \subcaptionbox{MLP, function ($i$)}{
    \includegraphics[width=0.475\textwidth]{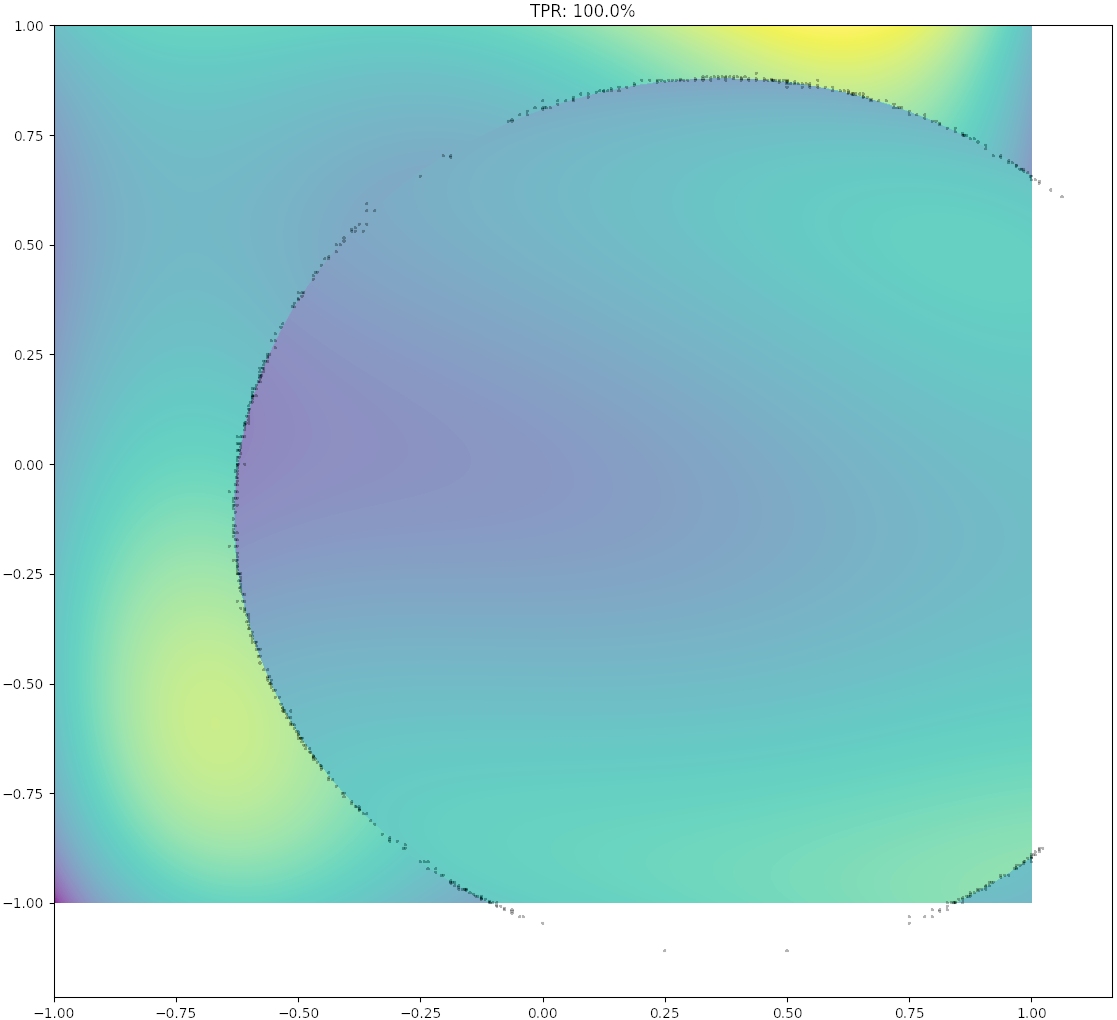}
    }
    \subcaptionbox{GINN, function ($i$)}{
    \includegraphics[width=0.475\textwidth]{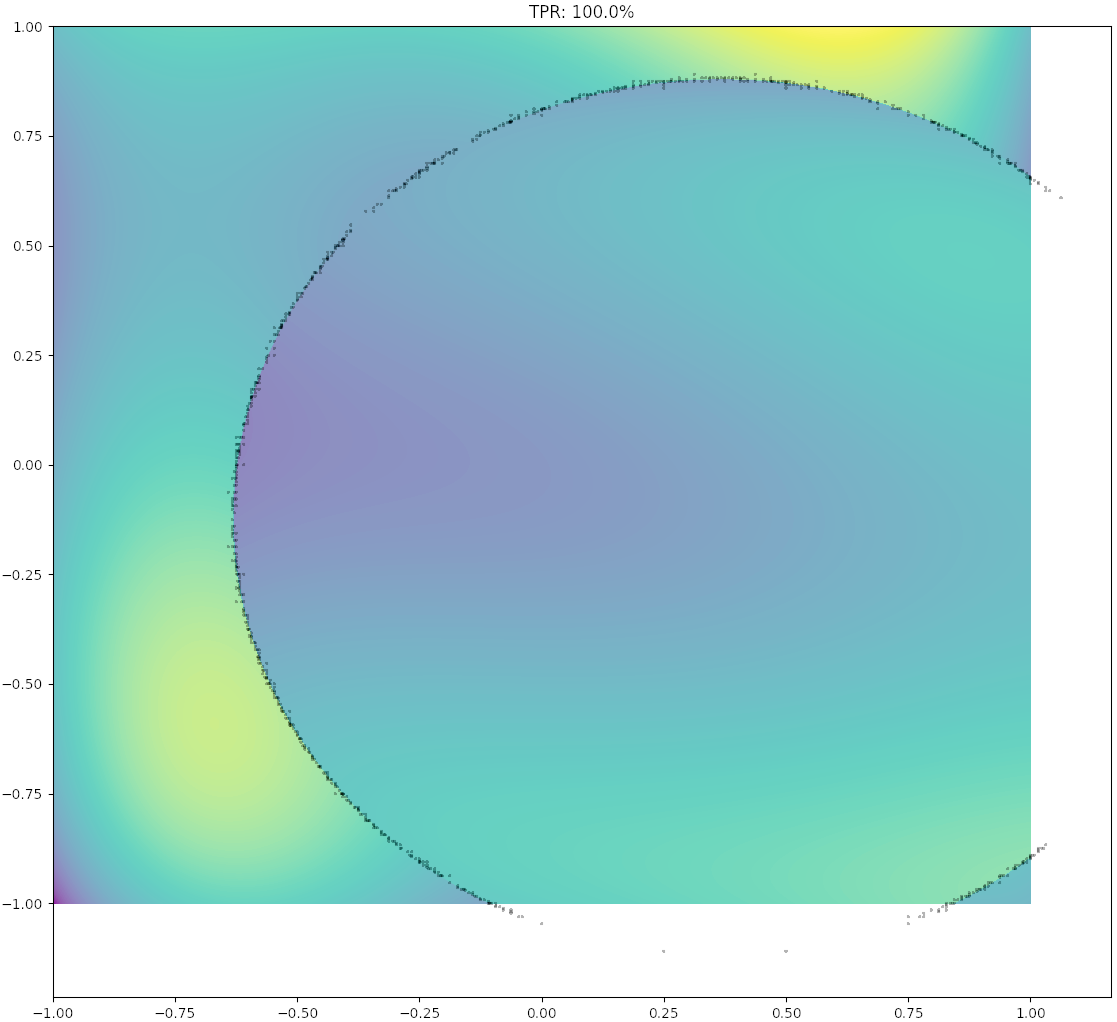}
    }
    \\
    \subcaptionbox{MLP, function ($ii$)}{
    \includegraphics[width=0.475\textwidth]{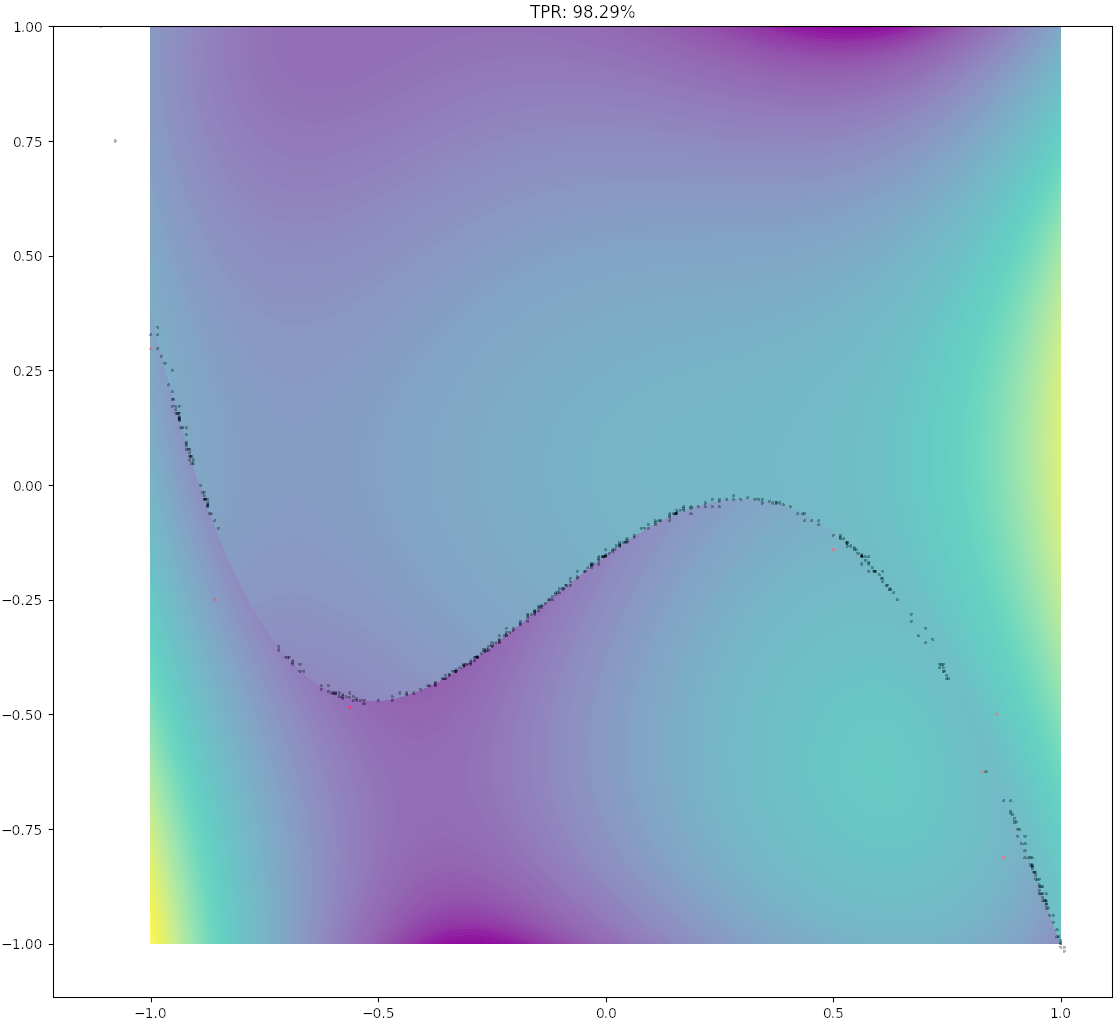}
    }
    \subcaptionbox{GINN, function ($ii$)}{
    \includegraphics[width=0.475\textwidth]{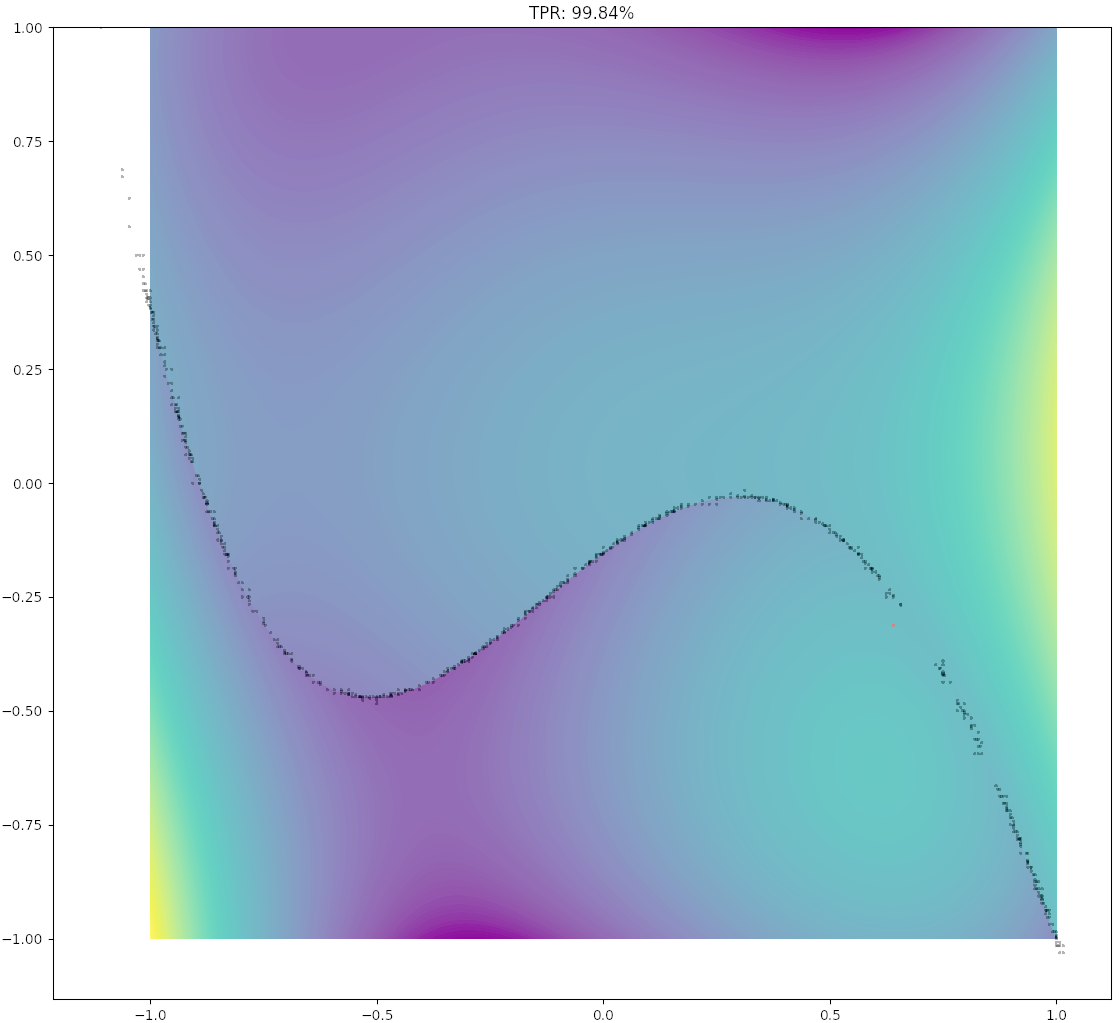}
    }
    \caption{Troubled points detected for the test functions ($i$) and ($ii$) with \Cref{alg:sg_nn_alg_detection_parallel}, using the NN-based detectors. Black dots are true troubled points, red dots (almost invisible for their small number) are false troubled points.}
    \label{fig:MLPGINN_2D_results_01}
\end{figure}
\begin{figure}[htb!]
    \subcaptionbox{MLP, function ($iii$)}{
    \includegraphics[width=0.475\textwidth]{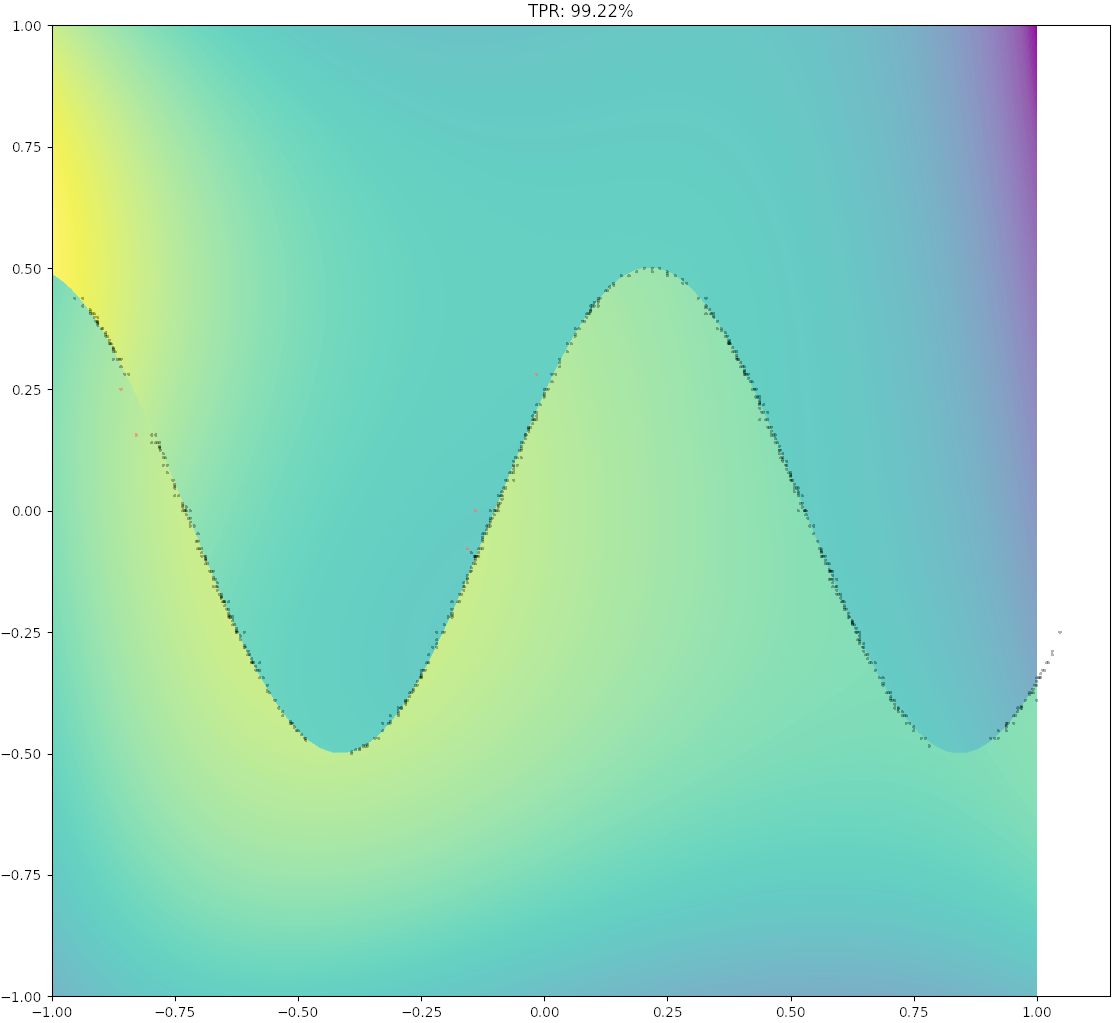}
    }
    \subcaptionbox{GINN, function ($iii$)}{
    \includegraphics[width=0.475\textwidth]{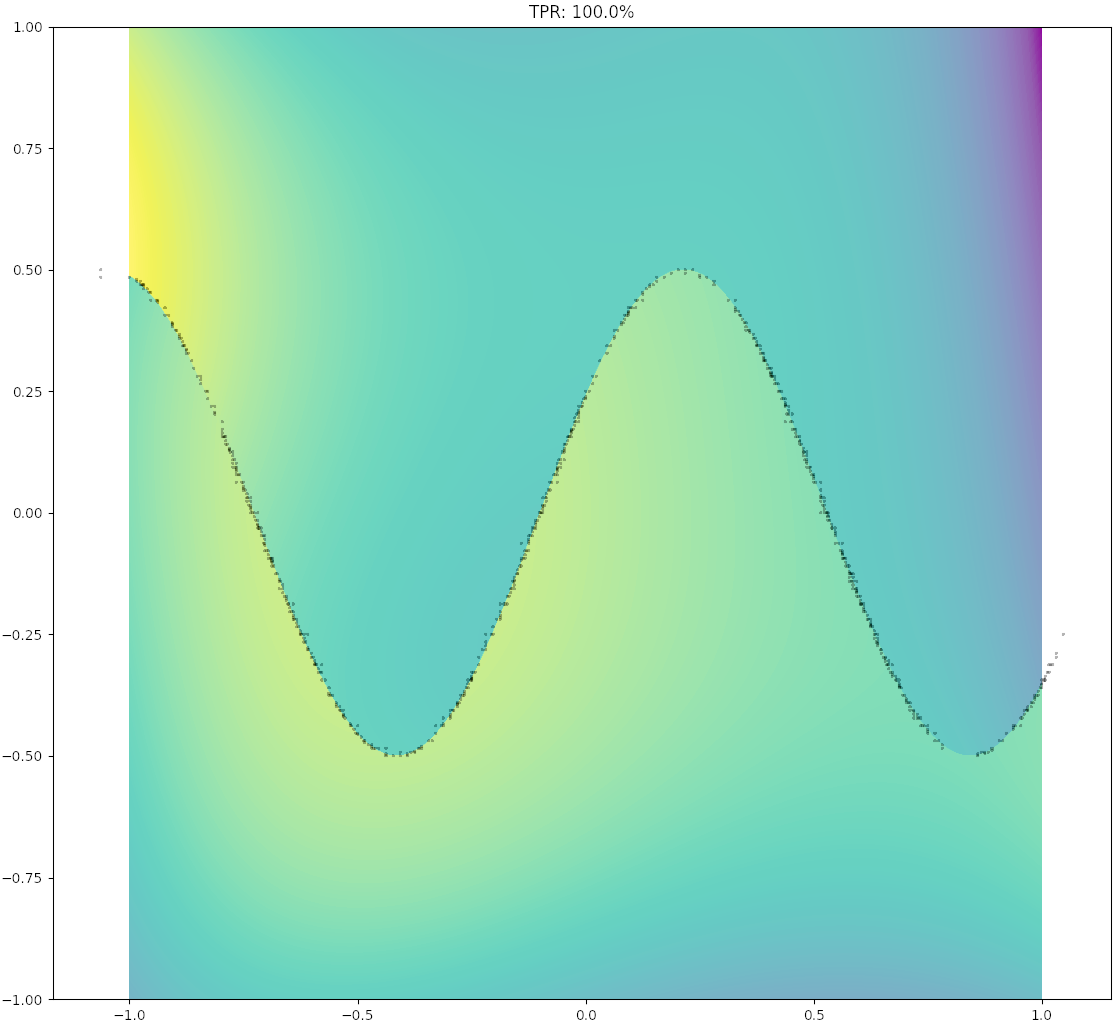}
    }
    \\
    \subcaptionbox{MLP, function ($iv$)}{
    \includegraphics[width=0.475\textwidth]{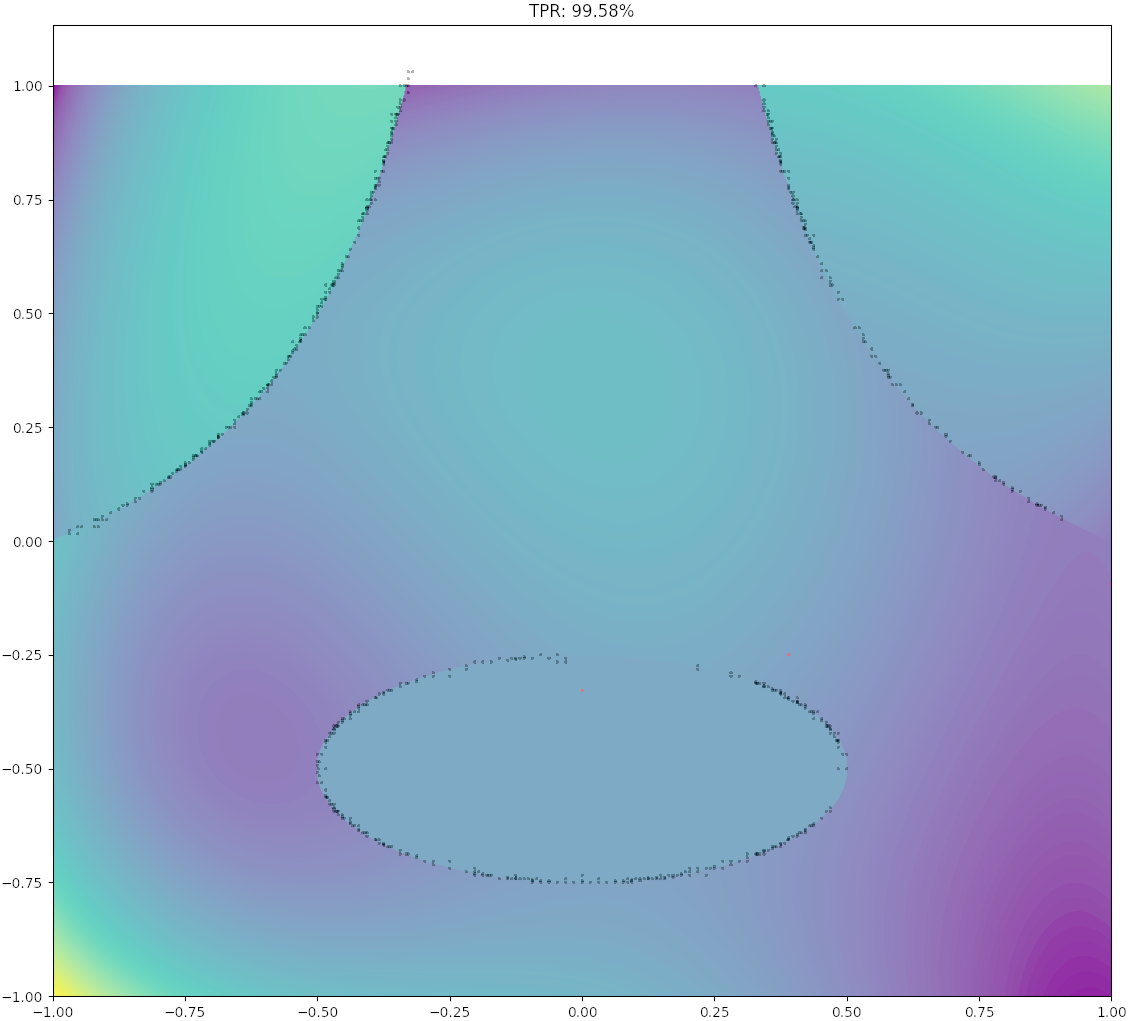}
    }
    \subcaptionbox{GINN, function ($iv$)}{
    \includegraphics[width=0.475\textwidth]{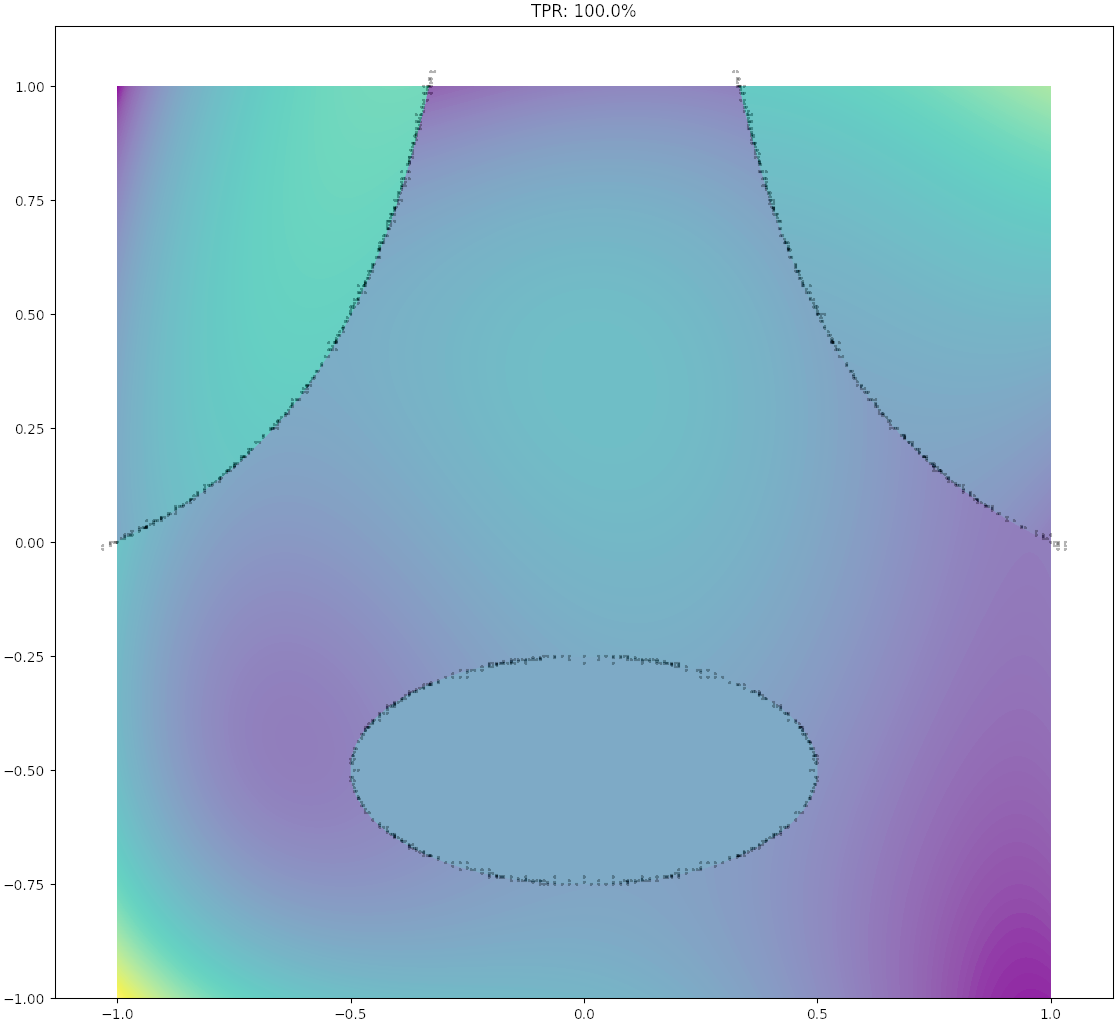}
    }
    \caption{Troubled points detected for the test functions ($iii$) and ($iv$) with \Cref{alg:sg_nn_alg_detection_parallel}, using the NN-based detectors. Black dots are true troubled points, red dots (almost invisible for their small number) are false troubled points.}
    \label{fig:MLPGINN_2D_results_02}
\end{figure}

\subsubsection{Edge Detection Problem in Images}\label{sec:edge_detection}

Black and white images can be modeled as functions $g:\Omega\subset\R^2\rightarrow[0,1]$, where $g(\v{x})$ is the grayscale level of the pixel, for each pixel $\v{x}\in\Omega$. Then, two-dimensional detectors can be used for detecting edges in pictures, since edges are discontinuity interfaces of $g$. Analogously to \cite{Wang2022}, we test our NN-based two-dimensional detectors on the edge detection problem of the Shepp-Logan phantom \cite{SheppLoganPhantom_1974}, varying the image resolution. In particular, let $M=(M_{ij})\in [0,1]^{r\times r}$ be the matrix representing the Shepp-Logan phantom image with resolution $r\times r$ pixels; then, we apply \Cref{alg:sg_nn_alg_detection_parallel} modeling the image as a function $g_r:\R^2\rightarrow [0,1]$ such that:
\begin{equation*}\label{eq:SLphantom_R}
    g_r(\v{x}) = 
    \begin{cases}
    M_{[x_1]+1,[x_2]+1}\,,\quad & \text{if }0\leq [x_1],[x_2]\leq r-1\\
    0\,,\quad & \text{otherwise}
    \end{cases}
    \,,
\end{equation*}
where $[x]$ is the rounding to the nearest integer of $x$, for each $x\in\R$.

In \Cref{fig:SLphantom}, we report the results obtained on $g_r$ using \Cref{alg:sg_nn_alg_detection_parallel} and the two NN-based discontinuity detectors, for each $r=
%256,
512,1024$. Looking at these figures, we observe that the edge detection performances of the NN-based method increase with the resolution, similarly to the CNN detectors of \cite{Wang2022}; however, we see that the GINN-based detector shows very good performances even in the resolution $512\times 512$
while the MLP-based detector obtains acceptable results (despite some imperfections) only for resolution $1024\times 1024$ (see \Cref{fig:SLphantom}). Therefore, we have further experimental evidence of the better detection abilities of the GINN-based detector with respect to the MLP-based one.

A very interesting characteristic of our method applied to the edge detection problem can be observed in the rightmost column of \Cref{tab:SLphantom}, which reports the number of troubled points detected and the total number of points visited by the NN-based detectors, for each image resolution. Indeed, we observe that increments in the resolution do not imply a considerably higher number of points visited by the method. This characteristic is very important because it means that increments in the image resolution guarantee better results at almost the same cost; in particular, 
% for the last two cases, 
we have that the function evaluations done by the algorithm are (approximately) only $330^2$, that is much smaller than the number of pixels available ($512^2$ and $1024^2$, respectively).

\begin{figure}[htbp!]
    \centering
    \subcaptionbox{MLP, $512\times 512$}{
    \includegraphics[width=0.475\textwidth]{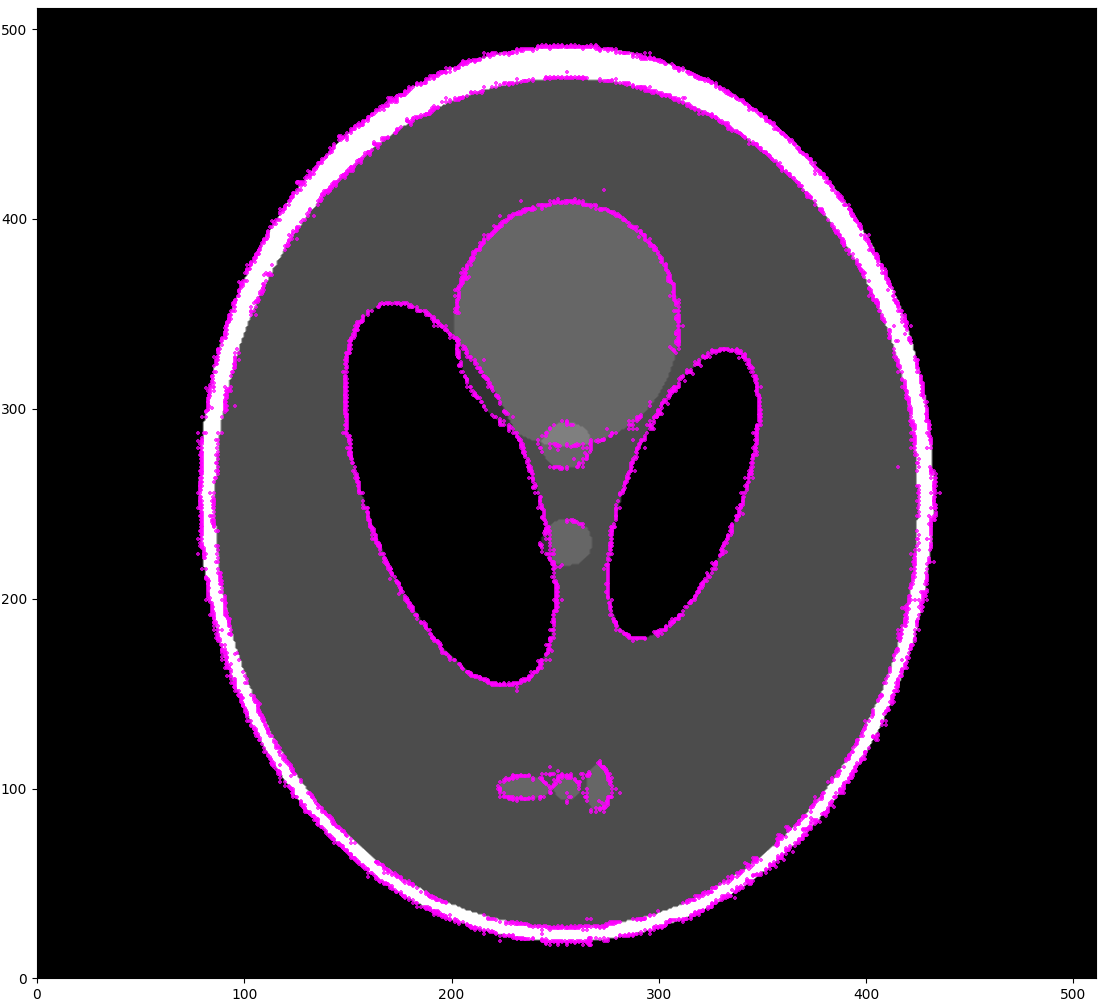}
    }
    \subcaptionbox{GINN, $512\times 512$}{
    \includegraphics[width=0.475\textwidth]{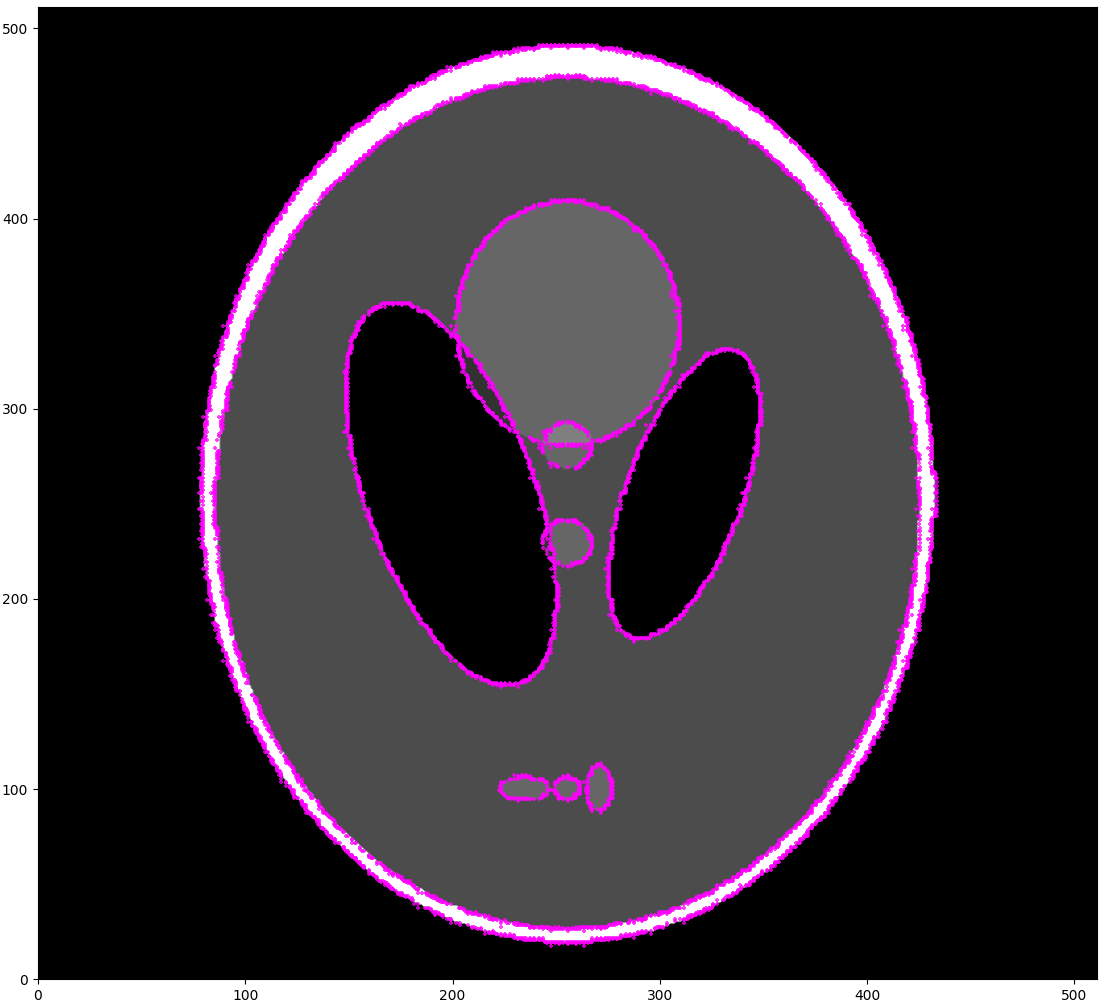}
    }
    \\
    \subcaptionbox{MLP, $1024\times 1024$}{
    \includegraphics[width=0.475\textwidth]{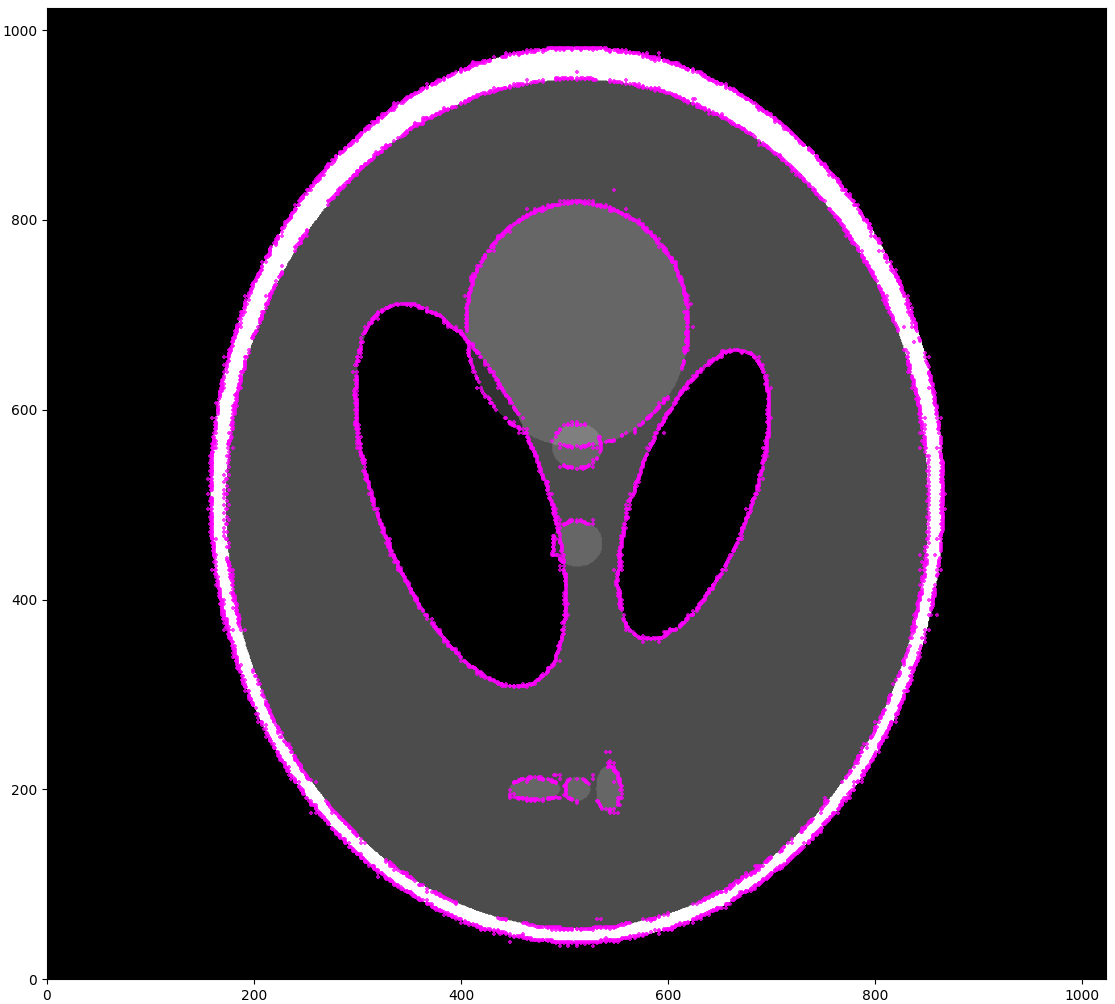}
    }
    \subcaptionbox{GINN, $1024\times 1024$}{
    \includegraphics[width=0.475\textwidth]{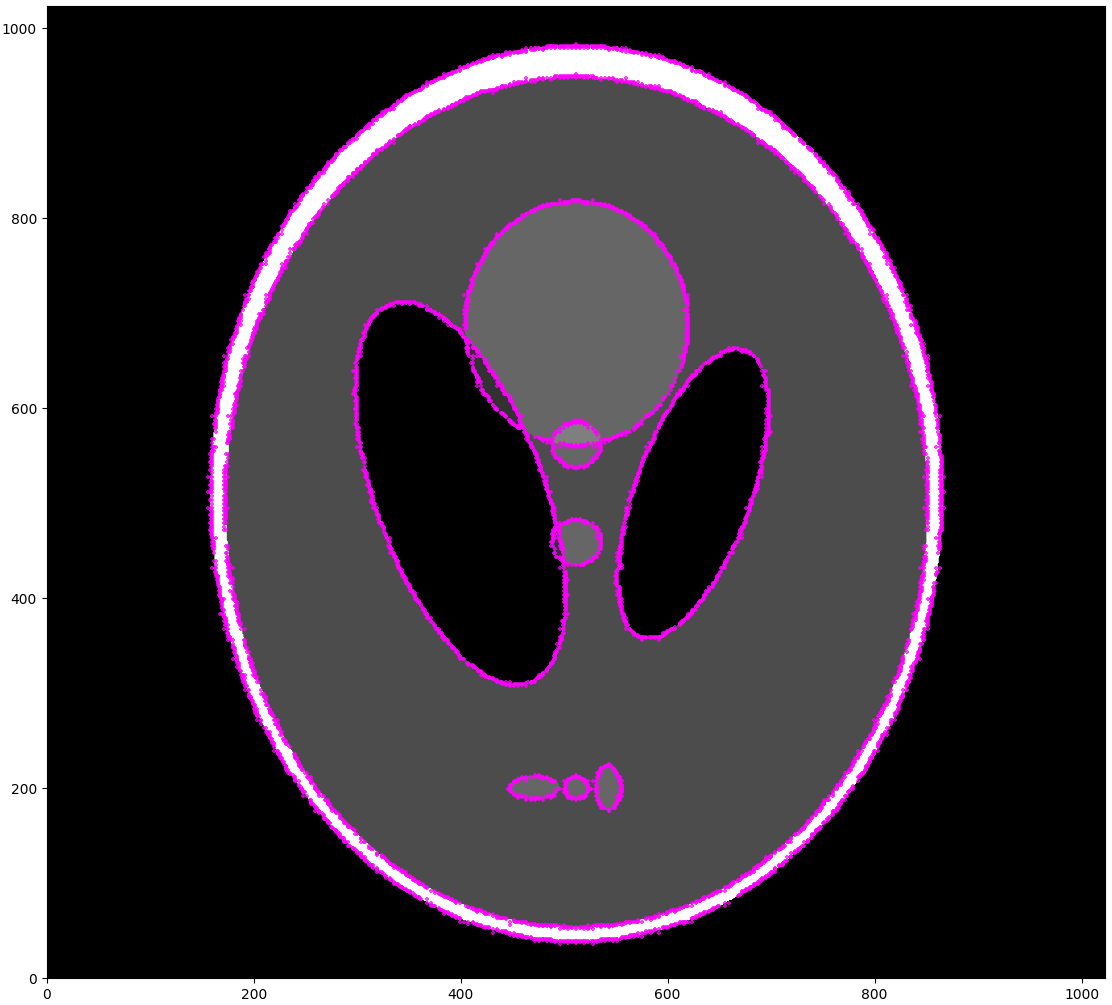}
    }
    \caption{Troubled points detected for the Shepp-Logan phantom (at different resolutions) with \Cref{alg:sg_nn_alg_detection_parallel}, using the NN-based detectors. Magenta dots are the troubled points.}
    \label{fig:SLphantom}
\end{figure}

\begin{table}[htb!]
    \centering
    \begin{tabular}{|c|c||c|c|}
         \hline
         NN model & Image Resolution & Num. of troubled points & Num. of visited points \\
         \hline
         \hline
         MLP & $512\times 512$ & 5585 & 108425 ($\sim 329^2$)
         \\
         GINN & $512\times 512$ & 6657 & 109903 ($\sim 332^2$)
         \\
         \hline
         MLP & $1024\times 1024$ & 5723 & 108282 ($\sim 329^2$)
         \\
         GINN & $1024\times 1024$ & 6795 & 110058 ($\sim 332^2$)
         \\
         \hline
    \end{tabular}
    \caption{Results of \Cref{alg:sg_nn_alg_detection_parallel} with respect to the Shepp-Logan phantom, at different resolutions, using NN-based detectors.}
    \label{tab:SLphantom}
\end{table}

\subsection{Four-dimensional Discontinuity Detection}\label{sec:4dim_detection}

Following the considerations made before, in this subsection we focus on a GINN-based detector, as in general GINNs perform better than MLPs. The sparse grid we consider is $\mathcal{S}_4:=\mathcal{I}_{\rm sum}(8)$ in $\R^4$ ($N=401$ points). Then, we consider a randomly generated set $\mathcal{G}_4$ of 750 functions, again with linear, spherical, and polynomial cut (250 each). We use a larger set of functions (with respect to \Cref{sec:2dim_detection}) because we use a less precise approximation of the exact detector for the creation of the synthetic dataset ($Z^{(3)}=\left . Z^{(t+1)}\right|_{t=2}$, see \ref{sec:zerolevdet} for details); indeed, we prefer a less precise detector to make the dataset creation process less expensive, even if at the cost of a less detailed dataset.

Given the sparse grid and $\mathcal{G}_4$, we proceed analogously to \Cref{sec:2dim_detection}: we build a GINN-based detector with respect to the SGG of $\mathcal{S}_4$, we generate a synthetic dataset $\mathcal{D}_4$, and we train the GINN-based detector. The performances of this detector on the test set $\mathcal{P}_4$ are reported in \Cref{tab:testset_perf_4dims}.

\begin{table}[htb!]
    \centering
    \begin{tabular}{|c|c||c|c|}
\hline
        NN model & \# trainable param.s & Loss function value & MAE \\
        \hline
        \hline
        GINN & $1\,263\,540$ & 37.5907 & 0.0585 \\
        \hline
    \end{tabular}
    \caption{Four-dimensional case. Performances on $\mathcal{P}_4$ of the GINN-based detector.}
    \label{tab:testset_perf_4dims}
\end{table}

Given the trained GINN-based four-dimensional detector, we test its abilities running \Cref{alg:sg_nn_alg_detection_parallel} ($\Lambda_{\min} = 2/2^{(h_{\max})} = 2^{-4}$) on a piece-wise continuous four-dimensional test function. Specifically, we use a test function $g_\vartheta:\Omega_4\subseteq\R^4\rightarrow\R$ with domain $\Omega_4=[-1,1]^4$, continuous pieces defined by four-dimensional Legendre polynomials (see \eqref{eq:n_dim_legendre}, \ref{sec:synth_data_creation}), and as discontinuity interface a torus with size depending on $x_4$; i.e., the discontinuity interface of $g_\vartheta$ is characterized by the zero-level set of the function
\begin{equation}\label{eq:enlarging_torus}
    \vartheta(\v{x})= \left(|x_4| - \sqrt{x_1^2 + x_2^2}\right)^2 + x_3^2 - (|x_4|/4)^2\,.
\end{equation}

The results of \Cref{alg:sg_nn_alg_detection_parallel} for this four-dimensional case are very promising. In \Cref{tab:GINN_4D_results}, we report the TPR, the number of troubled points detected, and the total number of points visited by the GINN-based detector; then, looking at this table, we see a TPR of $93.52\%$ for more than half a million of troubled points. Moreover, the total number of points evaluated during the algorithm execution is approximately equal to the number of points of a coarse regular grid of $60$ knots per dimension but the point distribution is more dense near to the discontinuity interface $\vartheta(\v{x})=0$, instead of being equally distributed in $\Omega_4$.

In \Cref{fig:GINN_4D_results}, we report some examples of the troubled points identified by the GINN-based detector for the function $g_\vartheta$. In particular, for fixed values of $x_4$, we show cases where the detector identifies the discontinuity interface almost perfectly with (almost) no false troubled points (Figures \ref{fig:GINN_4D_results_a} and \ref{fig:GINN_4D_results_b}), with few true troubled points (\Cref{fig:GINN_4D_results_c}), and with many false troubled points (\Cref{fig:GINN_4D_results_d}).

\begin{table}[htb!]
    \centering
    \begin{tabular}{|c||c|c|c|}
    \hline
         NN model & TPR & Num. of troubled points & Num. of visited points \\
         \hline
         \hline
         GINN & $93.52\%$ & 606220 & 13139968 ($\sim 60^4$)
         \\
         \hline
    \end{tabular}
    \caption{Results of \Cref{alg:sg_nn_alg_detection_parallel} with respect to test function $g_\vartheta$, using the four-dimensional GINN-based detector.}
    \label{tab:GINN_4D_results}
\end{table}

\begin{figure}[htb!]
    \centering
    \subcaptionbox{$x_4 = -1$\label{fig:GINN_4D_results_a}}{
    \includegraphics[width=0.45\textwidth]{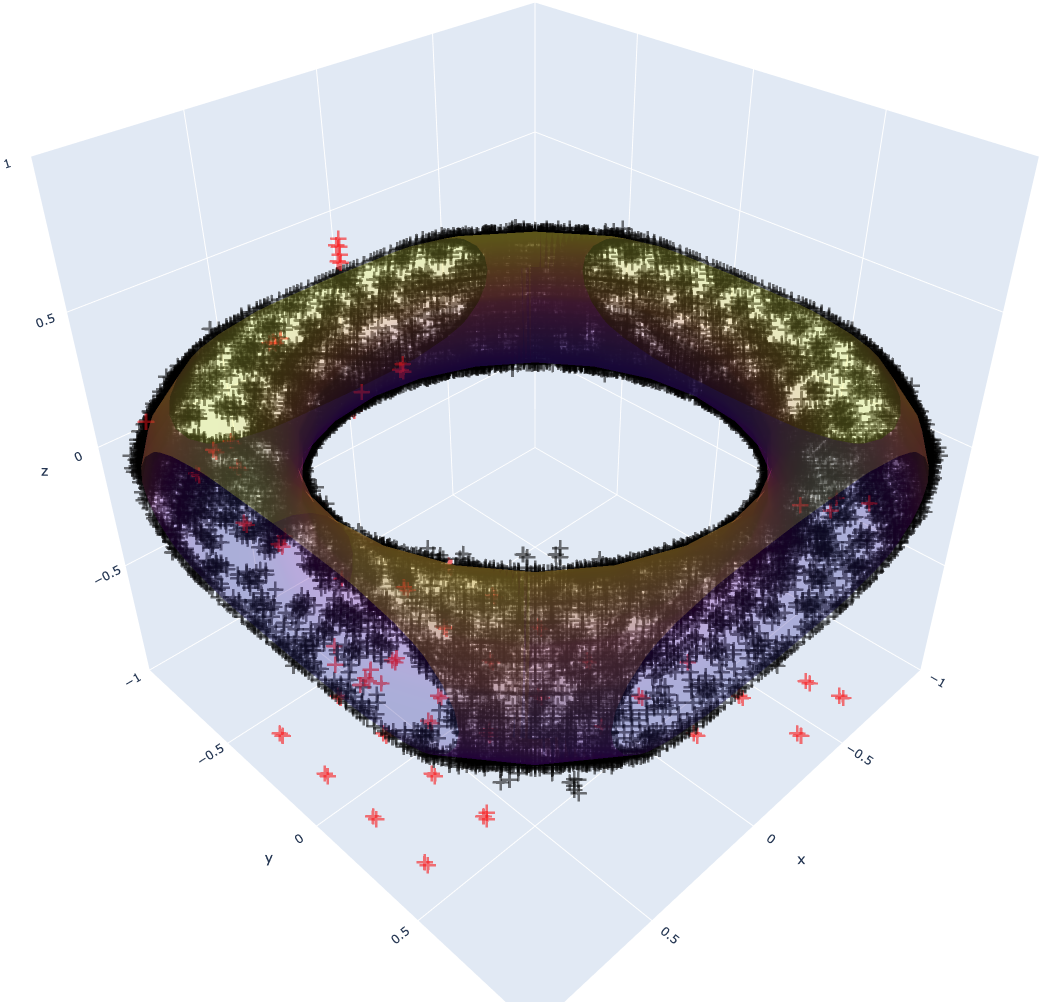}
    }
    \subcaptionbox{$x_4\simeq -0.8$\label{fig:GINN_4D_results_b}}{
    \includegraphics[width=0.45\textwidth]{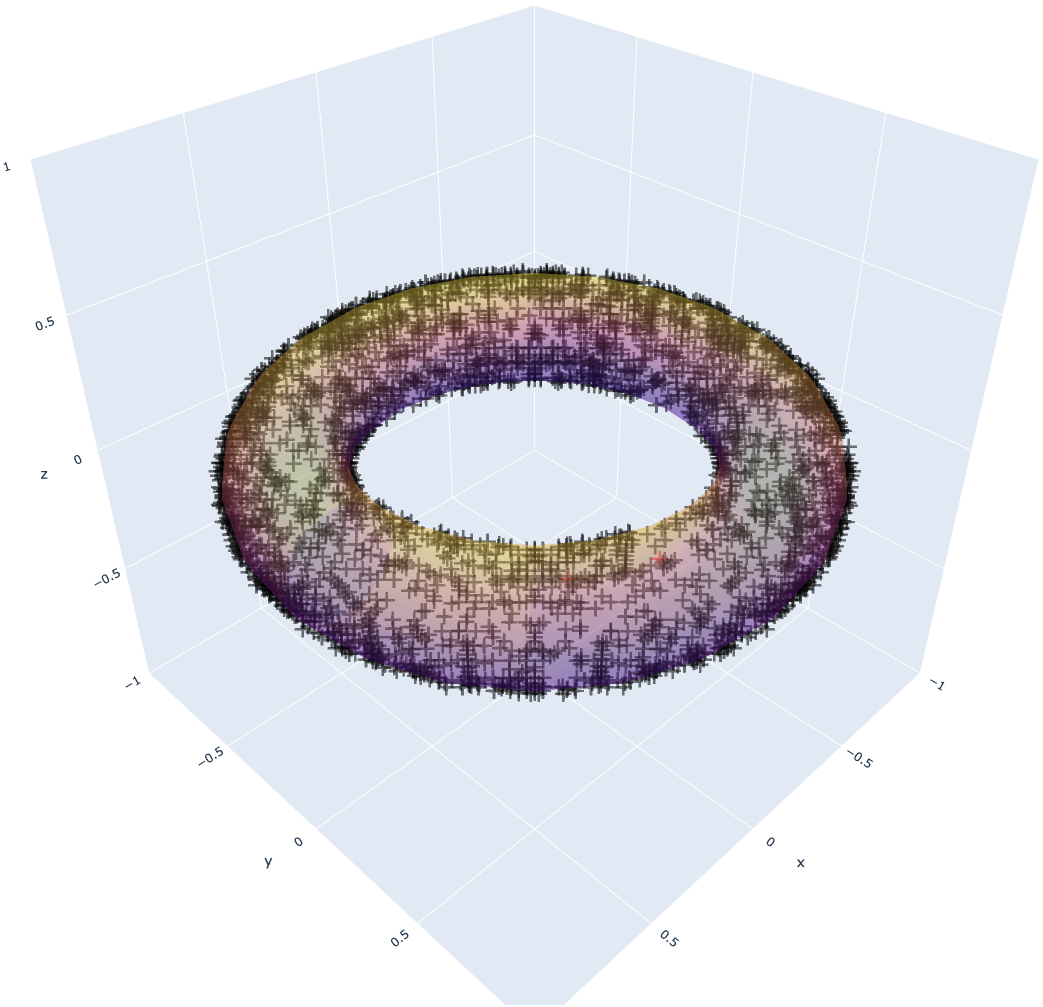}
    }
    \\
    \subcaptionbox{$x_4\simeq -0.6$\label{fig:GINN_4D_results_c}}{
    \includegraphics[width=0.45\textwidth]{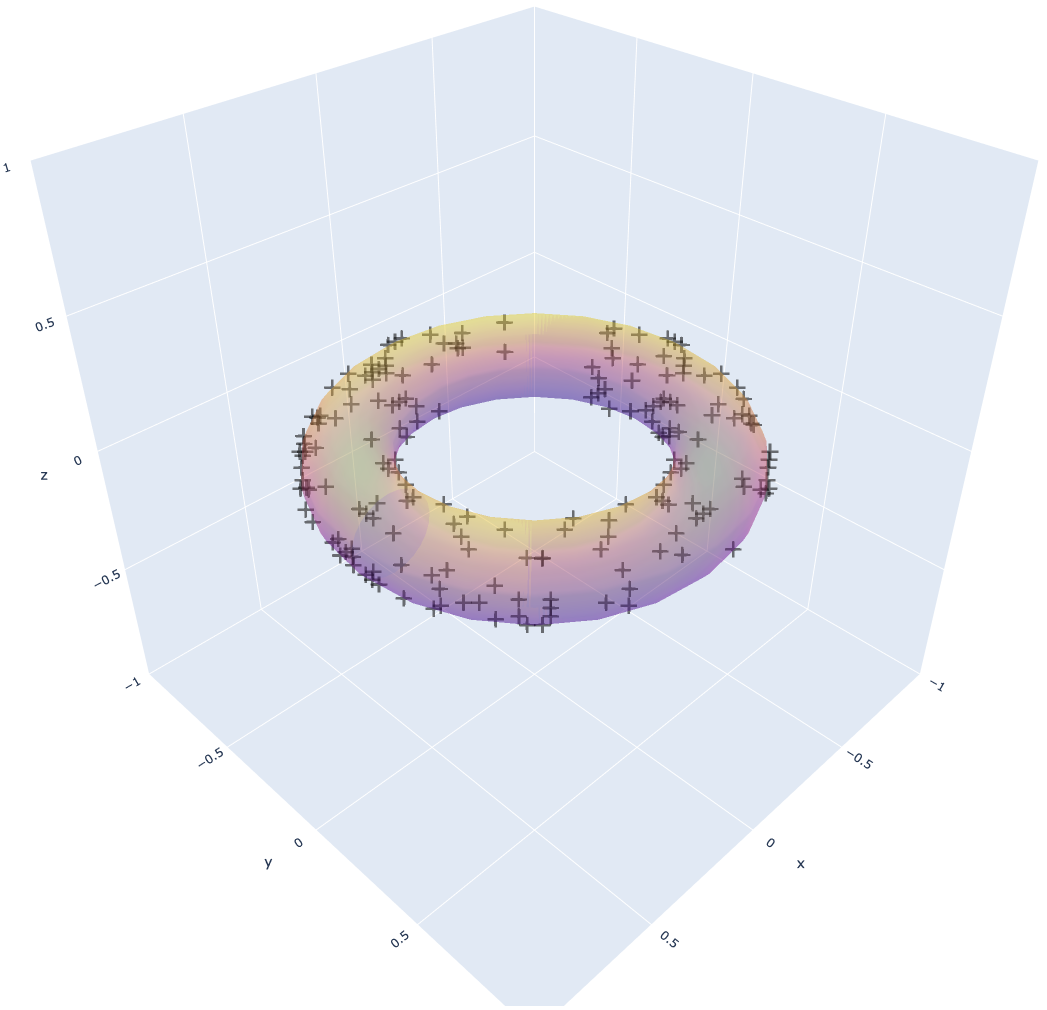}
    }
    \subcaptionbox{$x_4 = -0.5$\label{fig:GINN_4D_results_d}}{
    \includegraphics[width=0.45\textwidth]{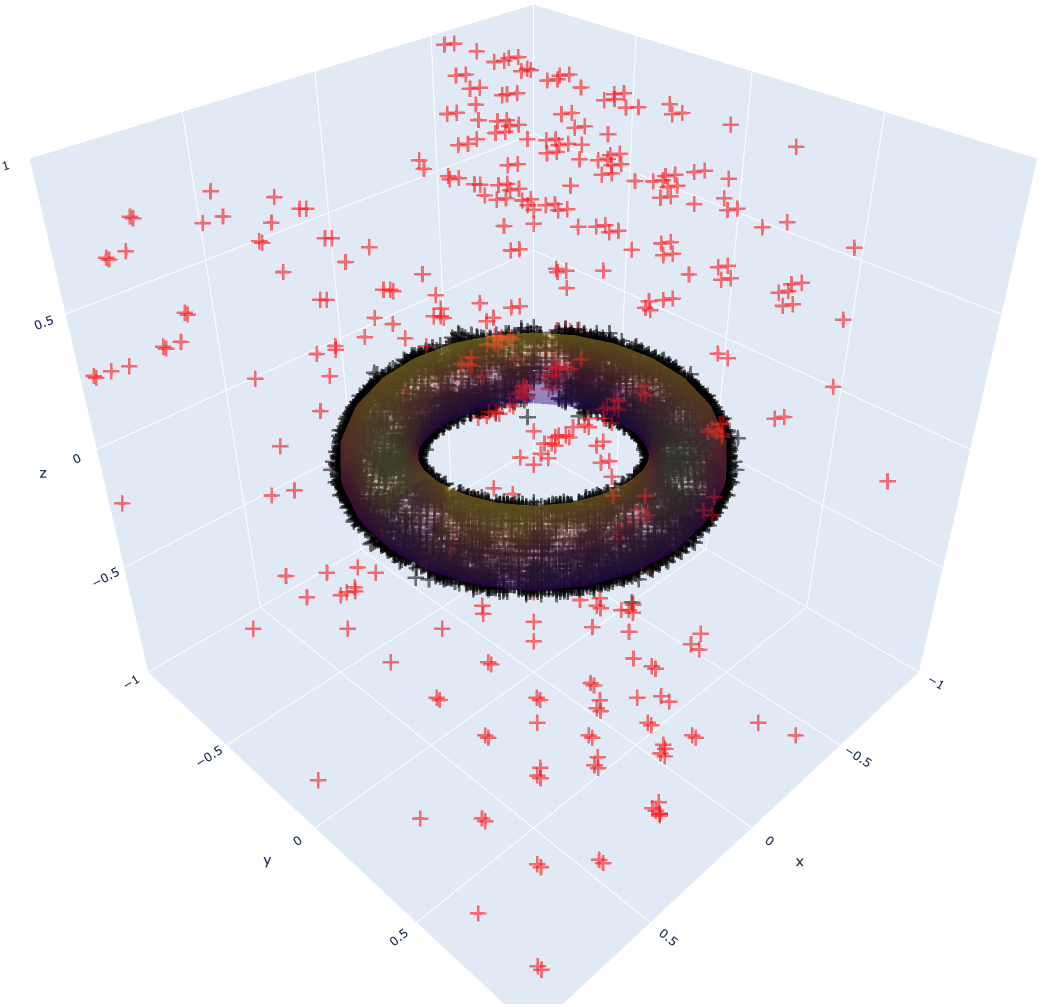}
    }
    \caption{Visual examples of the troubled points detected by the GINN-based detector for $g_\vartheta$. Each picture is a visualization in $\R^3$ of $g_\vartheta$ for a fixed value of $x_4$. The colored surface is the size-changing torus $\vartheta(\v{x})=0$ (see \eqref{eq:enlarging_torus}). Black dots are true troubled points, red dots are false troubled points.}
    \label{fig:GINN_4D_results}
\end{figure}

\subsection{Costs and Potentialities}\label{sec:proscons}

As observed in the experiments above, the proposed NN-based detectors have good potential. In particular, a well-trained NN-based $n$-dimensional detector can identify with high precision the discontinuity interfaces of generic discontinuous functions with domain in $\R^n$. We recall that the NN (better if a GINN) is trained only once on a synthetic dataset, interestingly showing very good generalization abilities. Moreover, such a kind of models have the advantage of being portable and versatile (see \Cref{rem:github_link} below); i.e., they can be shared among different users and integrated also into algorithms different from the ones proposed in this paper. Furthermore, the experiments show that these detectors and algorithms are suitable for the discontinuity detection problem in dimensions higher than 3.

Nevertheless, some drawbacks still exist. The main one is the synthetic dataset creation, based on deterministic detectors; indeed, increasing the problem's dimensionality, the cost of this operation becomes particularly expensive. One solution is to parallelize the operation; we adopted this approach for the creation of $\mathcal{D}_4$, reducing the cost from days to hours. However, we recall that the cost of the dataset creation is paid only once; hence, it is a price worth paying if it returns a good NN-based detector applicable to any discontinuity detection problem with the same domain dimension.

\begin{remark}[GINN-based detectors availability]\label{rem:github_link}
The trained NNs used in the experiments of \Cref{sec:2dim_detection} and \Cref{sec:4dim_detection}, an implementation of \Cref{alg:sg_nn_alg_detection_parallel}, and practical toy examples are available at \href{https://github.com/Fra0013To/NNbasedDiscDetectors}{https:\-//\-github.com/\-Fra0013To/\-NNbased\-DiscDetectors}. Users can modify the practical examples available in the repository to detect discontinuities in a custom target function.
\end{remark}

% # ------------------------------------------------- #
% # ------------------------------------------------- #
% # ------------------------------------------------- #
% # ------------------------------------------------- #

\section{Conclusion}\label{sec:conc}

Our work presents a promising advancement in the field of discontinuity detection, particularly in domains of dimension higher than 3. Our novel approach utilizes Graph-Instructed Neural Networks (GINNs) trained to detect troubled points within a sparse grid used for function evaluations, where the GINNs are based on the adjacency matrix of a graph constructed on the grid's structure. We have introduced a recursive algorithm for general sparse grid-based detectors, characterized by finite termination and convergence properties. Running this algorithm with a well-trained GINN as the detector results in a cost-effective and rapid discontinuity detection method suitable for functions with $n$-dimensional domains.
Our experiments, conducted on functions with dimensions $n = 2$ and $n = 4$, demonstrate the efficiency and notable generalization abilities of the GINNs in detecting discontinuity interfaces, even for functions different from those used for building the training sets. The portable and versatile nature of the trained GINNs allows them to be shared among users and integrated into new algorithms.
While our proposed NN-based detectors exhibit a good potential, some limitations persist. Creating synthetic datasets based on deterministic detectors can be costly, particularly as problem dimensionality increases. We have addressed this by parallelizing the dataset creation process.

In summary, our method offers a powerful tool for discontinuity detection, demonstrating its potential, especially for dimensions higher than 3. Further research in dataset creation efficiency, optimal NN hyper-parameters search, and algorithm improvements will permit us to develop even more efficient software for discontinuity detection techniques, expanding the applicability of this approach to a wider range of practical problems and, in particular, to even higher dimensions.

\section*{Acknowledgements}

The authors acknowledge that this study was carried out within the FAIR-Future Artificial Intelligence Research and received funding from the European Union Next-GenerationEU (PIANO NAZIONALE DI RIPRESA E RESILIENZA (PNRR)–MISSIONE 4 COMPONENTE 2, INVESTIMENTO 1.3---D.D. 1555 11/10/2022, PE00000013). This manuscript reflects only the authors’ views and opinions; neither the European Union nor the European Commission can be considered responsible for them.
The authors acknowledge support from the Italian MIUR PRIN project 20227K44ME, Full and Reduced order modeling of coupled systems: focus on non-matching methods and automatic learning (FaReX).

\subsection*{Code Availability:}
The code for loading the trained NN-based detectors illustrated in this paper, for an implementation of \Cref{alg:sg_nn_alg_detection_parallel}, and for running practical toy examples is available at \href{https://github.com/Fra0013To/NNbasedDiscDetectors}{https:\-//\-github.com/\-Fra0013To/\-NNbased\-DiscDetectors}. See also \Cref{rem:github_link}.

%--------------------------------------------------
%--------------------------------------------------

\appendix

\section{Practical Details for Building NN-based Detectors}\label{sec:building_details}

In this appendix, we report the details of the procedure used for building the NN-based detectors. We start illustrating how to build a deterministic approximation of an exact detector $\Delta^*$ for discontinuous functions with discontinuities contained in the zero-level set of known continuous functions (\ref{sec:zerolevdet}). Then, we introduce the random piece-wise continuous functions used for generating the synthetic data (\ref{sec:synth_data_creation}), the preprocessing function for NN inputs (\ref{sec:data_prep}), and the architecture archetype used for the NN models (\ref{sec:NNarch_archetypes}).

\subsection{Zero-level Set Detection and Approximation of Exact Detectors}\label{sec:zerolevdet}

\Cref{alg:sg_alg_detection} can be easily extended to other tasks if we use detectors focused on identifying other kind of troubled points. For example, we can consider detectors built for detecting points that are near to the zero-level set of a continuous function. In these cases, Definitions \ref{def:troubled_point}-\ref{def:exact_disc_det} can be reformulated considering a continuous function $f:\Omega\subseteq\R^n\rightarrow\R$ and its zero-level set $\zeta:=\{\v{x}\in\Omega \ | \ f(\v{x}) = 0\}$, in place of the function $g\in\mathcal{F}$ and its set of discontinuity points $\delta$, respectively. From now on, we let $Z$ denote the zero-level set detectors for such a kind of extension of \Cref{alg:sg_alg_detection}, and $Z^*$ denotes the exact ones.

Let $\mathcal{C}$ denote the set of continuous functions $f:\Omega\rightarrow\R$, for any $\Omega\subseteq\R^n$. Then, for each fixed sparse grid $\mathcal{S}$, we can define a deterministic zero-level detector $Z^{(t+1)}:\mathcal{C}\times\mathcal{S}_{/\simeq}\rightarrow\{0,1\}^N$, fixed $t\in\N$, $t\geq 2$, described by the following algorithm:
\begin{enumerate}
    \item Let $G'$ be the SGG based on $\mathcal{S}'\simeq\mathcal{S}$, $\mathbb{B}(\mathcal{S}')\subset\Omega$. Then, for each edge $\{\v{x}_i,\v{x}_j\}\in E'$, we compute the equispaced knots $\v{x}_{ij}^{(\tau)} := \v{x}_i + \frac{\tau}{t} (\v{x}_j-\v{x}_i)$, $\tau=0,\ldots ,t$;
    
    \item For each edge $\{\v{x}_i,\v{x}_j\}\in E'$, we compute $s(\v{x}_{ij}^{(\tau)}):=\mathrm{sign}(f(\v{x}_{ij}^{(\tau)}))$, $\tau=0,\ldots ,t$;
    
    \item For each edge $\{\v{x}_i,\v{x}_j\}\in E'$, $\v{x}_i$ is detected as troubled point ($p_i=1$) if $s(\v{x}_i)=0$ or exist $\tau\in\{1,\ldots ,\lceil t/2 \rceil \}$ such that $s(\v{x}_i)\neq s(\v{x}_{ij}^{(\tau)})$. Analogously, $\v{x}_j$ is detected as troubled point if $s(\v{x}_j)=0$ or exist $\tau\in\{\lfloor t/2 \rfloor,\ldots t-1\}$ such that $s(\v{x}_j)\neq s(\v{x}_{ij}^{(\tau)})$.
\end{enumerate}

\begin{proposition}[$Z^{(t+1)}$ approximates $Z^*$]\label{prop:zerolev_detect_t_limit}
The detector $Z^{(t+1)}$ can be considered an approximation of the exact detector $Z^*$, i.e.: 
\begin{equation}\label{eq:zerolev_detect_t_limit}
    \lim_{t\rightarrow +\infty}Z^{(t+1)}(\v{x})=Z^*(\v{x})\,,    
\end{equation}
for each $\v{x}\in\mathcal{S}'$ and each $\mathcal{S}'\simeq\mathcal{S}$.
\end{proposition}

\begin{proof}
Let $e_{ij}=\{\v{x}_i,\v{x}_j\}\in E'$ be an edge intersecting $\zeta$ and let $\v{x}_i^\zeta\in\zeta\cap e_{ij}$ be the nearest point to $\v{x}_i$ such that $\norm{\v{x}_i-\v{x}_i^\zeta}\leq \norm{\v{x}_i^\zeta-\v{x}_j}$; i.e., by definition, $\v{x}_i$ is a troubled point with respect to the zero-level set detection problem (generalization of \Cref{def:troubled_point}).

Then, there exists $\tau^\zeta\in[0, \norm{\v{x}_i -\v{x}_j}/2]$ such that
\begin{equation*}
    \v{x}_i^\zeta = \v{x}_i + \tau^\zeta \frac{\v{x}_j - \v{x}_i}{\norm{\v{x}_j - \v{x}_i}}\,.
\end{equation*}

Now, let $\widehat{\v{x}}_{i}$ be such that 
\begin{equation*}
    \widehat{\v{x}}_{i} = \v{x}_i + \frac{\widehat{\tau}}{t} (\v{x}_j-\v{x}_i)\,,
\end{equation*}
with $t\in\N$ fixed, $\widehat{\tau}\in\{1,\ldots ,\lceil t/2\rceil\}$, and 
\begin{equation*}
    \widehat{\tau}=\argmin_{\tau \geq \tau^\zeta} \tau - \tau^\zeta\,. 
\end{equation*}
Since $f$ is continuous, for each $t$ sufficiently large, we have that this $\widehat{\tau}$ exists such that $\widehat{\v{x}}_i$ is the first vector of the sequence $\{\v{x}_{ij}^{(\tau)}\}_{\tau=1,\ldots ,t}$ where the sign of $f$ is not equal to $\mathrm{sign}(f(\v{x}_i))$; i.e., $\v{x}_i$ is a troubled point for $Z^{(t+1)}$, detected thanks to $\widehat{\v{x}}_i$.

Let $h=\norm{\v{x}_i -\v{x}_j}/t$; then, by construction, we also have that $\norm{\v{x}_i^\zeta - \widehat{\v{x}}_i} < h$. Therefore, it holds 
\begin{equation*}
    \lim_{t\rightarrow +\infty}\norm{\v{x}_i^\zeta - \widehat{\v{x}}_i} = 0\,,
\end{equation*}
and the troubled points detected by $Z^{(t+1)}$ tend to coincide with the true troubled points detected by $Z^*$, when $t\rightarrow +\infty$.

\end{proof}

The zero-level set extension of the algorithm is particularly useful for building a good approximation of the troubled points returned by an exact discontinuity detector $\Delta^*$ for discontinuous functions with discontinuity points contained in the zero-level set of a continuous function; e.g., for a function $g:\Omega\subseteq\R^n\rightarrow\R$ such that
\begin{equation}\label{eq:g_piecewise_zerlevset}
    g(\v{x})=
    \begin{cases}
    g_1(\v{x})\,,\quad &\text{if }f(\v{x})\geq 0\\
    g_2(\v{x})\,,\quad &\text{otherwise}
    \end{cases}\,,
\end{equation}
where $g_1,g_2$, and $f$ are continuous functions in $\Omega$. 

Indeed, for a function as in \eqref{eq:g_piecewise_zerlevset}, we have that $\delta\subseteq\zeta$; then, assuming $f$ is known, it holds that the set of troubled points detected by an exact discontinuity detector $\Delta^*$ is contained in the set of troubled points detected by an exact zero-level set detector $Z^*$. Therefore, we can use a detector $Z^{(t+1)}$ to approximate $Z^*$ and, as a consequence, to approximate the troubled points returned by $\Delta^*$.

\subsection{Synthetic Dataset Creation}\label{sec:synth_data_creation}

At item \ref{item:random_discfuncs} of the procedure in \Cref{sec:sketch_procedure_NNbuilding}, we generate a set of random piece-wise continuous functions $g^{(q)}:\Omega\subseteq\R^n\rightarrow\R$, $q=1,\ldots ,Q$. In our numerical experiments (\Cref{sec:num_exp}), we generate a dataset from Legendre polynomial piece-wise continuous functions characterized by one linear, spherical (see \cite{Wang2022}), or polynomial discontinuity interface (see \cite{DISC4NN}). Indeed, similarly to the convolutional detectors in \cite{Wang2022}, we observe that our NN-based detectors show good generalization abilities in detecting discontinuities of different nature, even if trained only on these three limited cases (see \Cref{sec:num_exp}).

Let $P_n$ denote the Legendre polynomial of degree $n$; then, the random discontinuous functions in $\mathcal{G}$ we consider are piece-wise smooth functions $g:[-1,1]^n\rightarrow\R$ defined as \eqref{eq:g_piecewise_zerlevset}, where:
\begin{itemize}
    \item the functions $g_1,g_2:[-1,1]^n\rightarrow\R$ are defined as
    \begin{equation}\label{eq:n_dim_legendre}
        g_j(\v{x}) = \sum_{\v{h}\in\mathcal{I}_{\rm sum}(4)} \left(c_{\v{h}}\prod_{i=1}^n P_{h_i}\left(\frac{x_i + 1}{2}\right)\right)\,,\quad \forall \ j=1,2\,,
    \end{equation}
    with $\v{h}=(h_1,\ldots ,h_n)\in\N^n$ multi-index and $c_{\v{h}}$ random coefficients sampled from the normal distribution $\mathcal{N}(0, 10)$;

    \item the function $f$ characterizing the discontinuity interface of $g$ is also generated randomly. We consider three types of interfaces, labeled with specific symbols.
    \begin{enumerate}
        \item \emph{Linear cut:}
        \begin{equation*}\label{eq:linear_cut}
            \eta(\v{x})= \v{x}^T\frac{\v{w}}{\norm{\v{w}}} + b\,,
        \end{equation*}
        where $w_i\in\mathcal{N}(0,1)$ and $b\in\mathcal{U}(-1, 1)$.
        
        \item \emph{Spherical cut:}
        \begin{equation*}\label{eq:spherical_cut}
            \sigma(\v{x})= \norm{\v{x} - \v{c}} - r\,,
        \end{equation*}
        where $c_i\in\mathcal{U}(-1, 1)$ and $r = \min\{0.2,\rho\}$, with $\rho\in\mathcal{U}(0, \sqrt{n})$.
        
        \item \emph{Polynomial cut:}
        \begin{equation*}\label{eq:polynomial_cut}
            \pi(\v{x})= C \frac{\Pi(\v{\xi})}{\max_{\v{y}\in [-1,1]^{n-1}}|\Pi(\v{y})|} - x_{\iota}\,,
        \end{equation*}
        where $C\in \mathcal{U}(0.75, 1.15)$, $\iota\in\mathcal{U}(\{1,\ldots ,n\})$, $\v{\xi}:=(x_1,\ldots ,x_{\iota-1},x_{\iota+1},\ldots,x_n)\in\R^{n-1}$, and
        \begin{equation*}\label{eq:polynomial_cut_polypart}
            \Pi(\v{\xi})= \sum_{\v{h}\in\mathcal{I}_{\rm sum}(4)} \left(c_{\v{h}}\prod_{i=1, i\neq\iota}^n P_{h_i}\left(\frac{x_i + 1}{2}\right)\right)\,,
        \end{equation*}
        with $\v{h}=(h_1,\ldots ,h_{\iota-1},h_{\iota+1},\ldots,h_n)\in\N^{n-1}$ multi-index and $c_{\v{h}}\in\mathcal{N}(0,10)$.
        
        Namely, $\pi$ is the polynomial $\Pi$ rescaled to a maximum absolute value equal to C, minus $x_\iota$. See \Cref{fig:algorithm_scheme} and \Cref{sec:num_exp} for 2D examples of such a kind of discontinuity interface.
        
    \end{enumerate}
\end{itemize}

\subsection{Preprocessing of Input Data}\label{sec:data_prep}

At step \ref{item:NNtraining_step} of the procedure in \Cref{sec:sketch_procedure_NNbuilding}, we train the NN on the synthetic data generated. We observe that the inputs $\v{g}'$ have a range of values that depends on the random generation of the discontinuous functions in $\mathcal{G}$ (on which we have a limited control); then, a preprocessing of the input data is necessary in order to make the NN-based detector applicable to any discontinuous function $g$ (independently on the values it assumes on the sparse grid points).

Typically, NN inputs are standardized (i.e., centered in zero and with normalized standard deviation) but, for our task, we use a different preprocessing function. Indeed the standardization must be performed with respect to the mean and standard deviation vectors of the training data; therefore, it is not suitable for discontinuous functions particularly different from the ones randomly generated at items \ref{item:random_discfuncs}-\ref{item:zlc_det_dataset_02}.

Let $\gamma:\R^n\rightarrow\R^n$ denote the function representing the preprocessing operation of the input data. Naive but efficient preprocessing functions for our needs are the rescaling functions, because they permit to ``feed'' the NN with vectors having all the values into a chosen range $[\alpha,\beta]$. In particular, we use the following function:
\begin{equation}\label{eq:absmax_scaler}
    \gamma(\v{g}') = 
    \begin{cases}
        \v{0}\,,\quad & \text{if }g'_{\max}= g'_{\min}=0\\
        \v{g}' / \max\{|g'_{\max}|, |g'_{\min}|\})\,,\quad & \text{otherwise}
    \end{cases}
    \,,
\end{equation}
such that the elements of $\gamma(\v{g}')$ are always between $-1$ and $1$, and where $g'_{\min}:=\min\{g(\v{x}')\,|\,\v{x}'\in\mathcal{S}')\}$, $g'_{\max}:=\max\{g(\v{x}')\,|\,\v{x}'\in\mathcal{S}')\}$. We prefer \eqref{eq:absmax_scaler} than a simplest rescaling function with values between 0 and 1, defined by
\begin{equation}\label{eq:minmax_scaler}
    \gamma(\v{g}') = 
    \begin{cases}
        \v{0}\,,\quad & \text{if }g'_{\max}= g'_{\min}\\
        (\v{g}' - g'_{\min}\, \v{e})/(g'_{\max} - g'_{\min})\,,\quad & \text{otherwise}
    \end{cases}
    \,,
\end{equation}
where $\v{e}=(1,\ldots ,1)\in\R^N$, because \eqref{eq:minmax_scaler} always transforms almost-constant vectors into vectors that are not almost-constant. Some preliminary investigations confirmed the effectiveness of our choice, since NNs trained using \eqref{eq:absmax_scaler} showed better performances than NNs trained using \eqref{eq:minmax_scaler} (in particular, on piece-wise constant functions).

\subsection{Architecture Archetypes}\label{sec:NNarch_archetypes}

For both MLPs and GINNs, we use a fixed architecture archetype, with characteristics depending on the SGG. We postpone to future work experiments based on different architectures and/or hyper-parameters.

The architecture archetype we use is based on residual blocks \cite{He2016_ResidualNN} and exploits also the batch-normalization \cite{Ioffe2015_BATCHNORMALIZATION}. The number of units per layer is $N$ (i.e., the number of sparse grid nodes) while the depth is characterized by the diameter of the SGG; this depth characterization has been chosen mainly thinking to GINN models (see \cite[Proposition 1]{GINN}), but it is a good choice also for MLPs, since residual blocks permit to better exploit the model's depth (see \cite{He2016_ResidualNN}).

Let $\mathcal{S}$ be a sparse grid in $\R^n$ of $N$ points and let $A$ be the adjacency matrix of $G$, the WSGG based on $\mathcal{S}$. The architecture archetype we consider (see \Cref{fig:NNres_archetype}) is defined by:
\begin{itemize}
    \item One input layer $I$ of $N$ units;
    
    \item One hidden GI/FC layer $L_1$ of $N$ units and activation function $\psi$. If it is a GI layer, it is based on $A$ and it is characterized by $F\geq 1$ output features per node;
    
    \item One Batch-normalization layer $B_1$;
    
    \item $\left\lfloor \mathrm{diam}(G) / 2 \right\rfloor$ residual block. In particular, for $h=1,\ldots , \left\lfloor \mathrm{diam}(G) / 2 \right\rfloor$, the residual block $R_{h+1}$ is given by:
    \begin{itemize}
        \item  One hidden GI/FC layer $L_{h+1}'$ of $N$ units and activation function $\psi$. If it is a GI layer, it is based on $A$ and it is characterized by $F\geq 1$ input and output features per node;
        
        \item One Batch-normalization layer $B_{h+1}'$;
        
        \item  One hidden GI/FC layer $L_{h+1}''$ of $N$ units and \emph{linear} activation function. If it is a GI layer, it is based on $A$ and it is characterized by $F\geq 1$ input and output features per node;
        
        \item One layer $\Sigma_{h+1}$ that returns the sum of the outputs of $L_{h+1}''$ and $\Sigma_h$ ($\Sigma_1\equiv L_1$ by convention), applying to it the activation function $\psi$.
        
        \item One Batch-normalization layer $B_{h+1}''$;
    \end{itemize}
    
    \item One GI/FC layer $L_{\rm fin}$ of $N$ units and \emph{sigmoid} activation function. If it is a GI layer, it is based on $A$ and it is characterized by $F\geq 1$ output features per node; moreover, it is also endowed with a pooling operation, that aggregates the $F$ output features per node into one (with values still in $[0,1]$).
\end{itemize}

\begin{figure}[htb!]
    \centering
    \includegraphics[height=0.5\textheight]{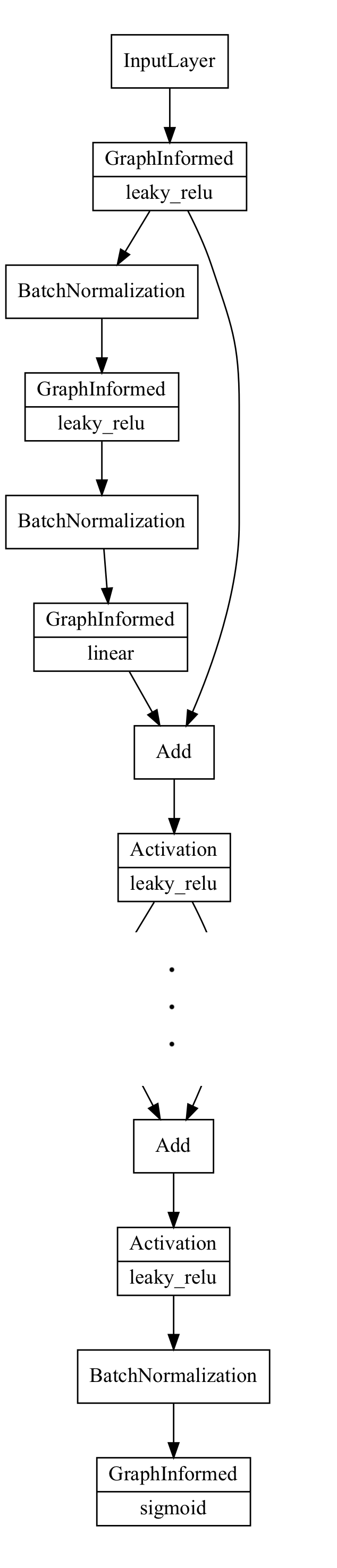}
    \caption{Example of architecture archetype described in \Cref{sec:NNarch_archetypes} for a GINN model, written in Keras (Tensorflow) \cite{tensorflow2015-whitepaper,keras2015}. In this example, the activation function used is the \emph{leaky relu} function.}
    \label{fig:NNres_archetype}
\end{figure}

% -------------------------------------------------------------------

 \bibliographystyle{elsarticle-num} 
 \bibliography{References/aiPapers,References/DellaPapers,References/discontinuous,References/NetworksPapers,References/optimizationPapers,References/otherPapers,References/SparseGridPapers,References/SBgroupPapers}

\providecommand{\noopsort}[1]{}
\begin{thebibliography}{10}
\expandafter\ifx\csname url\endcsname\relax
  \def\url#1{\texttt{#1}}\fi
\expandafter\ifx\csname urlprefix\endcsname\relax\def\urlprefix{URL }\fi
\expandafter\ifx\csname href\endcsname\relax
  \def\href#1#2{#2} \def\path#1{#1}\fi

\bibitem{JAKEMAN2013}
J.~D. Jakeman, A.~Narayan, D.~Xiu, Minimal multi-element stochastic collocation for uncertainty quantification of discontinuous functions, Journal of Computational Physics 242 (2013) 790 -- 808.
\newblock \href {https://doi.org/https://doi.org/10.1016/j.jcp.2013.02.035} {\path{doi:https://doi.org/10.1016/j.jcp.2013.02.035}}.

\bibitem{Jaeger1998_PHASETRANSITION}
G.~Jaeger, {The Ehrenfest Classification of Phase Transitions: Introduction and Evolution}, Archive for History of Exact Sciences 53 (1998) 51--81.
\newblock \href {https://doi.org/10.1007/s004070050021} {\path{doi:10.1007/s004070050021}}.

\bibitem{Wang2022}
S.~Wang, Z.~Zhou, L.-B. Chang, D.~Xiu, {Construction of discontinuity detectors using convolutional neural networks}, Journal of Scientific Computing 91~(2) (2022) 40.
\newblock \href {https://doi.org/10.1007/s10915-022-01804-z} {\path{doi:10.1007/s10915-022-01804-z}}.

\bibitem{Archibald2005}
R.~Archibald, A.~Gelb, J.~Yoon, {Polynomial fitting for edge detection in irregularly sampled signals and images}, SIAM Journal on Numerical Analysis 43~(1) (2005) 259--279.
\newblock \href {https://doi.org/10.1137/S0036142903435259} {\path{doi:10.1137/S0036142903435259}}.

\bibitem{Archibald2009}
R.~Archibald, A.~Gelb, R.~Saxena, D.~Xiu, \href{http://dx.doi.org/10.1016/j.jcp.2009.01.001}{{Discontinuity detection in multivariate space for stochastic simulations}}, Journal of Computational Physics 228~(7) (2009) 2676--2689.
\newblock \href {https://doi.org/10.1016/j.jcp.2009.01.001} {\path{doi:10.1016/j.jcp.2009.01.001}}.
\newline\urlprefix\url{http://dx.doi.org/10.1016/j.jcp.2009.01.001}

\bibitem{Jakeman2011}
J.~D. Jakeman, R.~Archibald, D.~Xiu, \href{http://dx.doi.org/10.1016/j.jcp.2011.02.022}{{Characterization of discontinuities in high-dimensional stochastic problems on adaptive sparse grids}}, Journal of Computational Physics 230~(10) (2011) 3977--3997.
\newblock \href {https://doi.org/10.1016/j.jcp.2011.02.022} {\path{doi:10.1016/j.jcp.2011.02.022}}.
\newline\urlprefix\url{http://dx.doi.org/10.1016/j.jcp.2011.02.022}

\bibitem{ZWGB2016}
G.~Zhang, C.~G. Webster, M.~Gunzbuerger, J.~Burkardt, Hyperspherical sparse approximation techniques for high-dimensional discontinuity detection, SIAM Review 58~(3) (2016) 517--551.

\bibitem{Sm63}
S.~Smolyak, Quadrature and interpolation formulas for tensor products of certain classes of functions,, Dokl. Akad. Nauk 4 (1963) 240--243.

\bibitem{bungartz_griebel_2004}
H.-J. Bungartz, M.~Griebel, Sparse grids, Acta Numerica 13 (2004) 147–269.
\newblock \href {https://doi.org/10.1017/S0962492904000182} {\path{doi:10.1017/S0962492904000182}}.

\bibitem{Piazzola2024}
C.~Piazzola, L.~Tamellini, Algorithm 1040: The sparse grids matlab kit - a matlab implementation of sparse grids for high-dimensional function approximation and uncertainty quantification, ACM Transactions on Mathematical Software 50 (2024) 1--22.
\newblock \href {https://doi.org/10.1145/3630023} {\path{doi:10.1145/3630023}}.

\bibitem{Bozzini1994}
M.~Bozzini, F.~De~Tisi, M.~Rossini, Irregularity detection from noisy data with wavelets, AK Peters/CRC Press, 1994, pp. 75--82.

\bibitem{Suresh2016}
V.~Suresh, S.~K. Rao, G.~Thiagarajan, R.~P. Das, Denoising and detecting discontinuities using wavelets, Indian Journal of Science and Technology 9 (5 2016).
\newblock \href {https://doi.org/10.17485/ijst/2016/v9i19/85440} {\path{doi:10.17485/ijst/2016/v9i19/85440}}.

\bibitem{Jain1995book}
R.~Jain, R.~Kasturi, B.~G. Schunck, Machine Vision, McGraw-Hill, 1995.

\bibitem{Wei2005}
M.~Wei, A.~R.~D. Pierro, J.~Yin, Iterative methods based on polynomial interpolation filters to detect discontinuities and recover point values from fourier data, IEEE Transactions on Signal Processing 53 (2005) 136--146.
\newblock \href {https://doi.org/10.1109/TSP.2004.838936} {\path{doi:10.1109/TSP.2004.838936}}.

\bibitem{Gao2017}
Z.~Gao, X.~Wen, W.~S. Don, Enhanced robustness of the hybrid compact-weno finite difference scheme for hyperbolic conservation laws with multi-resolution analysis and tukey’s boxplot method, Journal of Scientific Computing 73 (2017) 736--752.
\newblock \href {https://doi.org/10.1007/s10915-017-0465-0} {\path{doi:10.1007/s10915-017-0465-0}}.

\bibitem{Vuik2016}
M.~J. Vuik, J.~K. Ryan, Automated parameters for troubled-cell indicators using outlier detection, SIAM Journal on Scientific Computing 38 (2016) A84--A104.
\newblock \href {https://doi.org/10.1137/15M1018393} {\path{doi:10.1137/15M1018393}}.

\bibitem{Ray2018}
D.~Ray, J.~S. Hesthaven, An artificial neural network as a troubled-cell indicator, Journal of Computational Physics 367 (2018) 166--191.
\newblock \href {https://doi.org/10.1016/j.jcp.2018.04.029} {\path{doi:10.1016/j.jcp.2018.04.029}}.

\bibitem{Bracco2023}
C.~Bracco, F.~Calabrò, C.~Giannelli, Discontinuity detection by null rules for adaptive surface reconstruction, Journal of Scientific Computing 97 (11 2023).
\newblock \href {https://doi.org/10.1007/s10915-023-02348-6} {\path{doi:10.1007/s10915-023-02348-6}}.

\bibitem{DISC4NN}
F.~{Della Santa}, S.~Pieraccini, \href{https://www.sciencedirect.com/science/article/pii/S0377042722003430}{Discontinuous neural networks and discontinuity learning}, Journal of Computational and Applied Mathematics 419 (2023) 114678.
\newblock \href {https://doi.org/https://doi.org/10.1016/j.cam.2022.114678} {\path{doi:https://doi.org/10.1016/j.cam.2022.114678}}.
\newline\urlprefix\url{https://www.sciencedirect.com/science/article/pii/S0377042722003430}

\bibitem{ElSayed2013}
M.~{El-Sayed}, M.~Khafagy, \href{www.ijacsa.thesai.org}{Automated edge detection using convolutional neural network} (2013).
\newline\urlprefix\url{www.ijacsa.thesai.org}

\bibitem{Liu2019}
Y.~Liu, M.-M. Cheng, X.~Hu, J.-W. Bian, L.~Zhang, X.~Bai, J.~Tang, Richer convolutional features for edge detection, IEEE Transactions on Pattern Analysis and Machine Intelligence 41~(8) (2019) 1939--1946.
\newblock \href {https://doi.org/10.1109/TPAMI.2018.2878849} {\path{doi:10.1109/TPAMI.2018.2878849}}.

\bibitem{Wang2016inbook}
R.~Wang, Edge detection using convolutional neural network, Springer, Cham, 2016, pp. 12--20.

\bibitem{Wen2018}
C.~Wen, P.~Liu, W.~Ma, Z.~Jian, C.~Lv, J.~Hong, X.~Shi, Edge detection with feature re-extraction deep convolutional neural network, Journal of Visual Communication and Image Representation 57 (2018) 84--90.
\newblock \href {https://doi.org/10.1016/j.jvcir.2018.10.017} {\path{doi:10.1016/j.jvcir.2018.10.017}}.

\bibitem{Xue2019}
C.~Xue, J.~Zhang, J.~Xing, Y.~Lei, Y.~Sun, \href{https://ieeexplore.ieee.org/document/8785855/}{Research on edge detection operator of a convolutional neural network}, IEEE, 2019, pp. 49--53.
\newblock \href {https://doi.org/10.1109/ITAIC.2019.8785855} {\path{doi:10.1109/ITAIC.2019.8785855}}.
\newline\urlprefix\url{https://ieeexplore.ieee.org/document/8785855/}

\bibitem{tensorflow2015-whitepaper}
M.~Abadi, A.~Agarwal, P.~Barham, E.~Brevdo, Z.~Chen, C.~Citro, G.~S. Corrado, A.~Davis, J.~Dean, M.~Devin, S.~Ghemawat, I.~Goodfellow, A.~Harp, G.~Irving, M.~Isard, Y.~Jia, R.~Jozefowicz, L.~Kaiser, M.~Kudlur, J.~Levenberg, D.~Man\'{e}, R.~Monga, S.~Moore, D.~Murray, C.~Olah, M.~Schuster, J.~Shlens, B.~Steiner, I.~Sutskever, K.~Talwar, P.~Tucker, V.~Vanhoucke, V.~Vasudevan, F.~Vi\'{e}gas, O.~Vinyals, P.~Warden, M.~Wattenberg, M.~Wicke, Y.~Yu, X.~Zheng, \href{https://www.tensorflow.org/}{{TensorFlow}: Large-scale machine learning on heterogeneous systems}, software available from tensorflow.org (2015).
\newline\urlprefix\url{https://www.tensorflow.org/}

\bibitem{GINN}
S.~Berrone, F.~{Della Santa}, A.~Mastropietro, S.~Pieraccini, F.~Vaccarino, \href{https://doi.org/10.3390%2Fmath10050786}{Graph-informed neural networks for regressions on graph-structured data}, Mathematics 10~(5) (2022) 786.
\newblock \href {https://doi.org/10.3390/math10050786} {\path{doi:10.3390/math10050786}}.
\newline\urlprefix\url{https://doi.org/10.3390%2Fmath10050786}

\bibitem{EWGINN}
F.~{Della Santa}, A.~Mastropietro, S.~Pieraccini, F.~Vaccarino, \href{https://linkinghub.elsevier.com/retrieve/pii/S1877750324003119}{Edge-wise graph-instructed neural networks}, Journal of Computational Science 85 (2025) 102518.
\newblock \href {https://doi.org/10.1016/j.jocs.2024.102518} {\path{doi:10.1016/j.jocs.2024.102518}}.
\newline\urlprefix\url{https://linkinghub.elsevier.com/retrieve/pii/S1877750324003119}

\bibitem{Barthelmann2000}
V.~Barthelmann, E.~Novak, K.~Ritter, High dimensional polynomial interpolation on sparse grids, Advances in Computational Mathematics 12 (2000) 273--288.
\newblock \href {https://doi.org/10.1023/A:1018977404843} {\path{doi:10.1023/A:1018977404843}}.

\bibitem{Nobile2007}
I.~Babu{s}ka, F.~Nobile, R.~Tempone, \href{https://doi.org/10.1137/050645142}{A stochastic collocation method for elliptic partial differential equations with random input data}, SIAM Journal on Numerical Analysis 45~(3) (2007) 1005--1034.
\newblock \href {http://arxiv.org/abs/https://doi.org/10.1137/050645142} {\path{arXiv:https://doi.org/10.1137/050645142}}, \href {https://doi.org/10.1137/050645142} {\path{doi:10.1137/050645142}}.
\newline\urlprefix\url{https://doi.org/10.1137/050645142}

\bibitem{Nobile2008}
F.~Nobile, R.~Tempone, C.~G. Webster, \href{https://doi.org/10.1137/060663660}{A sparse grid stochastic collocation method for partial differential equations with random input data}, SIAM Journal on Numerical Analysis 46~(5) (2008) 2309--2345.
\newblock \href {http://arxiv.org/abs/https://doi.org/10.1137/060663660} {\path{arXiv:https://doi.org/10.1137/060663660}}, \href {https://doi.org/10.1137/060663660} {\path{doi:10.1137/060663660}}.
\newline\urlprefix\url{https://doi.org/10.1137/060663660}

\bibitem{BackNobileTamelliniTempone2011}
J.~B{\"a}ck, F.~Nobile, L.~Tamellini, R.~Tempone, {Stochastic Spectral Galerkin and Collocation Methods for PDEs with Random Coefficients: A Numerical Comparison}, in: J.~S. Hesthaven, E.~M. R{\o}nquist (Eds.), {Spectral and High Order Methods for Partial Differential Equations}, Springer Berlin Heidelberg, Berlin, Heidelberg, 2011, pp. 43--62.

\bibitem{BinaryCrossEntropy_TF}
{TensorFlow}, \href{https://www.tensorflow.org/api_docs/python/tf/keras/losses/BinaryCrossentropy}{Binary cross-entropy loss - tensorflow losses}, (Accessed on October 2024).
\newline\urlprefix\url{https://www.tensorflow.org/api_docs/python/tf/keras/losses/BinaryCrossentropy}

\bibitem{Goodfellow-et-al-2016}
I.~Goodfellow, Y.~Bengio, A.~Courville, Deep Learning, MIT Press, 2016, \url{www.deeplearningbook.org}.

\bibitem{GNNsurvey2020}
Z.~Wu, S.~Pan, F.~Chen, G.~Long, C.~Zhang, P.~S. Yu, A comprehensive survey on graph neural networks, IEEE Transactions on Neural Networks and Learning Systems 32~(1) (2021) 4--24.
\newblock \href {https://doi.org/10.1109/TNNLS.2020.2978386} {\path{doi:10.1109/TNNLS.2020.2978386}}.

\bibitem{HALL2021110192}
E.~J. Hall, S.~Taverniers, M.~A. Katsoulakis, D.~M. Tartakovsky, \href{https://www.sciencedirect.com/science/article/pii/S0021999121000875}{Ginns: Graph-informed neural networks for multiscale physics}, Journal of Computational Physics 433 (2021) 110192.
\newblock \href {https://doi.org/https://doi.org/10.1016/j.jcp.2021.110192} {\path{doi:https://doi.org/10.1016/j.jcp.2021.110192}}.
\newline\urlprefix\url{https://www.sciencedirect.com/science/article/pii/S0021999121000875}

\bibitem{sparseGIlayers}
F.~D. Santa, \href{https://arxiv.org/abs/2403.13781}{Sparse implementation of versatile graph-informed layers} (2024).
\newblock \href {http://arxiv.org/abs/2403.13781} {\path{arXiv:2403.13781}}.
\newline\urlprefix\url{https://arxiv.org/abs/2403.13781}

\bibitem{Kingma2015_ADAM}
D.~P. Kingma, J.~L. Ba, {Adam: A method for stochastic optimization}, 3rd International Conference on Learning Representations, ICLR 2015 - Conference Track Proceedings (2015) 1--15\href {http://arxiv.org/abs/1412.6980} {\path{arXiv:1412.6980}}.

\bibitem{ReduceLRPlateau_TF}
{TensorFlow}, \href{https://www.tensorflow.org/api_docs/python/tf/keras/callbacks/ReduceLROnPlateau}{Reduce learning rate on plateau - tensorflow callbacks}, (Accessed on October 2023).
\newline\urlprefix\url{https://www.tensorflow.org/api_docs/python/tf/keras/callbacks/ReduceLROnPlateau}

\bibitem{EarlyStopping_TF}
{TensorFlow}, \href{https://www.tensorflow.org/api_docs/python/tf/keras/callbacks/EarlyStopping}{Early stopping - tensorflow callbacks}, (Accessed on October 2023).
\newline\urlprefix\url{https://www.tensorflow.org/api_docs/python/tf/keras/callbacks/EarlyStopping}

\bibitem{SURVEYFACTIVATIONS_Apicella2021}
A.~Apicella, F.~Donnarumma, F.~Isgr{\`{o}}, R.~Prevete, \href{https://doi.org/10.1016/j.neunet.2021.01.026}{{A survey on modern trainable activation functions}}, Neural Networks 138 (2021) 14--32.
\newblock \href {http://arxiv.org/abs/2005.00817} {\path{arXiv:2005.00817}}, \href {https://doi.org/10.1016/j.neunet.2021.01.026} {\path{doi:10.1016/j.neunet.2021.01.026}}.
\newline\urlprefix\url{https://doi.org/10.1016/j.neunet.2021.01.026}

\bibitem{Glorot2010_GLOROTunifANDnormal}
X.~Glorot, Y.~Bengio, {Understanding the difficulty of training deep feedforward neural networks}, Journal of Machine Learning Research 9 (2010) 249--256.

\bibitem{He2016_ResidualNN}
K.~He, X.~Zhang, S.~Ren, J.~Sun, {Deep residual learning for image recognition}, Proceedings of the IEEE Computer Society Conference on Computer Vision and Pattern Recognition 2016-Decem (2016) 770--778.
\newblock \href {http://arxiv.org/abs/1512.03385} {\path{arXiv:1512.03385}}, \href {https://doi.org/10.1109/CVPR.2016.90} {\path{doi:10.1109/CVPR.2016.90}}.

\bibitem{SheppLoganPhantom_1974}
L.~A. Shepp, B.~F. Logan, The fourier reconstruction of a head section, IEEE Transactions on Nuclear Science 21~(3) (1974) 21--43.
\newblock \href {https://doi.org/10.1109/TNS.1974.6499235} {\path{doi:10.1109/TNS.1974.6499235}}.

\bibitem{Ioffe2015_BATCHNORMALIZATION}
S.~Ioffe, C.~Szegedy, Batch normalization: Accelerating deep network training by reducing internal covariate shift, in: Proceedings of the 32nd International Conference on International Conference on Machine Learning - Volume 37, ICML'15, JMLR.org, 2015, p. 448–456.

\bibitem{keras2015}
F.~Chollet, et~al., Keras, \url{https://keras.io} (2015).

\end{thebibliography}





% \end{thebibliography}
\end{document}